\documentclass[pdflatex,sn-mathphys-num]{sn-jnl}
\usepackage{graphicx}%
\usepackage{multirow}%
\usepackage{amsmath,amssymb,amsfonts}%
\usepackage{amsthm}%
\usepackage{mathrsfs}%
\usepackage{fix-cm}

\usepackage[title]{appendix}%
\usepackage[dvipsnames]{xcolor}%
\usepackage{textcomp}%
\usepackage{manyfoot}%
\usepackage{booktabs}%
\usepackage{algorithm}%
\usepackage{algorithmicx}%
\usepackage{algpseudocode}%
\usepackage{listings}%
\usepackage{bm}
\usepackage{subcaption}
\usepackage{hyperref}
\usepackage{enumitem}
\usepackage{multirow}
\usepackage{comment}
\usepackage{tikz}
\usepackage{lmodern}
\usepackage{fix-cm}

\usetikzlibrary{arrows.meta, positioning, calc}

\DeclareMathOperator*{\argmin}{arg\,min}
\newtheorem{theorem}{Theorem}
\newtheorem{proposition}[theorem]{Proposition}

\newtheorem{remark}{Remark}%
\newtheorem{lemma}[theorem]{Lemma} 
\newtheorem{corollary}[theorem]{Corollary}

\newtheorem{definition}{Definition}%

\newcommand{\DHN}{\operatorname{DHN}}
\newcommand{\skipp}{\operatorname{skip}}
\newcommand{\lin}{\operatorname{lin}}
\newcommand{\PC}{\operatorname{PC}}

\newcommand{\bv}{\mathbf{v}}
\newcommand{\bw}{\mathbf{w}}

\newcommand{\bx}{\mathbf{x}}
\newcommand{\bX}{\mathbf{X}}
\newcommand{\by}{\mathbf{y}}

\newcommand{\bz}{\mathbf{z}}

\newcommand{\sgnn}{\mathop{\rm sgn}}

\begin{document}

\title[Article Title]{On the expressivity of deep Heaviside networks}

\author[1]{\fnm{Insung} \sur{Kong}}\email{insung.kong@utwente.nl}

\author[2]{\fnm{Juntong} \sur{Chen}}\email{juntong.chen@xmu.edu.cn}

\author[3]{\fnm{Sophie} \sur{Langer}}\email{s.langer@rub.de}

\author*[1]{\fnm{Johannes} \sur{Schmidt-Hieber}}\email{a.j.schmidt-hieber@utwente.nl}

\affil[1]{\orgdiv{Department of Applied Mathematics}, \orgname{University of Twente}, \orgaddress{\street{Drienerlolaan 5}, \city{Enschede}, \postcode{7522 NB}, \state{Overijssel}, \country{Netherlands}}}

\affil[2]{\orgdiv{School of Mathematical Sciences}, \orgname{Xiamen University}, \orgaddress{\street{Siming South Road 422}, \city{Xiamen}, \postcode{361005}, \state{Fujian}, \country{China}}}

\affil[3]{\orgdiv{Faculty of Mathematics}, \orgname{Ruhr University Bochum}, \orgaddress{\street{Universitätsstraße 150}, \city{Bochum}, \postcode{44801}, \state{Nordrhein-Westfalen},  \country{Germany}}}

\abstract{
We show that deep Heaviside networks (DHNs) have limited expressiveness but that this can be overcome by including either skip connections or neurons with linear activation. We provide upper and lower bounds for the Vapnik-Chervonenkis (VC) dimensions and approximation rates of these network classes. As an application, we derive statistical convergence rates for DHN fits in the nonparametric regression model.}

\maketitle

\section{Introduction}\label{sec1}

Currently, the Heaviside activation function
\[ \sigma_0(x)=\mathbb{I}(x\geq 0)\] is witnessing a renewed interest. It plays a central role in Hopfield networks \cite{hopfield1982neural} that have recently seen a {  resurgence} due to their connections to attention layers \cite{vaswani2017attention, ramsauer2020hopfield} and their recognition in the 2024 Nobel Prize in Physics awarded in part for their development. 
Moreover, the Heaviside activation function is closely related to quantized neural networks \cite{hubara2018quantized, qin2020binary}, playing a key role in enabling energy efficient deployment of large language models (LLMs) \cite{wang2023bitnet, ma2024era}. {  For tabular data, neural networks with differentiable activation function are outperformed by tree-based methods \cite{NEURIPS2022_0378c769}. The Heaviside activation function leads to tree-like decompositions and has the potential to close the gap.}

We refer to neural networks with several hidden layers and Heaviside activation function as deep Heaviside (neural) networks (DHNs). 
These networks are also known as (linear) threshold networks.
 
The Heaviside activation function can be traced back to the first attempts to build an artificial counterpart of a biological neuron. 
In the brain, the inputs of a neuron contribute to its membrane potential and the neuron discharges/fires if the membrane potential exceeds a certain threshold. McCulloch and Pitts \cite{McCulloch1943} propose a simple mathematical model for neural activity and Rosenblatt \cite{rosenblatt58} considered functions of the form 
\begin{align*}
	f: \mathbb{R}^d \rightarrow \{-1,1\}, \quad f(\bx) = \rho\big(\bw^\top \bx +v\big), \quad \bw \in \mathbb{R}^d, \ v \in \mathbb{R},
\end{align*}
with
\begin{align*}
\rho(u)= 2\sigma_0(u)-1 =
\begin{cases}
1, \quad &u\geq 0, \\
-1, \quad &u<0,
\end{cases}
\end{align*}
an affine transformation of the Heaviside activation function $\sigma_0$. For any $\bx$ in the halfspace $\bw^\top \bx +v \geq 0,$ we have $f(\bx)=1$ and for any $\bx$ with $\bw^\top \bx +v < 0,$ $f(\bx)=-1.$ The parameter vector $\bw$ is called the {\it weight vector} and the parameter $v$ is called {\it bias} or {\it shift}. This imitates a biological neuron in the sense that if the incoming signal $\bw^\top \bx$ exceeds the threshold value $-v,$ then the neuron `fires', that is, the response is $1.$ In contrast to a biological neural network, the threshold value $-v$ is itself a learnable parameter.

The parameters $(\bw,v)$ can be learned from data via the perceptron algorithm and the perceptron convergence theorem provides an upper bound on the number of required updates in the case that the two classes are separable by a hyperplane \cite{noviko1963convergence}. The perceptron algorithm can, however, not be generalized to networks of several neurons. Instead, the gradient descent based backpropagation algorithm became state-of-the-art. Computing the gradient requires, however, that the activation function is differentiable, which excludes the Heaviside activation function. To address this, researchers in the 1980s and 1990s employed sigmoidal activation functions, which can be viewed as smoothed Heaviside functions (see Figure \ref{fig:all_approximations}b). Along with deep neural networks, (variations of) the rectified linear unit (ReLU) activation function $\sigma(x)=\max(x,0)$ became dominant. While biological meaning can be attributed to rectifier activation functions via firing rate models \cite{MR1985615, pmlr-v15-glorot11a}, the Heaviside activation is a static model that determines whether or not a neuron fires for a given input. The difference lies in the timescales: The Heaviside activation models a single firing per neuron, while firing rates model the aggregated firing pattern over longer time scales.

From a biological perspective, it seems therefore interesting to compare the expressiveness of Heaviside networks to ReLU networks. Notably, Heaviside activation function can be interpreted as the derivative of the ReLU function, and each Heaviside-activated neuron can be approximated by two ReLU neurons with appropriate chosen weights (see Figure \ref{fig:all_approximations}c). 
This suggests that any approximation rate of Heaviside networks carries over to ReLU networks with at most twice as many neurons, assuming sufficiently large weights. However, the converse is not true: due to the binary nature of the Heaviside activation function, DHNs are more constrained than ReLU networks. 
Each neuron can only output a single bit (0 or 1), which restricts the amount of information that can be propagated through the network. Consequently, the expressive power of ReLU networks does not directly transfer to DHNs.

\begin{figure}[htbp]
  \centering  
    \includegraphics[width=.92\textwidth]{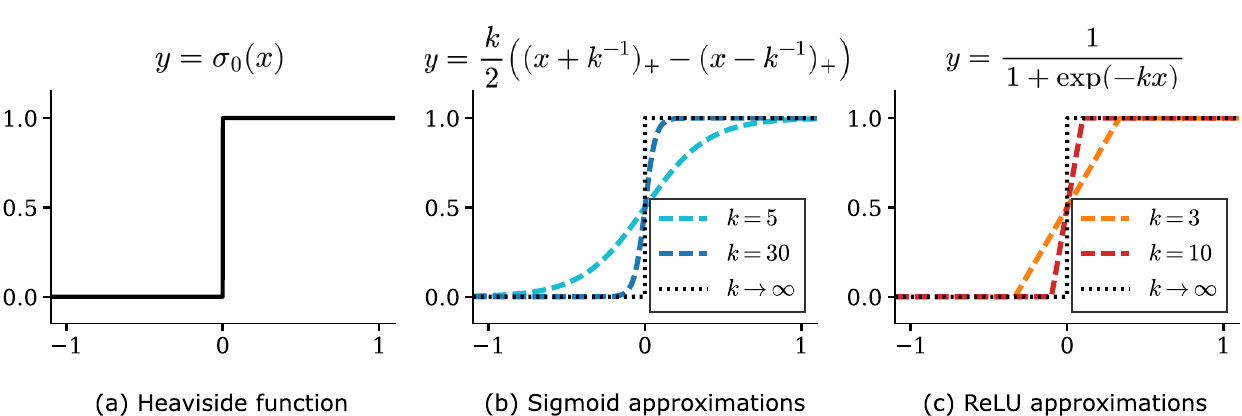}
    \vspace{2mm}
  \caption{Comparison of the Heaviside function and its approximations with ReLU and sigmoid networks. We write $x_+ := \max(x,0)$.}
  \label{fig:all_approximations}
\end{figure}

Since large parameters lead to network instability during training and are consequently hard to learn, it is desirable to avoid network parameters that tend to infinity. In this case, DHNs are no longer a subset of ReLU networks, but have distinct expressive properties. One of them is that the indicator function on a hyperrectangle $[a_1,b_1]\times \ldots \times [a_d,b_d]\subset \mathbb{R}^d$ can be perfectly represented as a two hidden layer DHN using the identity 
\begin{align}
	\mathbb{I}\big(\mathbf{x} \in [a_1,b_1]\times \ldots \times [a_d,b_d]\big)=\sigma_0 \bigg( \sum_{i=1}^d \sigma_0(x_i-a_i)+\sigma_0(-x_i+b_i) - 2d +\frac 12 \bigg), \label{represen_block}
\end{align}
see Section 7 in \cite{pinkus1999approximation}. The identity can be verified by noting that the outer neuron is activated if and only if all the inner neurons return one, which holds precisely if $a_i \leq x_i \leq b_i$ for all $i=1,\ldots,d.$

For deep ReLU networks, a rich body of literature has studied their approximation capabilities across a wide range of settings see, e.g., \cite{YAROTSKY2017103, petersen2018optimal, ohn2019smooth, schmidt2019deep, 10.1214/19-AOS1875, 10.1214/20-AOS2034, lu2021deep, langer2021approximating, jiao2023deep, yang2024optimal, kong2025posterior, YarotskyZhevnerchuk2020PhaseDiagram, SHEN2022101}. Notably,  \cite{YarotskyZhevnerchuk2020PhaseDiagram} provides a phase diagram, illustrating how approximation rates depend on function smoothness and input dimension across various deep network architectures and activation functions, linking optimal approximation rates to the VC dimension of the corresponding network class. An extension of these results to varying deep and wide networks was shown in \cite{SHEN2022101}.
Improved approximation rates via hybrid architectures has been derived in \cite{10.1162/neco_a_01364} for a combination of ReLU and floor activation function and in \cite{10.5555/3586589.3586865} for a combination of triangle wave and softsign activation function.

{  \cite{zhang2024deep} demonstrates that most approximation-theoretic properties of deep ReLU networks can be extended to a wide range of activations.
However, the Heaviside activation function is excluded from this class due to its distinctive properties, suggesting that DHNs may have different approximation properties.}

Most approximation studies involving the Heaviside activation function have primarily focused on hybrid architectures, where different non-linear activation functions are used across layers.
In particular, \cite{SHEN2021160} proves that three-layer networks with the floor activation function $x\mapsto \lfloor x \rfloor$, the exponential activation function $x\mapsto e^x$ and the Heaviside activation function can achieve uniform approximation of arbitrary continuous functions on compact domains. Their construction relies on partitioning the input domain into non-overlapping cubes and approximating target functions with piecewise constant functions. The partitioning is encoded via bit-extraction, implemented through a composition of exponential, floor, and Heaviside activations to form indicator functions.
Similarly, \cite{LongoOpschoorDischSchwabZech2023DeRhamFEM} employs networks with ReLU and Heaviside functions to approximate functions in non-continuous finite element spaces, leveraging the discontinuous nature of the Heaviside functions to construct exact representation of discontinuous basis functions. Approximation results with the Heaviside activation function are mostly used for performing localization or dealing with discontinuity instead of specifically focusing on the expressive properties of Heaviside networks itself \cite{KAINEN2000695}.

For DHNs, recent studies have explored training \cite{ergen2023globally}, model complexity \cite{wang2022vc}, and representation ability \cite{khalife2024neural}.
However, since each hidden neuron in DHNs can only represent a binary value,
most existing approximation techniques developed from prior literature are not directly applicable to DHNs. 
This work aims to bridge this gap by developing approximation methods tailored specifically to the structural properties of DHNs.

We first demonstrate the limitations of the approximation ability of DHNs (Section~\ref{sec_noskip}). This arises from the constraint that each neuron can transmit only one bit. To overcome these limitations, we suggest two possible ways to improve the original plain DHNs: adding skip connections (Section~\ref{sec_skip}) and incorporating neurons with linear activation functions (Section~\ref{sec_linear}).
Our main contribution is to derive upper and lower bounds for VC dimensions and approximation rates of both types of structure augmented DHNs. The findings illustrate the complex interplay between network topology and approximation properties. In particular, our results reveal the dependence on depth, width, and the level of structure augmentation. To illustrate how the approximation and VC dimension results complement each other, we present an application to the nonparametric regression problem (Section~\ref{sec_nonpara}). A summary of the main results is provided in Table \ref{table_summary_2}. 
A central building block of the approximation theory for ReLU networks are deep network constructions that approximate the square function $x\mapsto x^2$ on $[0,1]$ up to sup-norm error $\epsilon$ with $\lesssim \log(1/\epsilon)$ network parameters \cite{telgarsky2015representation}. 
This construction is unique for ReLU networks. Using bit-encoding, we establish a similar result for augmented DHNs, proving that sup-norm error $\epsilon$ can be achieved with $\lesssim \log^3(1/\epsilon)$ network parameters (Corollary \ref{cor_sq} and the subsequent discussion).

\begin{table}[t] 
\caption{\textbf{Summary of main results.}
$\mathcal{H}_{d}^{\beta}(1)$ denotes the class of $\beta$-H\"older smooth functions on $[0,1]^d$, see Definition \ref{hölder}. Here we write $a_n \asymp b_n$ if $b_n/|\log b_n|^{\gamma} \lesssim a_n \lesssim b_n |\log b_n|^{\gamma}$ for some $\gamma>0.$
\label{table_summary_2}}
\renewcommand{\arraystretch}{1.5}
\begin{tabular}{|c||c|c|}
\hline
DHN architecture     & VC dimension & Approximation error for $\mathcal{H}_{d}^{\beta}(1)$  \\ \hline \hline
\multirow{2}{*}{depth $L$, width $p$}  
    & $\operatorname{VC} \lesssim Lp^2 \log(Lp)  \land p^d$,  & $p^{-1} \lesssim \epsilon \lesssim p^{-(\beta \land \frac{d+1}{2})/d}$, \\  
    & by \cite{anthony1999neural, wang2022vc} & by Section \ref{sec_2.2} and \cite{yang2024optimal}\\ \hline
depth $L$, width $p$,  
    & $\operatorname{VC} \asymp Lp^2 $ & $\epsilon \asymp (Lp^2 )^{-\beta/d}$ \\  
    $s$ skip connections per layer & for $1 \leq s \leq p$, by Section \ref{sec_skip_VC} & for $d \leq s \leq p$, by Section \ref{sec_skip_smooth}  \\ \hline
depth $L$, width $p$,  
    & $\operatorname{VC} \asymp Lp^2 \vee L^2 ps$ & 
 $\epsilon \asymp (Lp^2 \vee L^2 ps)^{-\beta/d}$ \\  
    $s$ linear neurons per layer
    & for $1 \leq s \leq p$, by Section \ref{sec_4.2}  & for $d \leq s \leq p$, by Section \ref{sec_4.3} \\ \hline
\end{tabular}
\end{table}

\textit{Notation}: We use $\mathbb{N} := \{0,1,\ldots\}$, $[n] := \{1,2,\ldots,n\}$ and $\mathbb{I}(x) := \mathbb{I}(x \geq 0)$.
Vectors and vector valued functions are displayed with bold letters. 
We adopt the notation $\bm{0}_d := (0,\dots,0)$ and $\bm{1}_d := (1,\dots,1).$ Moreover, $\log(\cdot)$ denotes the binary logarithm $\log_2(\cdot)$, and $\ln(\cdot)$ the natural logarithm. 
We use $x \vee y  := \max(x,y)$ and $x \wedge y  := \min(x,y)$.
For a real-valued function  $f : \mathcal{X} \to \mathbb{R}$, we write $\|f\|_{2} :=(\int_{\mathcal{X}} f(\bx)^2 d \mu (\bx))^{1/2}$ and
$\|f\|_{\infty} := \operatorname{sup}_{\bx \in \mathcal{X}}|f(\bx)|$, where $\mu$ is the Lebesgue measure.
For the multivariate Taylor approximations, we use multi-index notation. In particular for $\bx := (x_1, \ldots, x_d)^{\top} \in \mathbb{R}^d$ and $\bm{\alpha} := (\alpha_1, \ldots, \alpha_d)^{\top} \in \mathbb{N}^d$,
    $\bx^{\bm{\alpha}} := x_1^{\alpha_1} \ldots x_d^{\alpha_d}$ and $\bm{\alpha}! := \alpha_1 ! \ldots \alpha_d !$.
We define $\sum_{\emptyset} := 0$ and $\prod_{\emptyset} := 1$. For positive sequences $(a_n)_n$ and $(b_n)_n$, 
we write $a_n \lesssim b_n$ if there exist a constant $c>0$ such that $f(n) \leq c g(n)$ for all sufficiently large $n$.
If $a_n/b_n \to \infty$, we write $a_n \gg b_n$.
For arguments $\bx_1 \in \mathbb{R}^{d_1}, \ldots, \bx_q \in \mathbb{R}^{d_q}$ and a function $f : \mathbb{R}^{d_1+\ldots+d_q} \to \mathbb{R},$ we set $f(\bx_1, \bx_2, \ldots, \bx_q):=f((\bx_1^{\top}, \bx_2^{\top}, \ldots, \bx_q^{\top})^{\top}).$ If there is no ambiguity, we write the same vector as column vector or as a row vector. An example is the width vector defining the network architecture.
\section{Deep Heaviside networks (DHNs)} \label{sec_noskip}
\subsection{Mathematical definition}

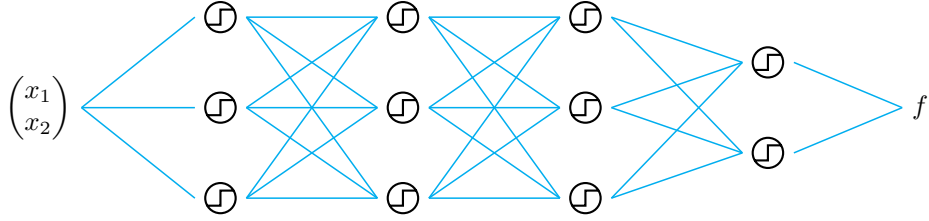
\begin{figure}[htbp]
    \centering
    \newcommand{\stepfuncicon}{
    \tikz[baseline=-0.5ex] 
    \draw[thick] (0,0) circle (0.2cm) 
        (-0.18,-0.1) -- (0,-0.1) -- (0,0.1) -- (0.18,0.1);
}
    
    \begin{tikzpicture}[node distance=2cm, auto, thick, >=Stealth, scale=0.8]

    \node[draw=white] (x) at (0, 1.5) {$	
\begin{pmatrix}
x_1 \\
x_2
\end{pmatrix}$};
        
        \node[draw=white] (f11) at (3, 3) {\stepfuncicon};
        \node[draw=white] (f12) at (3, 1.5) {\stepfuncicon};
        \node[draw=white] (f13) at (3, 0) {\stepfuncicon};
        
        \node[draw=white] (f21) at (6, 3) {\stepfuncicon};
        \node[draw=white] (f22) at (6, 1.5) {\stepfuncicon};
        \node[draw=white] (f23) at (6, 0) {\stepfuncicon};
        
        \node[draw=white] (f31) at (9, 3) {\stepfuncicon};
        \node[draw=white] (f32) at (9, 1.5) {\stepfuncicon};
        \node[draw=white] (f33) at (9, 0) {\stepfuncicon};
        
        \node[draw=white] (f41) at (12, 2.25) {\stepfuncicon};
        \node[draw=white] (f42) at (12, 0.75) {\stepfuncicon};
        
        \node[draw=white] (output) at (14.5, 1.5) {$f$};

        \foreach \j in {1,2,3} {
            \draw[-, cyan, line width=0.2mm] (x.east) -- (f1\j.west);
        }

        \foreach \i in {1,2,3} {
            \foreach \j in {1,2,3} {
                \draw[-, cyan, line width=0.2mm] (f1\i.east) -- (f2\j.west);
            }
        }

        \foreach \i in {1,2,3} {
            \foreach \j in {1,2,3} {
                \draw[-, cyan, line width=0.2mm] (f2\i.east) -- (f3\j.west);
            }
        }
        
        \foreach \i in {1,2,3} {
            \foreach \j in {1,2} {
                \draw[-, cyan, line width=0.2mm] (f3\i.east) -- (f4\j.west);
            }
        }
        
        \draw[-, cyan, line width=0.2mm] (f41.east) -- (output.west);
        \draw[-, cyan, line width=0.2mm] (f42.east) -- (output.west);
    \end{tikzpicture} \vspace{15pt}
    \caption{The network architecture  representing the network class $\DHN(L,\bm{p})$ with $L=4$ and $\bm{p}=(2,3,3,3,2,1).$} 
    \label{DHN}
\end{figure}
The network architecture $(L,\bm{p})$ is given by a number of hidden layers $L$ and an $(L+2)$-dimensional width vector $\bm{p}=(p_0,\ldots,p_{L+1}),$ where $p_0$ and $p_{L+1}$ are the respective input and output dimensions and for $\ell \in [L]$, $p_{\ell}$ denotes the number of neurons in the
$\ell$-th hidden layer. 
For a given architecture $(L,\bm{p})$,  
DHNs are parameterized by the weight matrices $\{W_{\ell} \in \mathbb{R}^{p_{\ell+1} \times p_{\ell}}\}_{\ell \in \{0,1,\ldots,L\}}$ and the shift vectors $\{\bm{b}_{\ell} \in \mathbb{R}^{p_{\ell+1}}\}_{\ell \in \{0,1,\ldots,L\}}$. 
For input $\bx  \in \mathbb{R}^{p_0},$ the DHN output $\bm{f}(\bx)$ is given by
\begin{equation} \label{DNN}
\bm{f}(\bx) := 
W_{L} \bm{f}^{(L)}(\bx) - \bm{b}_{L},
\end{equation}
where $\bm{f}^{(\ell)}$ for $\ell \in [L]$ is recursively defined via
\begin{align} \label{DNN_hidden}
\bm{f}^{(\ell)}(\bx) :=
\sigma_0\left(W_{\ell-1} \bm{f}^{(\ell-1)}(\bx) - \bm{b}_{\ell-1}\right), 
\end{align}
and $\bm{f}^{(0)}(\bx) := \bx.$ Here $\sigma_0$ denotes the componentwise application of the Heaviside activation function in the sense that for any positive integer $p,$
\begin{align} \label{def_sigma_0}
    \sigma_0
\left( (x_1, \dots, x_p)^{\top }\right)=
\big(\mathbb{I}(x_{1}), \dots, \mathbb{I}(x_p) \big)^{\top} .
\end{align}
The function class of DHNs with architecture $(L,\bm{p})$ is defined as
\begin{align*}
    \DHN\big(L, \bm{p}\big) := \big\{ \bm{f} \text{ as defined in \eqref{DNN} and \eqref{DNN_hidden}} \big\}.
\end{align*}
The parameters of this class are the $L+1$ weight matrices $W_0,\ldots,W_L$ and the corresponding shift vectors $\bm{b}_0,\ldots, \bm{b}_L.$ The total number of network parameters is $\sum_{\ell=0}^L (p_\ell+1) p_{\ell+1}.$

We briefly review existing results for shallow Heaviside networks ($L=1$). Theorem 1 of \cite{leshno1993multilayer} shows that shallow Heaviside networks have the universal approximation property. This means that for any continuous function and any error $\epsilon>0,$ one can find a (sufficiently wide) shallow neural network that approximates this function on a bounded domain in sup-norm up to error $\epsilon.$ Regarding approximation rates for shallow Heaviside networks, 
Theorem 1 of \cite{barron1993universal} proves 
$$\inf_{f\in\DHN(1,(d,p,1))} \|f - f_0 \|_{2} \lesssim p^{-\frac{1}{2}}$$
for any function $f_0: \mathbb{R}^d \to \mathbb{R}$ in the Barron class.
Recently, Corollary 2.4 of \cite{yang2024optimal} shows that for any function $f_0 : [0,1]^d \to \mathbb{R}$ that is $\beta$-H\"older smooth, 
$$\inf_{f\in\DHN(1,(d,p,1))} \|f - f_0 \|_{\infty} \lesssim p^{-\frac{\beta \land \frac{d+1}{2}}{d}}.$$
The class $\DHN(1,(d,p,1))$ has $dp+2p+1 \lesssim p$ network parameters for fixed input dimension $d$. This means that the approximation rate is $\lesssim (\text{number of parameter})^{-\beta/d}$ as long as $\beta \leq (d+1)/2.$ This is known to be the optimal approximation rate for approximating a $\beta$-H\"older smooth function defined on $[0,1]^d$ for various parametrized approximation classes. It is very interesting that this works for shallow DHNs up to $\beta=(d+1)/2$ while other schemes based on piecewise constant approximations can only achieve the optimal rate $(\text{number of parameter})^{-\beta/d}$ in the regime $\beta\leq 1,$ which is much smaller for large $d.$ \cite{yang2024optimal} shows that a similar phenomenon also holds for the powers of the ReLU activation function.

Heaviside neural networks with two or more hidden layers can exactly represent a number of well-known functions. As seen in \eqref{represen_block}, for any hypercube $A \subset \mathbb{R}^d$,
$f_0(\bx) := \mathbb{I}(\bx \in A)$ can be represented by a two hidden layer Heaviside network in the class $\DHN(2,(d,2d,1,1))$.  
Extending this result, Theorem 1 of \cite{khalife2024neural} shows that any piecewise constant function with polyhedron pieces can be represented by a two-hidden-layer DHN, 
where the total number of neurons is bounded by the number of pieces plus the number of boundary hyperplanes.
Another interesting example is the parity function  $$f_0(\bx) := \prod_{i=1}^d \big(2\mathbb{I}(x_i)-1\big) \in\{-1,1\},$$
which is piecewise constant on $\mathbb{R}^d$ with $2^d$ pieces. This function has been considered in Example 2 of \cite{khalife2024neural}, providing a construction with three hidden layers. It can also be represented by a two-hidden-layer DHN in the class $\DHN(2,(d,d,d,1))$ using that
\begin{align*}
    f_0(\bx)     
    =(-1)^d + 2 \sum_{k=1}^{d} 
     (-1)^{k+d} \cdot \mathbb{I}\left(\left(\sum_{i=1}^d \mathbb{I}(x_i)\right) - k \right).
\end{align*}

\subsection{Limited expressivity of DHNs}
\label{sec_2.2}

Among all network architectures with the same width in the hidden layers, deeper ReLU networks lead to improved approximation errors \cite{lu2021deep}. Interestingly, we show that for DHNs, the width of the first hidden layer drives the approximation error and depth cannot improve the rate.

We begin by introducing the restriction of a multivariate function to a segment and the function class of piecewise constant functions.

\medskip

\begin{definition} \label{def_restrict} For $f : [0,1]^d \to \mathbb{R}$ and $\bx_1, \bx_2 \in [0,1]^d$, we define $f|_{[\bx_1, \bx_2]} : [0,1] \to \mathbb{R}$ as
    \begin{align*}
        f\big|_{[\bx_1, \bx_2]}(t) = f\big((1-t)\bx_1 + t\bx_2\big).
    \end{align*}
\end{definition}

\begin{definition}[Piecewise constant
function]
    We say that $f : [0,1] \to \mathbb{R}$ is a piecewise constant function with $m$ pieces, if
    there exists an interval partition $\{A_1,A_2,\ldots,A_m\}$ of $[0,1]$ and real numbers $c_1,c_2,\ldots,c_m$ such that $$f(x) = \sum_{i=1}^m c_i \mathbb{I}(x \in A_i)$$
        for every $x \in [0,1]$.
    For $m=1,2,\ldots$, we define $\PC(m)$ as the set of piecewise constant functions with at most $m$ pieces.
\end{definition}

\medskip

\begin{theorem} \label{thm_plain_include}
For any $L$, any $\bm{p} = (d,p_1,\ldots,p_L,1)$, any $f \in \DHN(L,\bm{p})$, and any $\bx_1, \bx_2 \in [0,1]^d$, the function $f|_{[\bx_1, \bx_2]}$ is piecewise constant with at most $p_1+1$ pieces. Hence, for any function $f_0 : [0,1]^d \to \mathbb{R}$,
$$\inf_{f \in \DHN(L,\bm{p})} \|f-f_0 \|_{\infty} \geq \sup_{\bx_1, \bx_2 \in [0,1]^d} \, \inf_{f \in \PC(p_1 + 1)} \left\|f-f_0 |_{[\bx_1, \bx_2]} \right\|_{\infty}.$$
\end{theorem}

The proof is deferred to Appendix \ref{app_proof_thm_plain_include}.
The first statement of the theorem shows that 
the number of pieces of a DHN along any line is always bounded above by $p_1+1$ with $p_1$ the number of units in the first hidden layer. The deeper layers widths $p_2, \ldots, p_L$ cannot increase the number of pieces beyond the upper bound $p_1+1$.  
The proof is based on the fact that 
for any $f \in \DHN(L,\bm{p})$, each hidden layer function $\bm{f}^{(\ell)} : \mathbb{R}^{d} \to \{0,1\}^{p_{\ell}}$ for $\ell \in [L]$ is a piecewise constant vector-valued function, whose number of pieces is at most that of the previous hidden layer.
This has been observed in
Corollary 2 of \cite{khalife2024neural},
and was employed by \cite{wang2022vc} to derive the VC dimension bound $\lesssim p_1^d$, which depends only on the input dimension $d$ and $p_1$.

The lower bound on the approximation in Theorem~\ref{thm_plain_include} depends only on the number of neurons in the first hidden layer.
For a fixed $p_1$, arbitrary width in deeper layers cannot improve the approximation beyond this lower bound. A more concrete lower bound, proven in Appendix \ref{app_proof_thm_plain_include}, states that for any positive integer $m,$ and any continuous function $f_0 : [0,1]^d \to \mathbb{R},$
\begin{align} \label{inequ_delta_supinf}
   \sup_{\bx_1, \bx_2 \in [0,1]^d} \, \inf_{f \in \PC(m)} \left\|f-f_0 |_{[\bx_1, \bx_2]} \right\|_{\infty}
\geq \frac{\sup_{\bx\in [0,1]^d} f_0(\bx) - \inf_{\bx\in [0,1]^d} f_0(\bx)}{2m}.
\end{align}
The inequality combined with the previous theorem yields
\begin{align}
    \inf_{f \in \DHN(L,\bm{p})} \|f-f_0 \|_{\infty} \geq \frac{\sup_{\bx\in [0,1]^d} f_0(\bx) - \inf_{\bx\in [0,1]^d} f_0(\bx)}{2(p_1+1)}. \label{plain_lower_supinf}
\end{align}
Regarding the tightness of the lower bound in Theorem $\ref{thm_plain_include}$, the following proposition provides a family of target functions for which the opposite inequality holds.

\medskip

\begin{proposition} {\label{low_1_tight}}
    Consider $f_0 : [0,1]^d \to \mathbb{R}$.
    \begin{enumerate}
        \item[(i)] Let $f_0$ be a ridge function, that is, there exist an affine transformation $A_0 : \mathbb{R}^d \to \mathbb{R}$ and a function $g_0 : \mathbb{R} \to \mathbb{R}$ such that $f_0(\bx) = g_0 \circ A_0(\bx)$. Then, for any $p_1 = 1,2,\ldots,$
        \begin{align*}
 \inf_{f \in \DHN(1, (d,p_1,1))} \|f-f_0\|_{\infty}
 \leq \sup_{\bx_1, \bx_2 \in [0,1]^d} \, \inf_{f \in \PC(p_1 + 1)} \left\|f-f_0 |_{[\bx_1, \bx_2]} \right\|_{\infty}.
        \end{align*}
        \item[(ii)] If $f_0$ is $k$-Lipschitz continuous with respect to the $L_{\infty}$-norm, then, for any $p_1 = 1,2,\ldots,$
        \begin{align*}
         &\inf_{f \in \DHN(2,(d,p_1,\lceil p_1/d \rceil^d, 1))} \|f-f_0\|_{\infty} \\
         &\leq \frac{k d}{\sup_{\bx \in [0,1]^d} f_0(\bx) - \inf_{\bx \in [0,1]^d} f_0(\bx)} \sup_{\bx_1, \bx_2 \in [0,1]^d} \, \inf_{f \in \PC(p_1 + 1)} \left\|f-f_0 |_{[\bx_1, \bx_2]} \right\|_{\infty}.
        \end{align*}        
    \end{enumerate}
\end{proposition}

A bottleneck hidden layer in the sense that $\min_{\ell=2,\ldots,L} p_\ell \ll p_1,$ can further reduce the expressiveness of the network class. For instance, consider width vector $\bm{p}=(d,p_1,p_2,\ldots,p_{L-1},1,1).$ Then, there is only one neuron in the last hidden layer. Theorem \ref{thm_plain_include} guarantees that $f|_{[\bx_1, \bx_2]}$ is piecewise constant with at most $p_1+1$ pieces. However, $p_L=1$ constraints the neural network to at most two output values. For any continuous $f_0,$ the lower bound in this case is
$$\inf_{f \in \DHN(L,\bm{p})} \|f-f_0 \|_{\infty} \geq \frac{\sup_{\bx\in [0,1]^d} f_0(\bx) - \inf_{\bx\in [0,1]^d} f_0(\bx)}{4},$$
which is much worse than \eqref{plain_lower_supinf} if $p_1$ is large. Since the lower bound in Theorem \ref{thm_plain_include} already shows that this class has poor approximation theoretic properties, we do not think it is particularly valuable to work out these refinements in more detail.

Despite its limitations, the expressivity of DHNs is still superior to that of shallow Heaviside networks. 
For instance, the function 
\begin{align*}
f_0(x_1, x_2) :=& \mathbb{I}(x_1 \geq 0, x_2 < 0)+\mathbb{I}(x_1 < 0, x_2 \geq 0) \\
=& \mathbb{I}\left( \mathbb{I}(x_1) - \mathbb{I}(x_2) -\frac{1}{2}\right)+\mathbb{I}\left( -\mathbb{I}(x_1) + \mathbb{I}(x_2) -\frac{1}{2}\right) 
\end{align*}
can be easily represented using a two-hidden-layer DHN, whereas approximating it with shallow DHNs is impossible as this extends the XOR function to $[0,1]^2$ and it is well-known that the XOR function cannot be represented by a shallow network \cite{minsky69perceptrons}. 
Theorem 1 of \cite{eldan2016power} also proposes a target function that can be represented by a two-hidden-layer DHN with finite width, but requires exponentially large width in the dimension to approximate it with shallow neural networks for more than constant accuracy.
Whether the expressiveness of Heaviside networks can benefit from large depth remains unknown.

\section{Skip connections augmented DHNs} \label{sec_skip}

Plain DHNs can only transmit few bits across layers limiting their representational power, as also shown in Theorem \ref{thm_plain_include}. To improve the network's expressivity, one idea is to incorporate skip connections within the network. 
These connections allow outputs of neurons to be reused by directly linking them to subsequent layers.

Skip connections are an essential element in modern deep learning. Examples include DenseNets, introduced by \cite{huang2017densely}, where feature maps generated by earlier layers are reused in later layers. 
Other common forms of skip connections include additive shortcuts used in ResNets \cite{he2016deep} and encoder-decoder skip connection found in U-Nets \cite{ronneberger2015u}.
Skip connections in these models {  lead to smoother loss landscape \cite{NEURIPS2018_a41b3bb3}, resulting in improved} gradient flow and information propagation.

In this section, we explore DHNs with skip connections that connect the input to every hidden layer. These skip connections allow for the reuse of input data throughout the network, thereby significantly increasing
the expressive power of DHNs. 

\subsection{Mathematical definition} \label{sec_3.1}

\begin{figure}[htbp]
    \centering
\newcommand{\stepfuncicon}{
    \tikz[baseline=-0.5ex] 
    \draw[thick] (0,0) circle (0.2cm) 
        (-0.18,-0.1) -- (0,-0.1) -- (0,0.1) -- (0.18,0.1);
}

\newcommand{\stepfunciconskip}{
    \tikz[baseline=-0.5ex] 
    \draw[thick, red] (0,0) circle (0.2cm) 
        (-0.18,-0.1) -- (0,-0.1) -- (0,0.1) -- (0.18,0.1);
}

    \begin{tikzpicture}[node distance=2cm, auto, thick, >=Stealth, scale=0.8]
        \node[draw=white] (x) at (0, 1.5) {$	
\begin{pmatrix}
x_1 \\
x_2
\end{pmatrix}$};
        
        \node[draw=white] (f11) at (3, 3) {\stepfuncicon};
        \node[draw=white] (f12) at (3, 1.5) {\stepfuncicon};
        \node[draw=white] (f13) at (3, 0) {\stepfuncicon};
        
        \node[draw=white] (f21) at (6, 3) {\stepfunciconskip};
        \node[draw=white] (f22) at (6, 1.5) {\stepfunciconskip};
        \node[draw=white] (f23) at (6, 0) {\stepfuncicon};
        
        \node[draw=white] (f31) at (9, 3) {\stepfunciconskip};
        \node[draw=white] (f32) at (9, 1.5) {\stepfuncicon};
        \node[draw=white] (f33) at (9, 0) {\stepfuncicon};
        
        \node[draw=white] (f41) at (12, 2.25) {\stepfuncicon};
        \node[draw=white] (f42) at (12, 0.75) {\stepfuncicon};
        
        \node[draw=white] (output) at (14.5, 1.5) {$f$};

        \foreach \i in {1,2,3} {
            \draw[-, cyan, line width=0.2mm] (x.east) -- (f1\i.west);
        }
        
        \foreach \i in {1,2,3} {
            \foreach \j in {1,2,3} {
                \draw[-, cyan, line width=0.2mm] (f1\i.east) -- (f2\j.west);
            }
        }

        \foreach \i in {1,2,3} {
            \foreach \j in {1,2,3} {
                \draw[-, cyan, line width=0.2mm] (f2\i.east) -- (f3\j.west);
            }
        }
        
        \foreach \i in {1,2,3} {
            \foreach \j in {1,2} {
                \draw[-, cyan, line width=0.2mm] (f3\i.east) -- (f4\j.west);
            }
        }
        
        \draw[-, cyan, line width=0.2mm] (f41.east) -- (output.west);
        \draw[-, cyan, line width=0.2mm] (f42.east) -- (output.west);
        
        \draw[-, thick, red, line width=0.2mm] (x) to[out=60, in=150] (f21);
        \draw[-, thick, red, line width=0.2mm] (x) to[out=60, in=120] (f22);
        \draw[-, thick, red, line width=0.2mm] (x) to[out=60, in=160] (f31);

    \end{tikzpicture}\vspace{15pt}
    \caption{\textbf{Example of skip connections augmented DHN (skip-DHN).} 
    Hidden neurons in red are directly connected to the input. In particular, the second hidden layer contains two skip connected neurons ($s_2=2$), the third hidden layer has one skip connected neuron ($s_3=1$) and the last hidden layer has none ($s_4=0$). 
    The graph represents the network class $\DHN_{\skipp}(L,\bm{p}, \bm{s})$ with $L=4$, $\bm{p}=(2,3,3,3,2,1)$ and $\bm{s} = (2,1,0)$.
    }
    \label{DHN_skip}
\end{figure}
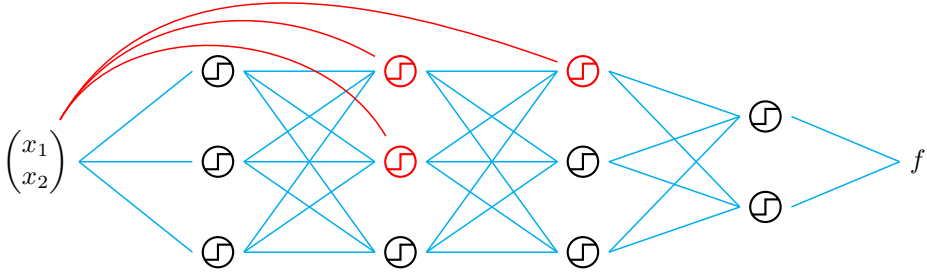

The network architecture $(L,\bm{p}, \bm{s})$ is given by a number of hidden layers $L$, an $(L+2)$-dimensional width vector $\bm{p}=(p_0,\ldots,p_{L+1})$ and an $(L-1)$-dimensional skip connection vector $\bm{s}=(s_2,\ldots,s_{L}).$ Here $p_0$ and $p_{L+1}$ are the respective input and output dimensions, and for $\ell \in [L]$, $p_{\ell}$ denotes the number of neurons in the
$\ell$-th hidden layer.
For $\ell \in \{2,\ldots,L\}$,
we augment the network by adding $s_{\ell}$ skip connections from the input to $s_{\ell}$ neurons in the $\ell$-th hidden layer. We refer to such networks as skip connections augmented DHNs (skip-DHNs).

For a given architecture $(L,\bm{p}, \bm{s})$,  
skip-DHNs are parameterized by the weight matrices $\{W_{\ell} \in \mathbb{R}^{p_{\ell+1} \times p_{\ell}}\}_{\ell \in \{0,1,\ldots,L\}}$, the shift vectors $\{\bm{b}_{\ell} \in \mathbb{R}^{p_{\ell+1}}\}_{\ell \in \{0,1,\ldots,L\}}$
and the skip connection matrices $\{V_{\ell} \in \mathbb{R}^{p_{\ell+1} \times p_{0}}\}_{\ell \in [L-1]}$, where at most $s_{\ell}$ row-vectors of $V_{\ell}$ are non-zero. 
For input $\bx  \in \mathbb{R}^{p_0},$ the output $\bm{f}(\bx)$ of the skip-DHN is given by
\begin{align} \label{DHN_skip_multi}
\bm{f}(\bx) := 
W_{L} \bm{f}^{(L)}(\bx) - \bm{b}_L,
\end{align}
where $\bm{f}^{(\ell)}$ for $\ell \in \{2,3,\ldots,L\}$ is recursively defined via 
\begin{align}
\bm{f}^{(\ell)}(\bx) :=
\sigma_0 \left(W_{\ell-1} \bm{f}^{(\ell-1)}(\bx) + V_{\ell-1}  \bx - \bm{b}_{\ell-1}\right) \label{DHN_skip_multi_hidden}
\end{align}
and 
\begin{align}
    \bm{f}^{(1)}(\bx) := \sigma_0 \big(W_{0}  \bx - \bm{b}_{0}\big). \label{DHN_skip_multi_first}
\end{align}
As in \eqref{def_sigma_0}, $\sigma_0$ denotes the componentwise application of the Heaviside activation function. Via the matrices $V_{\ell-1}$, we allow the later layers to learn features based on the output of the previous hidden layer and the input. The class of skip-DHNs is defined as
\begin{align*}
    \DHN_{\skipp}\big(L, \bm{p}, \bm{s}\big) := \Big\{ & \bm{f} \text{ as defined in (\ref{DHN_skip_multi})-(\ref{DHN_skip_multi_first}), } \# \left(\text{non-zero row vectors of } V_{\ell-1}\right) \leq s_{\ell}\Big\}.
\end{align*}
The total number of network parameters is $\sum_{\ell=0}^L (p_\ell+1) p_{\ell+1} + p_0 \sum_{\ell=2}^L s_{\ell}.$
Since $\DHN(L, \bm{p}) = \DHN_{\skipp}(L, \bm{p}, \bm{0}_{L-1})$, plain DHNs are a special case of skip-DHNs.

\medskip

\begin{remark}
\label{rem.1}
For networks defined on binary input spaces, skip connections can be replaced by few additional neurons in the sense that for any positive integers $d, D$ and any network architecture, a function $\bm{h}:\{0,1\}^d \to \mathbb{R}^D$ with $\bm{h} \in \DHN_{\skipp}(L,(d,p_1,\ldots,p_L, D), \bm{s})$ is also in $\DHN(L,(d, p_1+d,\ldots, p_L+d, D)).$

To verify this, one can forward the $d$ binary inputs $\in \{0,1\}$ to all hidden layers via the identity $b = \mathbb{I}(b - 1/2)$ for $b \in \{0,1\}$ at the cost of adding $d$ neurons to each hidden layer. The input is then available in all hidden layers. This is sufficient to replace the skip connections.
\end{remark}

\medskip

Skip connections, however, can make networks considerably more expressive for real-valued inputs. A first result in this direction is the following lower bound. 

\medskip

\begin{theorem} \label{thm_skip_include}
For any $L$, any $\bm{p} = (d,p_1,\ldots,p_L,1)$, any $\bm{s} = (s_2, \ldots, s_L)$, any $\bx_1, \bx_2 \in [0,1]^d$, and any $f \in \DHN_{\skipp}(L,\bm{p},\bm{s})$, the function $f|_{[\bx_1, \bx_2]}$ is piecewise constant with at most $(p_1+1)\prod_{\ell=2}^L (s_{\ell}+1)$ pieces.  
Hence, for any function $f_0 : [0,1]^d \to \mathbb{R}$,
    $$\inf_{f \in \DHN_{\skipp}(L,\bm{p},\bm{s})} \|f-f_0 \|_{\infty} \geq \sup_{\bx_1, \bx_2 \in [0,1]^d} \, \inf_{f \in \PC((p_1+1){\scalebox{0.7}{\(\prod\)}}_{\ell=2}^L (s_{\ell}+1))} \left\|f-f_0 |_{[\bx_1, \bx_2]} \right\|_{\infty}.$$
\end{theorem}

The proof is deferred to Appendix \ref{app_proof_skip_1}.
In fact, one can construct an example of $f \in \DHN_{\skipp}(L,\bm{p},\bm{s})$ and $\bx_1, \bx_2 \in [0,1]^d$ such that $f|_{[\bx_1, \bx_2]}$ is a piecewise constant function with $(p_1+1)\prod_{\ell=2}^L (s_{\ell}+1)$ pieces; see the proof of Theorem \ref{thm_sq}. 
Compared to the lower bound in Theorem~\ref{thm_plain_include}, this proves that skip connections can significantly increase the number of pieces and enhance the expressive power of the network.

In particular, for continuous $f_0$, the inequality combined with
\eqref{inequ_delta_supinf} yields
\begin{align}
    \inf_{f \in \DHN_{\skipp}(L,\bm{p}, \bm{s})} \|f-f_0 \|_{\infty} \geq \frac{\sup_{\bx\in [0,1]^d} f_0(\bx) - \inf_{\bx\in [0,1]^d} f_0(\bx)}{2(p_1+1) \prod_{\ell=2}^L (s_{\ell}+1)}.
    \label{skip_lower_supinf}
\end{align}
Compared to the lower bound of (\ref{plain_lower_supinf}), the bound in (\ref{skip_lower_supinf}) includes an additional factor $\prod_{\ell=2}^L (s_{\ell} + 1).$ 
If each layer has at least one skip-connected neuron, this lower bound decreases exponentially with the depth.

{  In Theorem \ref{thm_sq}, we show that the lower bounds in Theorem \ref{thm_skip_include} and \eqref{skip_lower_supinf} can be achieved up to multiplicative constants for a specific target function $f_0$.}

\subsection{Approximation of the square function}  \label{sec_3.2}
Approximating the square function $x \mapsto x^2$ is a fundamental step in deriving approximation rates for ReLU networks when approximating H\"older smooth target functions; see, e.g., \cite{telgarsky2015representation,10.1214/19-AOS1875, 10.1214/20-AOS2034, YAROTSKY2017103}. Combination with the polarization identity then leads to the approximation of the multiplication operation $xy$ for inputs $(x,y)$. 
This can be used to approximate monomials and Taylor polynomials, ultimately enabling the approximation of H\"older smooth functions.
In the following theorem, we derive an approximation result for the square function using skip-DHNs. 

\medskip

\begin{theorem} \label{thm_sq}
    For any $L, p_1$ and any $\bm{s} = (s_2,\ldots,s_L) \in \mathbb{N}^{L-1}$ with $s_L=0$,
there exists a network
$$f \in \DHN_{\skipp}(L, \bm{p}, \bm{s})$$
with width vector
    \begin{align*}
        \bm{p} &:= \left(1,p_1,p_1 + s_2, \ldots, p_1 + \sum_{\ell = 2}^{L-1} s_{\ell}, 
        \frac 12 
        \Big(p_1 + \sum_{\ell = 2}^{L-1} s_{\ell}\Big)\Big(p_1 + \sum_{\ell = 2}^{L-1} s_{\ell} +1\Big),1 \right)
    \end{align*}
such that
    $$\sup_{x \in [0, 1]} \, \left| f(x) - x^2 \right| \leq 
    \frac{1}{(p_1+1) \prod_{\ell=2}^L (s_{\ell}+1)}.$$
\end{theorem}

\medskip

The upper bound matches the lower bound established in \eqref{skip_lower_supinf} up to a factor of $2$.
Up to this factor, the approximation error in Theorem \ref{thm_sq} is therefore optimal for any depth and any skip connection vector.

The proof is deferred to Appendix \ref{app_proof_sq}. 
The proof begins by extracting bits corresponding to the digits of $x$ in a mixed radix numerical system, using the structure provided in Lemma \ref{lemma_bit_extract}. Next, each pair of bits is multiplied using the Heaviside activation function, which requires a large number of neurons in the final hidden layer.

We argue that for networks with the Heaviside activation function, existing techniques for approximating the square function fail. 
ReLU networks approximate the square function by efficiently representing compositions of the triangular function $x \mapsto 2x \mathbb{I}(x\leq 1/2) + (2-2x) \mathbb{I}(x > 1/2),$ see, e.g., \cite{YAROTSKY2017103}.  However, networks in $\DHN_{\skipp}$ are piecewise constant, making it impossible to represent such triangular functions. Another approach to approximate $x\mapsto x^2$ rewrites the second-order finite difference
\[\frac{\sigma(t+2xh)-2\sigma(t+xh)+\sigma(t)}{\sigma''(t) h^2} \approx x^2, \quad \text{as} \ h \downarrow 0\]
at a point $t$ with $\sigma''(t)\neq 0$ as a shallow network with $3$ hidden units. However, this requires the activation function $\sigma$ to be twice differentiable and does not apply to the Heaviside activation function.

To obtain a more intuitive understanding of the role of depth and the number of skip connected neurons in each layer, we restate the previous result for the case where all hidden layers have $s$ skip connected neurons. 

\medskip

\begin{corollary} \label{cor_sq}
    For any positive integers $L, s,$ and any width vector
    \begin{align*}
        \bm{p} &:= \left(1,s, 2s, \ldots, (L-1)s, 
        \frac{(Ls-s)(Ls-s+1)}{2}, 1 \right),
    \end{align*}
    we have
    $$ \inf_{f \in \DHN_{\skipp}(L, \bm{p}, s\cdot \bm{1}_{L-1})} \, \sup_{x \in [0, 1]} \, \left| f(x) - x^2 \right| \leq (s+1)^{-L+1}.$$
\end{corollary}
Here $\bm{1}_{L-1}=(1,\ldots,1)$ is the $(L-1)$-dimensional row vector with all entries equal to one. 
Specifically, for $s=1$, the width vector becomes $\bm{p}=(1,1,2,\ldots,L-1,(L-1)L/2,1)$, and the approximation error is $\leq \epsilon:=2^{-L+1},$ using $\lesssim \log^2(1/\epsilon)$ neurons and $\lesssim \log^3(1/\epsilon)$ network parameters. 
This is a significant improvement compared to deep Heaviside networks {\it without} skip connections, where the lower bound in (\ref{plain_lower_supinf}) shows that to achieve an approximation error $\leq \epsilon$,  at least $\gtrsim 1/\epsilon$ neurons and network parameters are needed. A peculiar aspect of the construction is that the width increases for deeper layers and that the last hidden layer has $(Ls-s)(Ls-s+1)/2$ many neurons. 

For deep ReLU networks, it is known that width $9$ in all hidden layers can achieve an approximation error $\leq \epsilon$ using $\lesssim \log(1/\epsilon)$ neurons and network parameters, see for instance, Lemma 20 of \cite{kohler2021supplementB}. From this point of view, DHNs seem to be inferior. However, DHNs only transmit individual bits, making the evaluation of DHNs much faster and requiring less memory if compared to deep ReLU networks of the same size.

\subsection{Bounds on the VC dimension}
\label{sec_skip_VC}
In this section, we analyze the capacity of skip-DHN function classes in terms of depth and width using the notion of VC dimension. Since skip-DHNs represent a class of real-valued functions, we adopt the extended concept of VC dimension from classification rules, as defined in \cite{peter2}. 

\medskip

\begin{definition}[Growth function, Shattering, VC dimension]\label{vc-def}
Let $\mathcal{G}$ be a class of functions from $\mathcal{X}$ to $\{0,1\}$. For any non-negative integer $m$, the growth function of $\mathcal{G}$ is defined as 
$$\Pi_{\mathcal{G}}(m):=\max_{{\bx}_{1},\ldots,{\bx}_{m}\in\mathcal{X}}\Big|\Big\{\big(g({\bx}_{1}),\ldots,g({\bx}_{m})\big),\;g\in\mathcal{G}\Big\}\Big|.$$ 
A set of $m$ points ${\bx}_{1},\ldots,{\bx}_{m}\in\mathcal{X}$ is said to be shattered by $\mathcal{G}$ if
$$\Big|\Big\{\big(g({\bx}_{1}),\ldots,g({\bx}_{m})\big),\;g\in\mathcal{G}\Big\}\Big|=2^m.$$ 
The Vapnik--Chervonenkis (VC) dimension of $\mathcal{G}$, denoted by $\operatorname{VC}(\mathcal{G})$, is defined as the largest integer $m$ for which $\mathcal{G}$ can shatter some set of $m$ points. If no such finite $m$ exists, $\operatorname{VC}(\mathcal{G}):=\infty$. For a class $\mathcal{F}$ of real-valued functions, define $\operatorname{VC}(\mathcal{F}):=\operatorname{VC}(\mathbb{I}(\mathcal{F}))$,
where $\mathbb{I}(\mathcal{F}):=\{\mathbb{I} \circ f,\;f\in\mathcal{F}\}$ and $\mathbb{I}(x) := \mathbb{I}(x \geq 0).$
\end{definition}
The VC dimension measures the complexity or richness of a function class and is thus deeply intertwined with approximation theory. We will further illustrate this connection in Proposition~\ref{vc-connect-appro}.

We analyze the VC dimension of skip-DHNs with a rectangular architecture, meaning that each hidden layer contains the same number of neurons, denoted by $p$, and the same number of skip-connected neurons, denoted by $s$. Formally, we define the class
$$\DHN_{\skipp}\big(L, p_0:p:p_{L+1}, s \big) :=  \DHN_{\skipp}\big(L, (p_0,\underbrace{p,p,\ldots,p}_{L},p_{L+1}), (\underbrace{s,s,\ldots,s}_{L-1})\big).$$
For any $L$ and $p\geq d\vee s$, the total number of parameters in $\DHN_{\skipp}(L,d\!:p\!:1,s)$ is of the order $Lp^2$. Hence, under the given condition, the increase in the number of parameters due to skip connections is of negligible order. The upper and lower bounds on the VC dimension of $\DHN_{\skipp}(L,d\!:p\!:1,s)$ are stated as follows.

\medskip

\begin{theorem}\label{thm_skip_vc_upper}
For any $d, L, s,$ and any $p \geq d \vee s\vee 2$, 
$$\operatorname{VC}\big(\DHN_{\skipp}(L,d:p:1,s)\big)\leq 30 \cdot 
Lp^2 \log\left(Lp\right).$$
Also, there exist universal constants $c, C>0$ such that 
for any $d, L, p, s$ satisfying $L \land p \geq c \vee 8 \log (Lp)$ and $1 \leq s \leq p$, 
\begin{align*}
\operatorname{VC}\big(\DHN_{\skipp}(L,d:p:1,s)\big) \geq C \cdot Lp^2.
\end{align*}
\end{theorem}

The proof is deferred to Appendix \ref{app_proof_vc}. Existing results on the VC dimension bound DHNs without skip connections. Corollary 2 of \cite{baum1988size} establishes the upper bound $\lesssim W \log W$ with $W$ the number of network parameters. Lower bound results state that under specific architectures, such as networks with very few hidden layers \cite{baum1988size, bartlett1993lower}, or networks with an extremely wide input layer and first hidden layer \cite{maass1994neural}, the bound $\gtrsim W \log W$ can be achieved. However, we cannot conclude that the VC dimension is of order $W \log W$ for other architectures of plain DHNs. In fact, for networks $\DHN(L,(d,p,\ldots,p,1))$ with a rectangular architecture, the number of parameters $W$ is of order $Lp^2$ provided that $d \leq p$. Theorem 3.6 of \cite{wang2022vc} implies that $\DHN(L,(d,p,\ldots,p,1))\lesssim p^d,$ which can be much smaller than $Lp^2$, if $L$ is large.

In this regime, Theorem~\ref{thm_skip_vc_upper} demonstrates that adding skip connections indeed increases the VC dimension, thereby significantly enriching the class of networks. Although the additional parameters introduced by the skip connections are of negligible order, they lead to skip-DHNs attaining matching upper and lower bounds on the VC dimension—up to a logarithmic factor—for a wide range of rectangular architectures.

The VC dimension of deep ReLU networks with depth $L$ and width $p$ is of order $L^2p^2$, up to a logarithmic factor; see Theorem 7 of \cite{peter2} for the upper bound, and combine Theorems 1.1 and 2.4 of \cite{lu2021deep} to obtain a matching lower bound. Thus, compared to deep ReLU networks with the same architecture, the skip-DHN function class exhibits lower complexity when $L \gg 1$.  

\subsection{Approximation of $\beta$-H\"older smooth functions}
\label{sec_skip_smooth} 

\medskip

\begin{definition}[Hölder smooth functions]
\label{hölder}
    Let $\beta = q+s$ for some $q \in \mathbb{N}$ and $s \in (0,1]$.
    The $\beta$-H\"older norm of a function $f : \mathcal{X} \to \mathbb{R}$ is defined as
    $$\left\| f \right\|_{\mathcal{C}^{\beta}} := \sum_{\substack{\bm{\alpha} \in \mathbb{N}^d \\ \left\|\bm{\alpha}\right\|_1 \leq q}}    
    \left\|\partial^{\bm{\alpha}} f\right\|_{\infty}
    +\sum_{\substack{\bm{\alpha} \in \mathbb{N}^d \\ \left\|\bm{\alpha}\right\|_1 = q}}
\sup_{\substack{\bx_1, \bx_2 \in \mathcal{X} \\ \bx_1 \neq \bx_2}} \, \frac{\left|\partial^{\bm{\alpha}} f(\bx_1)-\partial^{\bm{\alpha}} f(\bx_2)\right|}{\left\|\bx_1 - \bx_2\right\|_{\infty}^{s}}.$$
We say that a function $f$ is $\beta$-H\"older smooth if $\left\| f \right\|_{\mathcal{C}^{\beta}}$ exists and is finite. 
For $M>0$, we define
$$\mathcal{H}_{d}^{\beta}(M) := \Big\{ f : [0,1]^d \to \mathbb{R}, \left\| f \right\|_{\mathcal{C}^{\beta}} \leq M \Big\}$$
as the H\"older ball of $\beta$-smooth functions with radius $M$ on $[0,1]^d$.
\end{definition}

\medskip

As observed in previous works \cite{lu2021deep, achour2022general, SHEN2022101},
the ability of a given function class to approximate functions in the H\"older ball $\mathcal{H}_{d}^{\beta}(M)$ is closely related to the model complexity of the function class.

\medskip

\begin{proposition}\label{vc-connect-appro}
For any function class $\mathcal{F}$ consisting of functions defined on the domain $[0,1]^d$,
$$\sup_{f_0\in\mathcal{H}_{d}^{\beta}(M)} \,   \inf_{f\in\mathcal{F}}\|f-f_0\|_{\infty}\geq cM \big( \operatorname{VC}(\mathcal{F})\big)^{-\frac{\beta}{d}},
$$
with $c>0$ a constant only depending on $\beta$ and $d$.
\end{proposition}

\medskip
{The proof closely follows that of Theorem 2.4 in \cite{lu2021deep}, extending the result from integer-order smoothness to $\beta$-H\"older smooth target functions for arbitrary $\beta>0$. A complete proof is provided in Appendix~\ref{app_proof_skip_smooth}.}
Specifically, Proposition~\ref{vc-connect-appro} indicates that an upper bound on the VC dimension of $\mathcal{F}$ implies a lower bound on the approximation error over the $\beta$-H\"older class. For the function class $\DHN_{\skipp}(L,d\!:p\!:1,s)$, Theorem~\ref{thm_skip_vc_upper} shows that $\operatorname{VC}(\DHN_{\skipp}(L,d\!:p\!:1,s))\lesssim Lp^2\log(Lp).$ Combining this with Proposition~\ref{vc-connect-appro}, it follows that for any $p \geq d \vee s\vee 2$,
\begin{align*} 
    \sup_{f_0 \in \mathcal{H}_{d}^{\beta}(M)} \, \inf_{f \in \DHN_{\skipp}(L,d:p:1,s)} \left\| f - f_0 \right\|_{\infty} \geq cM\left(\frac{1}{Lp^2 \log(Lp) }\right)^{\frac{\beta}{d}}.
\end{align*}
This lower bound demonstrates that for a rectangular architecture with depth $L$ and width $p\geq d \vee s\vee 2$, the approximation error cannot decay faster than the $(Lp^2)^{-\beta/d}$, up to a logarithmic term. 

Next, we show that the rate $(Lp^2)^{-\beta/d}$ is indeed achievable, at least under some mild conditions on $L$ and $p$.

\medskip

\begin{theorem} \label{upper_smooth_skip} 
For any positive integer $d$ and any $\beta>0$, there exist a constant $c>0$ depending on $d$ and $\beta$ such that for any $L,p,s$ satisfying $L \land p \geq c$, $L \geq \log^2 (Lp)$, $p \geq \log^{\beta+2} (Lp)$, and $d \leq s \leq p$,
    $$\sup_{f_0 \in \mathcal{H}_{d}^{\beta}(M)} \, \inf_{f \in \DHN_{\skipp}(L,d:p:1,s)} \left\| f - f_0 \right\|_{\infty} \leq C M \left(\frac{\log^3 (Lp)}{Lp^2 }\right)^{\frac{\beta}{d}},$$
    with $C>0$ a constant only depending on $\beta$.
\end{theorem}

\medskip

The proof is deferred to Appendix \ref{app_proof_skip_smooth}; here we only sketch the main idea of the proof. We begin by dividing the domain $[0,1]^d$ into a grid of prespecified points. Given a target function $f_0\in\mathcal{H}_{d}^{\beta}(M)$, to approximate its value at a point $\bx\in[0,1]^d$, we construct a small network to identify the grid cell containing $\bx$. This step can be accomplished using the bit extraction technique. Subsequently, we approximate the partial derivatives of $f_0$ at the corresponding grid point, and finally use them to evaluate the Taylor polynomial expansion of $f_0$, up to a certain order, centered at  the grid point.

We first compare Theorem~\ref{upper_smooth_skip} with the approximation results for DHNs without skip connections, that is, $\DHN_{\skipp}(L,d\!:\!p\!:\!1,s)$ with $s=0$. With $p\geq d\vee s$, the total number of parameters $W$ in $\DHN_{\skipp}(L,d\!:\!p\!:\!1,s)$ is of order $Lp^2$. Now fix the number of parameters $W$. As shown in \eqref{plain_lower_supinf} of Section~\ref{sec_noskip}, in the absence of skip connections (i.e., $s = 0$), no non-constant continuous function can be approximated with faster error rate than order $p^{-1}$. Consequently, for any non-constant $f_0\in\mathcal{H}_{d}^{\beta}(M)$, approximation based on $\DHN_{\skipp}(L,d\!:\!p\!:\!1,0)$ with $W$ parameters yields the error lower bound $\gtrsim W^{-1/2}$. In contrast, Theorem \ref{upper_smooth_skip} shows that if each layer includes $d$ skip-connected neurons, the approximation error improves to $\lesssim (W/\log^3(W))^{-\beta/d}$, which is a faster rate whenever $\beta > d/2$. 

We next compare Theorem~\ref{upper_smooth_skip} with existing approximation results for deep ReLU networks. Under a fully connected architecture, Corollary 3.1 of \cite{jiao2023deep} establishes that a ReLU network with depth $L$ and width $p$ achieves an approximation error of order $\lesssim (Lp/(\log L \log p))^{-2\beta/d}$ for the class of $\beta$-H\"older smooth functions. Therefore, for networks with the same depth and the width, and particularly when $L \gg 1$, deep ReLU networks exhibit superior approximation capabilities compared to skip-DHNs. However, as mentioned already, the comparison is not fair, as Heaviside activated neurons only transmit individual bits and are therefore much easier to evaluate. 

Both of the above comparisons, regarding the approximation error for plain DHNs and deep ReLU networks, align with the discussion of the model's VC dimension in Section \ref{sec_skip_VC}.

We emphasize that, from a technical perspective, Theorem~\ref{upper_smooth_skip} allows for a wide range of choices for the depth $L$ and the width $p.$ As discussed in \cite{lu2021deep}, this is more challenging than deriving an upper bound only depending on the total number of parameters $W$.

\section{Linear neurons augmented DHNs} \label{sec_linear}

Neurons with linear (or identity) activation function $\sigma(x)=x$ can transmit information from previous layers to later layers without distortion.
In particular, this allows to represent skip connections {  between any two layers}. An alternative approach to improve the expressivity of DHNs is therefore to add to each layer neurons with linear activation. 

\subsection{Mathematical definition} \label{sec_4.1}

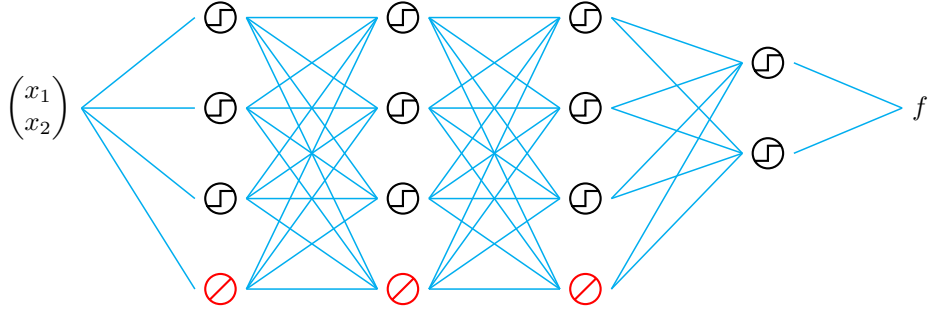
\begin{figure}[h]
    \centering
    \newcommand{\stepfuncicon}{
    \tikz[baseline=-0.5ex] 
    \draw[thick] (0,0) circle (0.2cm) 
        (-0.18,-0.1) -- (0,-0.1) -- (0,0.1) -- (0.18,0.1);
}

\newcommand{\linearfuncicon}{
    \tikz[baseline=-0.5ex] 
    \draw[thick, red] (0,0) circle (0.2cm) 
        (-0.14,-0.14) -- (0.14,0.14);
}
    
    \begin{tikzpicture}[node distance=2cm, auto, thick, >=Stealth, scale=0.8]

    \node[draw=white] (x) at (0, 1.5) {$	
\begin{pmatrix}
x_1 \\
x_2
\end{pmatrix}$};
        
        \node[draw=white] (f11) at (3, 3) {\stepfuncicon};
        \node[draw=white] (f12) at (3, 1.5) {\stepfuncicon};
        \node[draw=white] (f13) at (3, 0) {\stepfuncicon};
        \node[draw=white] (f14) at (3, -1.5) {\linearfuncicon};
        
        \node[draw=white] (f21) at (6, 3) {\stepfuncicon};
        \node[draw=white] (f22) at (6, 1.5) {\stepfuncicon};
        \node[draw=white] (f23) at (6, 0) {\stepfuncicon};
        \node[draw=white] (f24) at (6, -1.5) {\linearfuncicon};
        
        \node[draw=white] (f31) at (9, 3) {\stepfuncicon};
        \node[draw=white] (f32) at (9, 1.5) {\stepfuncicon};
        \node[draw=white] (f33) at (9, 0) {\stepfuncicon};
        \node[draw=white] (f34) at (9, -1.5) {\linearfuncicon};
        
        \node[draw=white] (f41) at (12, 2.25) {\stepfuncicon};
        \node[draw=white] (f42) at (12, 0.75) {\stepfuncicon};
        
        \node[draw=white] (output) at (14.5, 1.5) {$f$};

        \foreach \j in {1,2,3,4} {
            \draw[-, cyan, line width=0.2mm] (x.east) -- (f1\j.west);
        }

        \foreach \i in {1,2,3,4} {
            \foreach \j in {1,2,3,4} {
                \draw[-, cyan, line width=0.2mm] (f1\i.east) -- (f2\j.west);
            }
        }

        \foreach \i in {1,2,3,4} {
            \foreach \j in {1,2,3,4} {
                \draw[-, cyan, line width=0.2mm] (f2\i.east) -- (f3\j.west);
            }
        }
        
        \foreach \i in {1,2,3,4} {
            \foreach \j in {1,2} {
                \draw[-, cyan, line width=0.2mm] (f3\i.east) -- (f4\j.west);
            }
        }
        
        \draw[-, cyan, line width=0.2mm] (f41.east) -- (output.west);
        \draw[-, cyan, line width=0.2mm] (f42.east) -- (output.west);
    \end{tikzpicture} \vspace{15pt}
    \caption{\textbf{Example of linear neurons augmented DHN (lin-DHN).} 
    All (except the last) hidden layers are augmented by red linear neurons. 
    The graph represents the network class $\DHN_{\lin}(L,\bm{p}, s)$ with $L=4$, $\bm{p}=(2,3,3,3,2,1)$ and $s=1$. 
    }
    \label{lin_def}
\end{figure}

We augment a deep Heaviside network with depth $L$ and width vector $\bm{p} = (p_0,\ldots,p_{L+1})$ by adding 
$s$ neurons with a linear activation function to all except for the last hidden layer. This means that the total number of neurons is changed to $p'_{\ell}:=p_{\ell}+s$ for $\ell=1,\ldots, L-1$ and is kept the same in all other layers. We do not augment the last hidden layer as we want to restrict the network function class to piecewise constant functions. We refer to the resulting networks as linear neurons augmented DHNs (lin-DHNs).

Permuting the ordering of the neurons in a hidden layer does not change the expressiveness. We can therefore assume that the $s$ linearly activated neurons correspond to the ``last" $s$ neurons in each layer. For a vector $\bz=(z_1,\ldots,z_p,z_{p+1},\ldots, z_{p+r})$ let 
\begin{align*} 
    \sigma_r
\left( (z_1, \dots, z_p, z_{p+1}, \dots, z_{p+r})^{\top }\right)=
\big( \mathbb{I}(z_{1}), \dots, \mathbb{I}(z_p ), z_{p+1}, \dots, z_{p+r}\big)^{\top}.
\end{align*}
In particular, $\sigma_0$ applies the Heaviside activation function to all components and $\sigma_s$ applies the Heaviside function to the first $p$ components and the linear activation to the remaining $s$ components.

Setting also $p'_{0} := p_{0}$, $p'_{L} := p_{L}$ and $p'_{L+1} := p_{L+1},$ these networks are parameterized by the weight matrices $\{W_{\ell} \in \mathbb{R}^{p'_{\ell+1} \times p'_{\ell}}\}_{\ell \in \{0,1,\ldots,L\}}$ and the shift vectors $\{\bm{b}_{\ell} \in \mathbb{R}^{p'_{\ell+1}}\}_{\ell \in \{0,1,\ldots,L\}}$.
For input $\bx  \in \mathbb{R}^{p_0},$ the output $\bm{f}(\bx)$ of the lin-DHN is given by
\begin{align} \label{DHN_linear}
\bm{f}(\bx) := 
W_{L} \bm{f}^{(L)}(\bx) - \bm{b}_L,
\end{align}
where 
\begin{align} \label{DHN_linear_L}
    \bm{f}^{(L)}(\bx) :=
\sigma_0 \left(W_{L-1} \bm{f}^{(L-1)}(\bx) - \bm{b}_{L-1}\right),
\end{align}
\begin{align}  \label{DHN_linear_others}
\bm{f}^{(\ell)}(\bx) :=
\sigma_{s} \left(W_{\ell-1} \bm{f}^{(\ell-1)}(\bx) - \bm{b}_{\ell-1}\right), \quad \text{for all} \ \ell=1,\ldots, L-1,
\end{align}
and $\bm{f}^{(0)}(\bx) := \bx.$ 
The class of lin-DHNs is defined as
$$\DHN_{\lin}(L, \bm{p}, s) := \big\{ \bm{f} \text{ as defined in (\ref{DHN_linear})-(\ref{DHN_linear_others})}\big\}.$$
The total number of network parameters is 
$\sum_{\ell=0}^L (p'_\ell+1) p'_{\ell+1}$. 
Since $\DHN(L, \bm{p}) = \DHN_{\lin}(L, \bm{p}, 0)$, plain DHNs occur as a special case.

Skip-DHNs can also be modeled as lin-DHNs by introducing $p_0$ neurons with linear activation function per layer, which passes the input values across all layers.
For any $L$, $\bm{p} = (p_0,p_1,\ldots,p_L,p_{L+1})$ and $\bm{s} \in \mathbb{N}^{L-1}$, we have
\begin{align}
    \DHN(L, \bm{p}) \subseteq \DHN_{\skipp}(L, \bm{p}, \bm{s}) \subseteq \DHN_{\lin}(L, \bm{p} , p_0). \label{skip_lin_include}
\end{align}
Hence, with suitable modifications, every approximation result for the class $\DHN_{\skipp}$ also holds for $\DHN_{\lin}$.

Since the last hidden layer is not augmented, every function in $\DHN_{\lin}(L, \bm{p}, s)$ remains piecewise constant.
Consequently, we obtain a statement similar to Theorem \ref{thm_plain_include} and Theorem \ref{thm_skip_include}.

\medskip

\begin{proposition} \label{thm_lin_include}
For any $L$, any $\bm{p} = (d,p_1,\ldots,p_L,1)$, any $s \in \mathbb{N}$, any $\bx_1, \bx_2 \in [0,1]^d$, and any $f \in \DHN_{\lin}(L,\bm{p},s)$, the function 
$f|_{[\bx_1, \bx_2]}$ is piecewise constant with at most $\prod_{\ell=1}^L (p_{\ell}+1)$ pieces.  
\end{proposition}

\medskip

Compared to Theorem \ref{thm_skip_include},
the numbers of skip-connected neurons in the hidden layers of skip-DHNs are replaced with the numbers of Heaviside neurons in the hidden layers of lin-DHNs. 

\subsection{Bounds on the VC dimension} \label{sec_4.2}
To simplify notation, we focus on rectangular architectures in which each layer, except the last, contains $p$ Heaviside neurons and $s$ linear neurons. More precisely, define
\begin{align*}
    \DHN_{\lin}\big(L, p_0:p:p_{L+1}, s \big) :=  \DHN_{\lin}\big(L, (p_0,\underbrace{p,p,\ldots,p}_{L},p_{L+1}), s\big).
\end{align*}
In what follows, we always assume that $p \geq d \vee s$. Consequently, the total number of parameters in $\DHN_{\lin}(L,d\!:p\!:1,s)$ is of order $Lp^2$, and the increase in parameters due to augmentation remains negligible in this order.

\medskip

\begin{theorem}\label{VCdim-linear}
For any $d, L, s,$ and any $p \geq d \vee s \vee 2$, 
\begin{align*}
\operatorname{VC}\big(\DHN_{\lin}(L,d:p:1,s)\big)\leq 30 \cdot \big(L^2 ps \vee Lp^2 \big)\log(Lp).
\end{align*}
Also, there exist universal constants $c, C>0$ such that 
for any $d, L, p, s$ satisfying $L \land p \geq c \vee 8 \log (Lp)$ and $1 \leq s \leq p$, 
\begin{align*}
\operatorname{VC}\big(\DHN_{\lin}(L,d:p:1, s)\big)\geq C \cdot \big(L^2 ps \vee Lp^2\big).
\end{align*}    
\end{theorem}
The proof is deferred to Appendix \ref{app_proof_lin_VC}. Theorem~\ref{VCdim-linear} shows that, up to a logarithmic factor, the VC dimension of $\DHN_{\lin}(L,\!d:\!p:\!1,s)$ is of order $L^2 ps \vee Lp^2$. 
It explicitly reveals the dependence of the VC dimension on the number of linear neurons $s$, which, to the best of our knowledge, is novel in the literature. We further illustrate this with the following three scenarios.

\begin{enumerate}
    \item [(i)] If $0< s \lesssim 1$, then
    $\operatorname{VC}(\DHN_{\lin}(L,d\!:p\!:1,s))$ is of order $L^2 p \vee L p^2$, up to a logarithmic factor.  
    \item [(ii)] If $p \gg L$ and $1 \leq s \lesssim p/L$, then, up to a logarithmic factor,
    $\operatorname{VC}(\DHN_{\lin}(L,d:p:1,s))$ is of order $L p^2$. The rate is the same as for skip-DHNs.
    \item [(iii)] If $s$ is proportional to $p$, then, up to a logarithmic factor, $\operatorname{VC}(\DHN_{\lin}(L,d\!:p\!:1,s))$ is of order $L^2 p^2$. The rate is the same as for deep ReLU networks.
\end{enumerate}
To conclude, when either $s$ or $L$ is sufficiently large, augmenting with linear neurons significantly enhances model capacity compared to skip-DHNs of the same depth and width. In particular, (iii) highlights that although $\DHN_{\lin}(L,d:\!p:\!1,p)$ represents a class of piecewise constant functions, its model complexity matches that of deep ReLU networks with depth $L$ and width $p$. This observation is consistent with previous findings \cite{peter2, SCHMIDTHIEBER2021119}, which show that ReLU networks are closely related to architectures employing specific combinations of the Heaviside and linear activation functions.

Due to the above connection between lin-DHNs and ReLU networks, several results in the literature have established VC dimension bounds for lin-DHNs. Theorem 5 of \cite{sontag1998vc} provides an upper bound of $\lesssim W^2$ for any network architecture that combines the Heaviside and linear activation functions with $W$ parameters. Meanwhile, Theorem 1 of \cite{koiran1995neural} shows that, $L$ proportional to $W$ implies VC dimension $\gtrsim W^2$. An improved upper bound of $\lesssim WL\log(W) + WL^2$ was shown in \cite{peter1}, and further refined to $\lesssim WL\log(W)$ in \cite{peter2}. Although the lower bound results in \cite{peter1,peter2} nearly match the order of their upper bounds, these results rely on carefully designed architectures with enforced sparsity and a prescribed number of linear activation functions per hidden layer. In contrast, Theorem~\ref{VCdim-linear} provides a quantitative analysis of how the VC dimension grows with $s$, offering a lower bound that holds for a large range of rectangular architectures. 

\subsection{Approximation of smooth functions} \label{sec_4.3}
We now examine the approximation of H\"older smooth target functions by linear neurons augmented DHNs. As discussed in Section~\ref{sec_skip_smooth}, the approximation problem is closely tied to the VC dimension of the class $\DHN_{\lin}(L,d:p:1,s)$. By combining Proposition \ref{vc-connect-appro} with Theorem \ref{VCdim-linear}, we obtain the following lower bound
\begin{align*}
    \sup_{f_0 \in \mathcal{H}_{d}^{\beta}(M)} \, \inf_{f \in \DHN_{\lin}(L,d:p:1,s)} \left\| f - f_0 \right\|_{\infty} \geq cM \left(\frac{1}{(L^2 p s \vee Lp^2) \log(Lp) }\right)^{\frac{\beta}{d}}.
\end{align*}
The next result establishes an upper bound on the approximation error, confirming that the optimal rate $(L^2ps \vee Lp^2)^{-\beta/d}$ is achievable, up to logarithmic factors.

\medskip

\begin{theorem} \label{upper_smooth_lin} 
For any positive integer $d$ and any $\beta>0$, there exist a constant $c>0$ depending on $d$ and $\beta$ such that
for any $L, p, s$ satisfying $L \land p \geq c$,
$L \geq \log^2(Lp)$, $p \geq \log^{\beta+2}(Lp)$, and $d \leq s \leq p$,
    $$\sup_{f_0 \in \mathcal{H}_{d}^{\beta}(M)} \, \inf_{f \in \DHN_{\lin}(L,d:p:1,s)} \, \left\| f - f_0 \right\|_{\infty} \leq C M  \left(\frac{\log^3 (Lp)}{L^2 ps \vee Lp^2 }\right)^{\frac{\beta}{d}},$$
with $C>0$ a constant only depending on $\beta$.
\end{theorem}

The proof is deferred to Appendix \ref{app_proof_skip_smooth}. The imposed constraints on $L,p,s$ are mild and the approximation result applies to a broad class of architectures, accommodating varying depths, widths, and levels of linear neurons augmentation. 

In the regime $s \lesssim p/L$, the approximation rates of lin-DHNs are comparable to those of skip-DHNs. However, when $s \gg p/L$, lin-DHNs achieve faster approximation rates than skip-DHNs, given the same depth and width. In particular, when $s$ is proportional to $p$, the approximation rate improves to $\lesssim (L^2p^2)^{-\beta/d}$, matching that of deep ReLU networks with depth $L$ and width $p$. These observations from the approximation perspective align with the VC dimension analysis in Section~\ref{sec_4.2}.

\section{Application to nonparametric regression} \label{sec_nonpara}
To illustrate how the approximation and VC dimension bounds complement each other, we now discuss the nonparametric regression problem as an application.

The statistical framework is as follows. We observe $n$ independent and identically distributed (i.i.d.) training samples $(\bX_i, Y_i)$ drawn from the joint distribution of $(\bX,Y)$, where $$\bX\sim P_{\bX}, \quad Y = f_0(\bX) + \epsilon,$$ $P_{\bX}$ denotes the marginal distribution of $\bX$ over $\mathcal{X}$, and $\epsilon$ an independent noise term following a standard normal distribution. We refer to $\bX$ as the covariate and $Y\in \mathbb{R}$ as the response. For simplicity, and without loss of generality, we assume $\mathcal{X} = [0,1]^d$. 

The goal is to recover the unknown function $f_0:[0,1]^d \to \mathbb{R}$. A common paradigm in the statistical literature is to minimize the $L_2$-risk, 
$$f_0 =\argmin_{f}L(f)=\argmin_{f}\mathbb{E}_{(\bX,Y)}[(Y-f(\bX))^2].$$ However, since the distribution of $(\bX,Y)$ is unknown and only $n$ i.i.d.\ training samples $\{(\bX_i, Y_i)\}_{i=1}^n$ are available, we instead minimize the empirical risk. Specifically, given a candidate function class $\mathcal{F}$ (the model), the least-squares estimator $\widehat{f}_n$ is defined as 
\begin{align}
    \widehat{f}_n \in \argmin_{f \in \mathcal{F}} L_n(f)=\argmin_{f \in \mathcal{F}} \frac{1}{n} \sum_{i=1}^n \big(Y_i - f(\bX_i)\big)^2. \label{estimator_def_1}
\end{align}

The performance of any estimator $\hat f$ is evaluated via its excess risk, defined as the difference between the $L_2$-risks of $\hat f$ and the true regression function $f_0$, that is,
\begin{align}
R(\hat{f}, f_0) :&= L(\hat{f})-L(f_0)\nonumber\\
&=\mathbb{E}_{(\bX,Y)}[(Y-\hat f(\bX))^2]-\mathbb{E}_{(\bX,Y)}[(Y-f_0(\bX))^2]\nonumber\\
&=\|\hat{f}-f_0\|_{P_{\mathbf{X}}}^2,\label{risk-equ}
\end{align}
where $\|f\|_{P_{\mathbf{X}}}^2 := \int f(\bx)^2 \, d P_{\mathbf{X}}(d \bx)$.
For estimator $\widehat f_n$ defined as in \eqref{estimator_def_1}, a short calculation (e.g., Lemma 3.1 of \cite{jiao2023deep}) based on \eqref{risk-equ}, leads to the following decomposition 
\begin{align}
\mathbb{E}\left[\|\widehat f_n-f_0\|_{P_{\mathbf{X}}}^2\right]\leq\underbrace{2\inf_{f\in\mathcal{F}}\|f - f_0\|_{P_{\mathbf{X}}}^2}_{\text{approximation error}}+\underbrace{2 \sup_{f \in \mathcal{F}} \mathbb{E} \Big[\big|L_n(f) - L(f)\big|\Big]}_{\text{stochastic error}},\label{decomposite}
\end{align}
where the expectation $\mathbb{E}$ is taken over the randomness in the $n$ i.i.d.\ training samples $\{(\mathbf{X}_i, Y_i)\}_{i=1}^n$. The inequality \eqref{decomposite} bounds the risk of $\widehat f_n$ by decomposing it into two components: the approximation error and the stochastic error. The approximation error measures how well $\mathcal{F}$ can approximate the target regression function $f_0$, typically decreasing with model complexity. The stochastic error is linked via empirical process theory to model complexity through measures such as the VC dimension, and typically increases as the complexity of the model grows. A good estimator obtained from \eqref{estimator_def_1} should effectively balance these two terms, meaning that the model $\mathcal{F}$ used for estimation should not be too complex, yet must provide sufficient approximation power.

To further illustrate this, define the truncated version of the least-squares estimator as 
\begin{equation}\label{estimator_def_2}
\widehat{f}_{n,B_n}(\cdot) := \min\Big( \max\big(-B_n,  \widehat{f}_n(\cdot)\big), B_n \Big),
\end{equation}
$B_n=c\log n,$ for some constant $c>0.$ By applying Lemma 18 and following the proof of Lemma 19 in \cite{kohler2021supplementB}, we can derive, based on \eqref{decomposite}, that \begin{align}
\mathbb{E}\left[R\left(\widehat{f}_{n,B_n},f_0\right)\right] \lesssim \inf_{f \in \mathcal{F}} \| f-f_0 \|_{\infty}^2 + \frac{\log^3 n}{n} \operatorname{VC}(\mathcal{F}),\label{nonpara_oracle}
\end{align}
which directly links the stochastic error to the VC dimension of $\mathcal{F}$.
The approximation error depends on the properties of the target function $f_0$, assumed to be $\beta$-H\"older smooth, in line with the discussion in Sections~\ref{sec_skip_smooth} and \ref{sec_4.3}.

We first discuss the model based on skip-DHNs. Let $L_n$, $p_n$, and $s_n$ denote the depth, width, and the number of skip connected neurons in each hidden layer, respectively, all of which vary with the number of observations $n$. The estimator defined in \eqref{estimator_def_2}, based on the model $\mathcal{F} = \DHN_{\skipp}(L_n,d\!:p_n\!:1,s_n)$, is denoted by $\widehat f_{n,B_n}^{\skipp}$. Provided that $f_0 \in \mathcal{H}_{d}^{\beta}(1)$, plugging the results from Theorem \ref{thm_skip_vc_upper} and Theorem \ref{upper_smooth_skip} into \eqref{nonpara_oracle} gives 
$$\mathbb{E}\left[R\left(\widehat f_{n,B_n}^{\skipp},f_0\right)\right]\lesssim \left(\frac{\log^3 (L_n p_n)}{L_n p_n^2}\right)^{2\beta/d} + \frac{\log^3 n}{n} L_n p_n^2 \log(L_n p_n).$$
Choosing $L_n$ and $p_n$ such that $L_n p_n^2$ is of the order $n^{d/(2\beta+d)}$, we obtain, up to logarithmic factors, the convergence rate $ n^{-2\beta/(2\beta+d)}$.

Similarly, denote the estimator $\widehat{f}^{\lin}_{n,B_n}$, defined by \eqref{estimator_def_2}, based on the model $\mathcal{F}=\DHN_{\lin}(L_n,d:\!p_n:\!1,s_n)$. Again, here $L_n$, $p_n$, and $s_n$ indicate that those are sequences in $n$. Provided that $f_0 \in \mathcal{H}_{d}^{\beta}(1)$, plugging the results from Theorem \ref{VCdim-linear} and Theorem \ref{upper_smooth_lin} into \eqref{nonpara_oracle} yields
$$\mathbb{E}\left[R\left(\widehat{f}_{n,B_n}^{\lin},f_0\right)\right]\lesssim \left(\frac{\log^3 (L_n p_n)}{L_n^2 p_n s_n \vee L_n p_n^2}\right)^{2\beta/d} + \frac{\log^3 n}{n} \left(L_n^2 p_n s_n \vee L_n p_n^2\right) \log(L_n p_n).$$ Taking any combination of $L_n$, $p_n$ and $s_n$ such that $L_n^2 p_n s_n \vee L_n p_n^2$ is of order $n^{d/(2\beta+d)}$, we obtain, up to logarithmic factors, the convergence rate $n^{-2\beta/(2\beta+d)}$.

The convergence rate achieved above by the truncated least-squares estimator based on appropriately designed skip-DHN and lin-DHN models is not a coincidence. Indeed, as demonstrated in Theorem 3 of \cite{10.1214/19-AOS1875}, one can show that
$$\inf_{\widehat{f}_n} \sup_{f_0 \in \mathcal{H}_{d}^{\beta}(1)} \mathbb{E}\left[\|\widehat{f}_{n}-f_0\|_{P_{\mathbf{X}}}^2\right]\gtrsim n^{-2\beta/(2\beta+d)},$$ where the $\inf$ is taken over all possible estimators $\widehat{f}_n$. The result tells that, in general, no estimator can improve this rate for all $f_0\in\mathcal{H}_{d}^{\beta}(1)$. Therefore, for $f_0\in\mathcal{H}_{d}^{\beta}(1)$, the rate $n^{-2\beta/(2\beta+d)}$ is the fastest achievable in the worst situation, also known as the minimax rate of convergence. In conclusion, our theory shows that with appropriately chosen architectures ($L_np_n^2\asymp n^{d/(2\beta+d)}$ for skip-DHN and $(L_n^2 p_n s_n \vee L_n p_n^2)\asymp n^{d/(2\beta+d)}$ for lin-DHN) the minimax convergence rate over  $\mathcal{H}_{d}^{\beta}(1)$ can be achieved up to $\log(n)$-factors.

\section{Discussion}
In this paper, we demonstrated that skip-DHNs and lin-DHNs can overcome the limitations of DHNs.
A different extension of DHNs is the SLTNN (Shortcut Linear Threshold Neural Network) \cite{khalife2024neural}. Following the same recursive formula used to define the last hidden layer of DHNs in (\ref{DNN_hidden}), the output of the SLTNN is given by
$$f(\bx) := \left(A \bx + \bm{b}\right)^{\top} \bm{f}^{(L)}(\bx),$$
where $A \in \mathbb{R}^{p_{\ell} \times p_0}$ and $\bm{b} \in \mathbb{R}^{p_{\ell}}$ are additional parameters.
That is, instead of applying an affine transformation of the last hidden layer, the final output is computed as the inner product between the last hidden layer and an affine transformation of $\bx$. 
While all DHN types considered in this work yield piecewise constant functions, SLTNNs are piecewise linear functions with flexible slopes exhibiting entirely different properties.

Regarding the training of (large) DHNs, the Heaviside function has vanishing derivative everywhere except at a single non-differentiable point. Traditional gradient-based algorithms can therefore not be deployed to train the DHN parameters. For shallow Heaviside networks, the Extreme Learning Machine (ELM) \cite{huang2004extreme, huang2006can} suggest to randomly assign all network parameters and only training the network parameters in the output layer. For DHNs, \cite{ergen2023globally, khalife2024neural} proposed learning algorithms based on an ELM variation: Given a dataset, they first determine all possible (vector) values for the last hidden layer and then, for each given value, only train the last layer.
The computational cost for training is polynomial in the sample size and exponential in the number of parameters.

Widely recognized for training is the Straight-Through Estimator (STE) \cite{bengio2013estimating}. It approximates the gradient during the backward pass to enable gradient-based optimization in networks with discrete activations.
If $\sigma(x) := (1+\exp(-x))^{-1}$ is the sigmoid function, \cite{corwin1994iterative} replaced in the backward pass the Heaviside activation with the scaled sigmoid function $\sigma(-\lambda x)$, 
\cite{darabi2018regularized} proposed the SignSwish function $2 \sigma(\lambda x)[1+\lambda x\{1-\sigma(\lambda x)\}]-1$, and \cite{hubara2018quantized} employed the HardTanh function $- \mathbb{I}(x < -1) + x  \mathbb{I}(-1 \leq x < 1) + \mathbb{I}(1 \leq x)$. 
Notably, \cite{hubara2018quantized} showed that their approach also works for large datasets (e.g., ImageNet \cite{deng2009imagenet}) and large-scale models (e.g., GoogleNet \cite{szegedy2015going}), implying that there is still potential for further advancements in training DHNs.
These algorithms can directly be adapted to train the structure-augmented DHNs proposed in this work.

\clearpage

\appendix

\section{Proofs of Section \ref{sec_noskip}}
\label{app_proof_thm_plain_include}

\begin{proof}[Proof of Theorem \ref{thm_plain_include}]
Using the recursive formula to define the neural network output in (\ref{DNN}) and (\ref{DNN_hidden}), $\bm{f}^{(1)}|_{[\bx_1, \bx_2]} : [0,1] \to \{0,1\}^{p_1}$ is defined by 
\begin{align*}
    \bm{f}^{(1)}\big|_{[\bx_1, \bx_2]}(t) :=
\mathbb{I}\Big(W_{0} \big((1-t)\bx_1 + t\bx_2\big) - \bm{b}_{0}\Big)
= \mathbb{I}\Big(W_{0} (\bx_2-\bx_1) t - \big(\bm{b}_{0}-W_0\bx_1\big)\Big).
\end{align*}
If we treat $t$ as the input variable, this defines a network in $\DHN(L,(1,p_1,\ldots,p_L,1))$ with first weight matrix $W_0(\bx_2-\bx_1)$ and first bias/shift vector $\bm{b}_{0}-W_0\bx_1.$

For each $i \in [p_1]$, the $i$-th component of $\bm{f}^{(1)}|_{[\bx_1, \bx_2]}(t)$ can takes values of the form $\mathbb{I}(t \geq c_i)$ or $\mathbb{I}(t \leq c_i)$ for some constant $c_i \in [0,1]$.
Hence, $\bm{f}^{(1)}|_{[\bx_1, \bx_2]}$ is a piecewise constant vector-valued function with at most $p_1$ breakpoints.
In other words, there exists an  interval partition $A_1, \ldots, A_{K}$ of $[0,1]$ with $K \leq p_1+1$ such that $\bm{f}^{(1)}|_{[\bx_1, \bx_2]}(t)$ is a constant (vector-valued) function on each element of the partition.

For any $k \in [K]$ and any $t_1, t_2 \in A_k$, we have  $\bm{f}^{(1)}|_{[\bx_1, \bx_2]}(t_1) = \bm{f}^{(1)}|_{[\bx_1, \bx_2]}(t_2)$. The activation of the first hidden layer will be the input of the next layer. A neural network with the same input will also have the same output. Thus, $\bm{f}^{(1)}|_{[\bx_1, \bx_2]}(t_1) = \bm{f}^{(1)}|_{[\bx_1, \bx_2]}(t_2)$ implies $f|_{[\bx_1, \bx_2]}(t_1) = f|_{[\bx_1, \bx_2]}(t_2)$.

Hence, $f|_{[\bx_1, \bx_2]}$ is a piecewise constant function with $K\leq p_1+1$ pieces.

We obtain the second assertion of the theorem by
    \begin{align*}
        \inf_{f \in \DHN(L,\bm{p})} \|f-f_0\|_{\infty}
        &\geq \sup_{\bx_1, \bx_2 \in [0,1]^d} \,
        \inf_{f \in \DHN(L,\bm{p})} \Big\|f|_{[\bx_1, \bx_2]}-f_0|_{[\bx_1, \bx_2]}\Big\|_{\infty} \\
        &\geq \sup_{\bx_1, \bx_2 \in [0,1]^d} \,
        \inf_{f \in \PC(p_1 + 1)} \Big\|f -f_0|_{[\bx_1, \bx_2]}\Big\|_{\infty}.
    \end{align*}    
\end{proof}

\medskip

\begin{proof}[Proof of (\ref{inequ_delta_supinf})]
    Let $\bx_{\max}, \bx_{\min} \in [0,1]^d$ be the points where 
    $f_0(\bx_{\max}) = \sup_{\bx\in [0,1]^d} f_0(\bx)$ and 
    $f_0(\bx_{\min}) = \inf_{\bx\in [0,1]^d} f_0(\bx)$, respectively.
    
    Every $f \in \PC(m)$ can be represented as
    $f(x) = \sum_{i=1}^m c_i \mathbb{I}(x \in A_i)$ for some  interval partition $\{A_1,A_2,\ldots,A_m\}$ of $[0,1]$ and some real numbers $c_1,c_2,\ldots,c_m$.    
    There exists $y_0 \in [f_0(\bx_{\min}), f_0(\bx_{\max})]$ such that 
    \begin{align*}   
         \min_{i \in [m]} \, \left|c_i - y_0\right| 
         \geq  \frac{ f_0(\bx_{\max}) - f_0(\bx_{\min}) }{2 m}.
    \end{align*}
    Since $f_0|_{[\bx_{\min}, \bx_{\max}]}$ is continuous, 
    there exists $t_0 \in [0,1]$ such that $f_0|_{[\bx_{\min}, \bx_{\max}]}(t_0)= y_0$. Hence, we have 
\begin{align}
    \sup_{t \in [0,1]} \Big|f(t)-f_0|_{[\bx_1, \bx_2]}(t) \Big| 
    &\geq \Big|f(t_0)-f_0|_{[\bx_1, \bx_2]}(t_0) \Big| \nonumber \\
    &\geq \frac{ f_0(\bx_{\max}) - f_0(\bx_{\min}) }{2 m}. \label{tmp_8}
\end{align}
Since (\ref{tmp_8}) holds for any $f \in \PC(m)$, we get
\begin{align*}
    \sup_{\bx_1, \bx_2 \in [0,1]^d} \, \inf_{f \in \PC(m)} \left\|f-f_0 |_{[\bx_1, \bx_2]} \right\|_{\infty} 
    &\geq \inf_{f \in \PC(m)} \, \sup_{t \in [0,1]} \, \Big|f(t)-f_0|_{[\bx_{\min}, \bx_{\max}]}(t) \Big| \\
    & \geq \frac{\sup_{\bx\in [0,1]^d} f_0(\bx) - \inf_{\bx\in [0,1]^d} f_0(\bx)}{2m}.
\end{align*}
\end{proof}

\medskip

\begin{lemma} \label{lemma_onedim_equ}
    For any $p = 1,2,\ldots,$ $$\DHN \big(1,(1,p,1)\big) = \PC(p+1).$$
\end{lemma}
\begin{proof}
    The inclusion $\DHN(1,(1,p,1)) \subseteq \PC(p+1)$ follows from Theorem \ref{thm_plain_include}.
    
    Every $f \in \PC(p+1)$ can be represented as
    $$f(x) = \sum_{i=0}^{p} c_i \mathbb{I}(x \in A_i)$$
    for an interval partition $\{A_0, A_1, \ldots, A_{p}\}$ of $[0,1]$ and constants $c_0,c_1,\ldots,c_{p}$. 
    Without loss of generality, assume that $A_0, A_1, \ldots, A_{p}$ are ordered. 
    There exist breakpoints $0 \leq x_0 \leq x_1 \leq \ldots \leq x_p \leq x_{p+1} =1$, binary values $b_1, b_2, \ldots,b_p \in \{0,1\}$, $b_0 = 1$ and $b_{p+1}=-1$ such that
    $$A_i = \begin{cases}
        [x_{i}, x_{i+1}), \qquad& \text{if } (b_{i}, b_{i+1})= (1,1) \\
        [x_{i}, x_{i+1}], \qquad& \text{if } (b_{i}, b_{i+1})= (1,-1),  \\
        (x_{i}, x_{i+1}), \qquad& \text{if } (b_{i}, b_{i+1})= (-1,1),  \\
        (x_{i}, x_{i+1}], \qquad& \text{if } (b_{i}, b_{i+1})= (-1,-1)
    \end{cases} $$
    for every $i \in \{0,1,,\ldots,p\}$.
    That is $b_i=1$ when $x_i$ is included in $A_i$, while $b_i=-1$ when $x_i$ is included in $A_{i-1}$.
    For every $x \in [0,1]$, we have
    \begin{align*}
        f(x) &= 
        c_0 + \sum_{i=1}^{p} (c_{i} - c_{i-1}) \mathbb{I}\big( x \in \cup_{k=i}^p A_k \big)  \\
        &= 
        c_0 + \sum_{i=1}^{p} (c_{i} - c_{i-1}) \Big( \mathbb{I}(x>x_i) + \mathbb{I}(b_i = 1) \mathbb{I}(x=x_i) \Big).
    \end{align*}
    Since $\mathbb{I}(x>x_i) + \mathbb{I}(b_i = 1) \mathbb{I}(x=x_i) = \mathbb{I}(b_i = 1) \mathbb{I}(x \geq x_i)
        + \mathbb{I}(b_i = -1) (1-\mathbb{I}(x \leq x_i))$, we get
    \begin{align*}
        f(x) &= c_0 + \sum_{i=1}^{p} (c_{i} - c_{i-1}) \Big( \mathbb{I}(b_i = 1) \mathbb{I}(x \geq x_i)
        + \mathbb{I}(b_i = -1) (1-\mathbb{I}(x \leq x_i)) \Big) \\
        &= \left( c_0 + \sum_{i=1}^{p} (c_{i} - c_{i-1}) \mathbb{I}(b_i = -1)  \right) 
        + \sum_{i=1}^{p} b_i (c_{i} - c_{i-1})  \mathbb{I}(b_i x \geq b_i x_i), 
    \end{align*}
    which can be represented by a function in $\DHN(1,(1,p,1))$.
    Hence, we obtain $\PC(p+1) \subseteq \DHN(1,(1,p,1)) $.     
\end{proof}

\medskip

\begin{proof}[Proof of Proposition  \ref{low_1_tight}]

\textbf{(i)}
If $\max_{\bx \in [0,1]^d} A_0(\bx) = \min_{\bx \in [0,1]^d} A_0(\bx)$, then $f_0$ is a constant function and hence $\inf_{f \in \DHN(1, (d,p_1,1))} \|f-f_0\|_{\infty} = 0$.

If $\max_{\bx \in [0,1]^d} A_0(\bx) - \min_{\bx \in [0,1]^d} A_0(\bx) > 0$, consider affine transformation 
$$B_0(\bx) := \frac{A_0(\bx) - \min_{\bx \in [0,1]^d} A_0(\bx)}{\max_{\bx \in [0,1]^d} A_0(\bx) - \min_{\bx \in [0,1]^d} A_0(\bx)}$$
and univariate function
$$h_0(t) := g_0\left( \min_{\bx \in [0,1]^d} A_0(\bx) + t \cdot \left( \max_{\bx \in [0,1]^d} A_0(\bx) - \min_{\bx \in [0,1]^d} A_0(\bx)\right)  \right),$$
which satisfy $h_0 \circ B_0(\bx) = g_0 \circ A_0(\bx) = f_0(\bx)$ for every $\bx \in [0,1]^d$.

Since the composition of two affine transformations is also an affine transformation, 
we have 
$$\Big\{h \circ B_0 : h \in \DHN\big(1, (1,p_1,1)\big) \Big\} \subset \DHN\big(1, (d,p_1,1)\big).$$
Hence, we have
\begin{align*}
    \inf_{f \in \DHN(1, (d,p_1,1))} \|f-f_0\|_{\infty} 
    &=
    \sup_{\bx \in [0,1]^d} \, \inf_{f \in \DHN(1, (d,p_1,1))} \big|f(\bx)-h_0 \circ B_0(\bx)\big|\\
    & \leq \sup_{\bx \in [0,1]^d} \, \inf_{h \in \DHN(1, (1,p_1,1))} \big|h \circ B_0(\bx) - h_0 \circ B_0(\bx)\big| \\
    & = \sup_{t \in [0,1]} \, \inf_{h \in \DHN(1, (1,p_1,1))} \big|h(t) -h_0(t)\big|\\
    & = \sup_{t \in [0,1]} \, \inf_{h \in \PC(p_1+1)} \big|h(t) -h_0(t)\big|,
\end{align*}
where the last equality holds by Lemma \ref{lemma_onedim_equ}.

Since $B_0$ has been normalized we can find vectors $\bx_{\min}, \bx_{\max} \in [0,1]^d$ such that
$B_0(\bx_{\min}) = 0$
and
$B_0(\bx_{\max}) = 1$.
We have
\begin{align*}
    f_0|_{[\bx_{\min}, \bx_{\max}]}(t)
    &= h_0 \circ B_0\big( (1-t)\bx_{\min} + t \bx_{\max}\big)\\
    &= h_0 \circ \big((1-t) B_0( \bx_{\min} ) + t B_0( \bx_{\max} ) \big)\\
    & = h_0(t).
\end{align*}
Hence, we get
\begin{align*}
    \inf_{f \in \DHN(1, (d,p_1,1))} \|f-f_0\|_{\infty} 
    &\leq \sup_{t \in [0,1]} \, \inf_{h \in \PC(p_1+1)} \big|h(t) -h_0(t)\big| \\
    &= \sup_{t \in [0,1]} \, \inf_{h \in \PC(p_1+1)} \big|h(t) -f_0|_{[\bx_{\min}, \bx_{\max}]}(t)\big|\\
    & \leq \sup_{\bx_1, \bx_2 \in [0,1]^d} \, \inf_{h \in \PC(p_1 + 1)} \left\|h-f_0 |_{[\bx_1, \bx_2]} \right\|_{\infty}.
\end{align*}

\medskip

\noindent
\textbf{(ii)}
    Let $M := \lceil p_1/d \rceil$.
    We partition $[0,1]^d$ into $M^d$ hypercubes,     
    each with a side length of $1/M$.
    For $a_0 := -\infty$, $a_n := n/M$ for $n \in [M-1]$ and $a_M := \infty$,
    each hypercube can be represented by
    $$\Big\{\bx = (x_1,\ldots,x_d)^{\top} \in [0,1]^d : \forall i \in [d],\  a_{n_i} \leq x_i < a_{n_i + 1} \Big\}$$     
    for some $n_1, \ldots, n_{d} \in \{0, 1, \ldots, M-1\}$. 
    We have
    \begin{align}
        &\mathbb{I} \big( \forall i \in [d], \ a_{n_i} \leq x_i < a_{n_i + 1} \big) \nonumber \\
        &= \mathbb{I} \left( \sum_{i=1}^d \mathbb{I} \left( a_{n_i} \leq x_i < a_{n_i + 1} \right) - \left(d - \frac{1}{2}\right)  \right) \nonumber \\
        & = \mathbb{I} \left( \sum_{i=1}^d \big( \mathbb{I}\left( x_i - a_{n_i} \right) -  \mathbb{I}\left( x_i - a_{n_i + 1} \right) \big) - \left(d - \frac{1}{2}\right)  \right).
        \label{tmp_7}
    \end{align} 
    We denote this partition by $\mathcal{U}:=\{U_{k}\}_{k \in [M^d]}$, and let $\bm{u}_{k} \in [0,1]^d$ be the center of $U_{k}$. 
    For $\bx \in [0,1]^d$, 
    we denote by $\bm{u}(\bx)$ the center of the hypercube that includes $\bx$.

    We construct $f \in \DHN(2,(d,p_1,M^d, 1))$ as following. 
   In the first hidden layer, we represent 
   $$\big(\mathbb{I}\left(x_i - a_n \right)\big)_{i \in [d], n \in [M - 1]}.$$
    We need at most $d(\lceil p_1/d \rceil - 1) \leq p_1$ neurons in the first hidden layer. 
    We do not need to represent 
    $\mathbb{I}(x_i - a_0)$ and $\mathbb{I}(x_i - a_M)$ because their values are always $1$ and $0$, respectively.
   In the second hidden layer, we represent
   $$\big(\mathbb{I}\left(\bx \in U_k \right)\big)_{k \in [M^d]}$$
   using (\ref{tmp_7}).   
   We need $M^d$ neurons in the second hidden layer. 
Then, we define
   $$f(\bx) := f_0\big(\bm{u}(\bx)\big) = \sum_{k=1}^{M^d} f_0(\bm{u}_k) \mathbb{I}(\bx \in U_k).$$
   From $\|\bx - \bm{u}(\bx) \|_{\infty} \leq 1/(2M)$, we have
   \begin{align*}
       &\sup_{\bx \in [0,1]^d} |f(\bx)- f_0(\bx)| \\
       &=\sup_{\bx \in [0,1]^d} |f_0\big(\bm{u}(\bx)\big)- f_0(\bx)| \\
       &\leq \frac{k}{2 M} \\
       &\leq \frac{k d}{2p_1} \\
       &\leq \frac{k d}{p_1+1} \\     
       &\leq \frac{k d}{\sup_{\bx \in [0,1]^d} f_0(\bx) - \inf_{\bx \in [0,1]^d} f_0(\bx)} \sup_{\bx_1, \bx_2 \in [0,1]^d} \, \inf_{f \in \PC(p_1 + 1)} \left\|f-f_0 |_{[\bx_1, \bx_2]} \right\|_{\infty},
   \end{align*}
   where we use (\ref{inequ_delta_supinf}) for the last inequality.
\end{proof}

\section{Proofs of Section \ref{sec_skip}} \label{app_proof_skip}

\subsection{Proof of Theorem \ref{thm_skip_include}} \label{app_proof_skip_1}

The $j$-th component of a vector $\bx$ is denoted by $(\bx)_{(j)}.$

\medskip

\begin{proof}
Using the recursive formula to define the neural network output in (\ref{DHN_skip_multi}), (\ref{DHN_skip_multi_hidden}) and (\ref{DHN_skip_multi_first}), $\bm{f}^{(1)}|_{[\bx_1, \bx_2]} : [0,1] \to \{0,1\}^{p_1}$ is defined by 
\begin{align*}
    \bm{f}^{(1)}\big|_{[\bx_1, \bx_2]}(t) :=
\sigma_0 \Big(W_{0} \big((1-t)\bx_1 + t\bx_2\big) - \bm{b}_{0}\Big).
\end{align*}
For each $i \in [p_1]$, the $i$-th component of $\bm{f}^{(1)}|_{[\bx_1, \bx_2]}(t)$ can take values of the form $\mathbb{I}(t \geq c_i)$ or $\mathbb{I}(t \leq c_i)$ for some constant $c_i \in [0,1]$.
Hence, $\bm{f}^{(1)}|_{[\bx_1, \bx_2]}$ is a piecewise constant (vector-valued) function with at most $p_1$ breakpoints, and thus consists of at most $p_1+1$ pieces.

For simplicity, 
we write $s_1 := p_1$ and $S_\ell:=\prod_{r=1}^\ell (s_r+1)$.
Assume that for $\ell \in [L-1]$, $\bm{f}^{(\ell)}|_{[\bx_1, \bx_2]}$ is a piecewise constant vector-valued function with at most $S_\ell$ pieces on $[0,1]$. 
In other words, there exists an interval partition $A_1, \ldots, A_{K}$ of $[0,1]$ with $K \leq S_\ell$ such that $\bm{f}^{(\ell)}|_{[\bx_1, \bx_2]}(t)$ is a constant vector-valued function in $t$ on each element of the partition.
Hence, there exists $\by_1, \by_2, \ldots, \by_K \in \mathbb{R}^{p_{\ell}}$ such that 
$t \in A_k$ implies $ \bm{f}^{(\ell)}|_{[\bx_1, \bx_2]}(t) = \by_k$.
Let 
\begin{align*}
    W_{\ell} := 	
\begin{pmatrix}
\bw_{\ell,1}^{\top} \\
\bw_{\ell,2}^{\top} \\
\vdots \\
\bw_{\ell, p_{{\ell}+1}}^{\top}
\end{pmatrix}
, \qquad
    V_{\ell} := 	
\begin{pmatrix}
\bv_{\ell,1}^{\top} \\
\bv_{\ell,2}^{\top} \\
\vdots \\
\bv_{\ell, p_{{\ell}+1}}^{\top}
\end{pmatrix}
, \qquad
\bm{b}_{\ell} := \begin{pmatrix}
b_{\ell,1} \\
b_{\ell,2} \\
\vdots \\
b_{\ell, p_{{\ell}+1}}
\end{pmatrix},
\end{align*}
with $\bw_{\ell,1}, \ldots, \bw_{\ell, p_{{\ell}+1}} \in \mathbb{R}^{p_{{\ell}}}$, $\bv_{\ell,1}, \ldots, \bv_{\ell, p_{{\ell}+1}} \in \mathbb{R}^{d}$ and $b_{\ell,1}, b_{\ell,2}, \ldots, b_{\ell, p_{{\ell}+1}} \in \mathbb{R}$.
For each $k \in [K]$, the $i$-th component of $\bm{f}^{(\ell+1)}\big|_{[\bx_1, \bx_2]}(t)$ for $t \in A_k$ is 
\begin{align}
    \left(\bm{f}^{(\ell+1)}\big|_{[\bx_1, \bx_2]}(t)\right)_{(i)} &=
    \left( \bm{f}^{(\ell+1)}\big((1-t)\bx_1 + t \bx_2\big)\right)_{(i)} \nonumber \\
    &=
    \bigg( \sigma_0\left(W_{\ell} \bm{f}^{(\ell)}\big((1-t)\bx_1 + t\bx_2\big) + V_{\ell} \big((1-t)\bx_1 + t\bx_2\big) - \bm{b}_{\ell}\right) \bigg)_{(i)} \nonumber \\
    &=
    \bigg( \sigma_0\left(W_{\ell} \bm{f}^{(\ell)}\big|_{[\bx_1, \bx_2]}(t) 
    + V_{\ell} \big((1-t)\bx_1 + t\bx_2\big) 
    - \bm{b}_{\ell}\right) \bigg)_{(i)} \nonumber \\    
    &=
    \mathbb{I}\left(\bw_{\ell, i}^{\top} 
 \bm{f}^{(\ell)}\big|_{[\bx_1, \bx_2]}(t)
    + \bv_{\ell, i}^{\top} \big((1-t)\bx_1 + t\bx_2\big)
    - b_{\ell, i} \right) \nonumber \\
    &=
    \mathbb{I}\Big(
      \bv_{\ell, i}^{\top}(\bx_2-\bx_1) t
      + \bw_{\ell, i}^{\top} \by_k
    + \bv_{\ell, i}^{\top} \bx_1 
    - b_{\ell, i} \Big) \label{tmp_9}.
\end{align}
If $\bv_{\ell, i} \neq \bm{0}_d$, (\ref{tmp_9}) takes values of the form 
$\mathbb{I}(t \geq c_i)$ or $\mathbb{I}(t \leq c_i)$ for some constant $c_i \in [0,1]$.
On the other hand, if $\bv_{\ell, i} = \bm{0}_d$, (\ref{tmp_9}) remains constant. 
Since at most $s_{\ell+1}$ vectors among $\bv_{\ell,1}, \ldots, \bv_{\ell, p_{{\ell}+1}}$ are nonzero, 
$\bm{f}^{(\ell+1)}|_{[\bx_1, \bx_2]}(t)$ is a piecewise constant vector-valued function with at most $s_{\ell+1}$ breakpoints on $t \in A_k$, and thus consists of at most $s_{\ell+1}+1$ pieces.
As a result, $\bm{f}^{(\ell)}|_{[\bx_1, \bx_2]}$ is a piecewise constant vector-valued function with at most $K (s_{\ell+1}+1) \leq S_{\ell+1}$ pieces on [0,1]. 
Iterating this process,  
$\bm{f}^{(L)}|_{[\bx_1, \bx_2]}$ is a piecewise constant vector-valued function with at most $S_L$ pieces, and $f|_{[\bx_1, \bx_2]}$ is a piecewise constant function with the same bound on the number of pieces.

We obtain the second assertion of the theorem by
    \begin{align*}
        \inf_{f \in \DHN_{\skipp}(L,\bm{p},\bm{s})} \|f-f_0\|_{\infty}
        &\geq \sup_{\bx_1, \bx_2 \in [0,1]^d} \,
        \inf_{f \in \DHN(L,\bm{p},\bm{s})} \Big\|f|_{[\bx_1, \bx_2]}-f_0|_{[\bx_1, \bx_2]}\Big\|_{\infty} \\
        &\geq \sup_{\bx_1, \bx_2 \in [0,1]^d} \,
        \inf_{f \in \PC\left(S_L \right)} \Big\|f -f_0|_{[\bx_1, \bx_2]}\Big\|_{\infty}.
    \end{align*} 
\end{proof}

\subsection{Proof of Theorem \ref{thm_sq}} \label{app_proof_sq}

We approximate the square function using the \textit{mixed {  radix} numerical representation}, that generalizes the positional numerical system.  
The most popular positional numerical system is the decimal numerical system.
The decimal numerical representation of $x=1/7$ is
$$x = 0.142\ldots_{(10)} := \frac{1}{10} +  \frac{4}{10^2} + \frac{2}{10^3}+ \ldots,$$
with recursively chosen coefficients
\begin{align*}
&\left\lfloor 10 \cdot x \right\rfloor = 1, \quad
\left\lfloor 10^2 \cdot 
\left( x-\frac{1}{10}\right) \right\rfloor = 4,
\quad
\left\lfloor 10^3 \cdot 
\left( x- \frac{1}{10} -  \frac{4}{10^2}\right) \right\rfloor = 2, \ldots 
\end{align*}
Another system is the binary numerical system. 
The binary numeral representation of $x=1/7$ is
\[x = 0.00100\ldots_{(2)} := \frac{0}{2} +  \frac{0}{2^2} + \frac{1}{2^3}+ \frac{0}{2^4}+ \frac{0}{2^5}+ \ldots,\]
with recursively chosen coefficients
\begin{align*}
&\left\lfloor 2 \cdot x \right\rfloor = 0, \quad
\left\lfloor 2^2 \cdot 
\left( x-\frac{0}{2}\right) \right\rfloor = 0,
\quad
\left\lfloor 2^3 \cdot 
\left( x- \frac{0}{2} -  \frac{0}{2^2}\right) \right\rfloor = 1, \ldots
\end{align*}

The mixed {  radix} numerical representation considers a different base for each digit.
That is, 
for any given positive integer sequence $d_1,d_2,d_3,d_4,\ldots$, the mixed radix numerical representation of $x \in [0,1]$ is defined by
$$x = 0.b_1(x) b_2(x) b_3(x)  \ldots_{(d_1,d_2,d_3,\ldots)}:= 
\sum_{\ell=1}^{\infty} \frac{{b_\ell}(x)}{\prod_{i=1}^{\ell} d_i},$$
where the coefficients $b_{\ell}(x) \in \{0,1,\ldots, d_{\ell}-1\}$ are recursively defined by
\begin{align*}    
    b_1(x):=\left\lfloor d_1 x\right\rfloor, \ 
    b_2(x):=\left\lfloor d_1 d_2 
\left( x-\frac{b_1(x)}{d_1}\right) \right\rfloor, \ 
    b_3(x):=\left\lfloor d_1 d_2 d_3 
\left( x-\frac{b_1(x)}{d_1} - \frac{b_2(x)}{d_1 d_2}\right) \right\rfloor,
\end{align*}
and so on, if $x \in [0,1)$, while $b_\ell(x) := d_{\ell}-1$ for all $\ell = 1,2,\ldots,$ if $x=1$.
Both the decimal and binary numerical systems are special cases of mixed radix numerical systems, corresponding to $d_1=d_2=\ldots=10$ and $d_1=d_2=\ldots=2$, respectively.

Truncating the mixed radix numerical representation of $x\in [0,1]$ after $L$ digits
$$\widetilde{x} := 0.{b_1(x) b_2(x) \ldots b_L(x)}_{(d_1, d_2, \ldots, d_L)} := \sum_{\ell=1}^{L} \frac{{b_\ell}(x)}{\prod_{i=1}^{\ell} d_i},$$
satisfies $\widetilde{x} \leq x \leq \widetilde{x} + 1/(\prod_{\ell=1}^L d_\ell)$.
We call these $b_1(x),\ldots,b_L(x)$ the mixed radix coefficients of $x$ for the radix vector $(d_1,\ldots,d_L)$.

\medskip

\begin{lemma} \label{lemma_bit_extract}
     For any integers $L \geq 1, p_1 \geq 1$ and $\bm{s} = (s_2,\ldots,s_L) \in \mathbb{N}^{L-1}$,
     we define $b_1(x), b_2(x), \ldots, b_L(x)$
as the mixed radix coefficients of $x$ for the radix vector $(p_1+1, s_2+1, \ldots, s_{L}+1)$.
There exists a network
\begin{align*}
    f \in \DHN_{\skipp}(L, \bm{p}, \bm{s})
\end{align*}
with width vector
    \begin{align*}
        \bm{p} &:= \left(1,p_1,p_1 + s_2, \ldots, p_1 + \sum_{\ell = 2}^{L} s_{\ell}, 1 \right)
    \end{align*}
such that for any input $x \in [0,1]$, the output of the last hidden layer is
$$\Big(\mathbb{I}(b_{\ell}(x) \geq t )\Big)_{\ell \in [L], t \in [s_\ell]}.$$
\end{lemma}
\begin{proof}

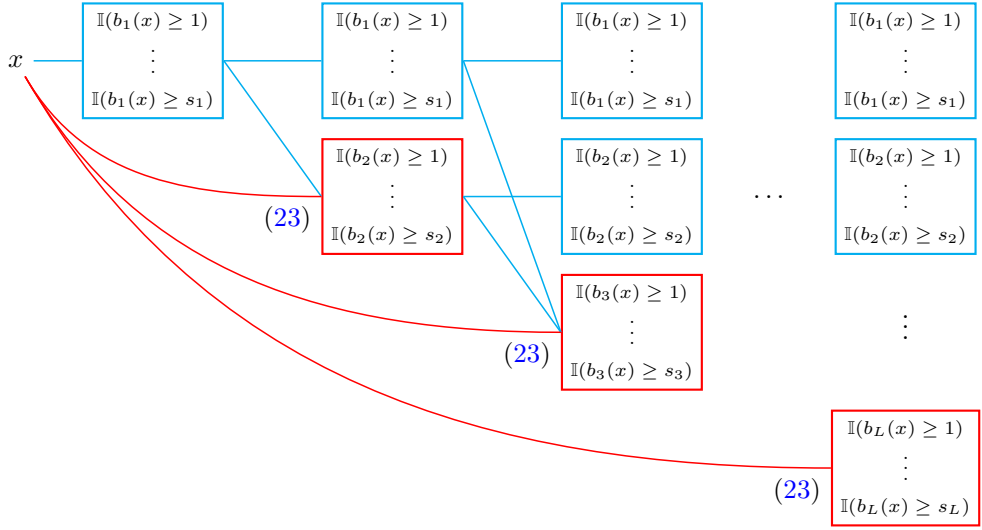
\begin{figure}[t]
    \centering
        \begin{tikzpicture}[node distance=2cm, auto, thick, >=Stealth, scale=0.9]
        \node[draw=white] (x) at (-0.5, 4.5) {$x$};        

        \node[draw=cyan, align=center] (f11) at (1.5, 4.5) {\footnotesize 
        $\mathbb{I}(b_1(x) \geq 1)$\\ 
        \footnotesize$\vdots$ \\
        \footnotesize$\mathbb{I}(b_1(x) \geq s_1)$};
        
        \node[draw=cyan, align=center] (f21) at (5, 4.5) {\footnotesize 
        $\mathbb{I}(b_1(x) \geq 1)$\\ 
        \footnotesize$\vdots$ \\
        \footnotesize$\mathbb{I}(b_1(x) \geq s_1)$};
        \node[draw=red, align=center] (f22) at (5, 2.5) {\footnotesize 
        $\mathbb{I}(b_2(x) \geq 1)$\\ 
        \footnotesize$\vdots$ \\
        \footnotesize$\mathbb{I}(b_2(x) \geq s_2)$};
        
        \node[draw=cyan, align=center] (f31) at (8.5, 4.5) {\footnotesize 
        $\mathbb{I}(b_1(x) \geq 1)$\\ 
        \footnotesize$\vdots$ \\
        \footnotesize$\mathbb{I}(b_1(x) \geq s_1)$};
        \node[draw=cyan, align=center] (f32) at (8.5, 2.5) {\footnotesize 
        $\mathbb{I}(b_2(x) \geq 1)$\\ 
        \footnotesize$\vdots$ \\
        \footnotesize$\mathbb{I}(b_2(x) \geq s_2)$};
        \node[draw=red, align=center] (f33) at (8.5, 0.5) {\footnotesize 
        $\mathbb{I}(b_3(x) \geq 1)$\\ 
        \footnotesize$\vdots$ \\
        \footnotesize$\mathbb{I}(b_3(x) \geq s_3)$};

        \node at (10.5, 2.5) {$\ldots$};
        \node at (12.5, 0.7) {$\vdots$};

        \node[draw=cyan, align=center] (f41) at (12.5, 4.5) {\footnotesize 
        $\mathbb{I}(b_1(x) \geq 1)$\\ 
        \footnotesize$\vdots$ \\
        \footnotesize$\mathbb{I}(b_1(x) \geq s_1)$};
        \node[draw=cyan, align=center] (f42) at (12.5, 2.5) {\footnotesize 
        $\mathbb{I}(b_2(x) \geq 1)$\\ 
        \footnotesize$\vdots$ \\
        \footnotesize$\mathbb{I}(b_2(x) \geq s_2)$};
        \node[draw=red, align=center] (f43) at (12.5, -1.5) {\footnotesize 
        $\mathbb{I}(b_L(x) \geq 1)$\\ 
        \footnotesize$\vdots$ \\
        \footnotesize$\mathbb{I}(b_L(x) \geq s_L)$ };

        \draw[-, cyan, line width=0.2mm] (x.east) -- (f11.west);

        \draw[-, cyan, line width=0.2mm] (f11.east) -- (f21.west) ;
        \draw[-, cyan, line width=0.2mm] (f11.east) -- (f22.west);

        \draw[-, cyan, line width=0.2mm] (f21.east) -- (f31.west);
        \draw[-, cyan, line width=0.2mm] (f22.east) -- (f32.west);
        \draw[-, cyan, line width=0.2mm] (f21.east) -- (f33.west);
        \draw[-, cyan, line width=0.2mm] (f22.east) -- (f33.west);


        \draw[-, thick, red, line width=0.2mm] (x) to[out=300, in=180] (f22);
        \draw[-, thick, red, line width=0.2mm] (x) to[out=300, in=180] (f33);
        \draw[-, thick, red, line width=0.2mm] (x) to[out=300, in=180] (f43);

        \node[below left, align=center] at (f22.west) 
       {(\ref{tmp_21})};

        \node[below left, align=center] at (f33.west) 
       {(\ref{tmp_21})};

        \node[below left, align=center] at (f43.west) 
       {(\ref{tmp_21})};

    \end{tikzpicture} \vspace{15pt}
    \caption{\textbf{Illustration for the proof of Lemma \ref{lemma_bit_extract}.}  Neurons in red squares are skip connected to the input, while those in blue squares are not.}
    \label{fig:binary}
\end{figure}

For simplicity, 
we denote $s_1 := p_1$ and $S_\ell:=\prod_{r=1}^\ell (s_r+1)$.
Since we consider one-dimensional input, the matrix $V_{\ell-1}$ in the layer recursion formula is a $(p_1 + \sum_{r=2}^\ell s_r)$-dimensional vector, and will be defined as
$$V_{\ell-1} := (\underbrace{0,\ldots,0}_{\sum_{r=1}^{\ell-1} s_{r}  }, \underbrace{1,\ldots,1}_{s_\ell})^\top.$$

In the first hidden layer, we assign $s_1$ neurons to the respective values $\mathbb{I}(b_{1}(x) \geq t) = \mathbb{I}(x-t/(s_1+1))$ for each $t \in [s_1]$.
Assume now that the $\ell$-th hidden layer with $1 \leq \ell \leq L-1$ has $s_1 + \ldots + s_\ell$ neurons whose values are 
$(\mathbb{I}(b_{r}(x) \geq t))_{r \in [\ell], t \in [s_r]}$. 
Then, for the $(\ell+1)$-st hidden layer, we can forward $(\mathbb{I}(b_{r}(x) \geq t))_{r \in [\ell], t \in [s_r]}$ from the previous layer using the identity $b = \mathbb{I}(b - 1/2)$ for $b \in \{0,1\}$. 
This will require $\sum_{r=1}^\ell s_r$ neurons. 
In addition, for each $t \in [s_{\ell+1}]$ we can represent
\begin{align}
    \mathbb{I}(b_{\ell+1}(x) \geq t)
    &= \mathbb{I} \left( S_{\ell+1} \left(x - \frac{b_1(x)}{S_1} - \ldots - \frac{b_{\ell}(x)}{S_{\ell}}\right) \geq t\right) \nonumber \\
    &= \mathbb{I} \left(x - \frac{b_1(x)}{S_1} - \ldots - \frac{b_{\ell}(x)}{S_{\ell}} \geq \frac{t}{S_{\ell+1}}\right) \nonumber \\
    &= \mathbb{I} \left(x - \frac{\sum_{j=1}^{s_1} \mathbb{I}(b_{1}(x) \geq j) }{S_1} - \ldots - \frac{\sum_{j=1}^{s_\ell} \mathbb{I}(b_{\ell}(x) \geq j)}{S_{\ell}} - \frac{t}{S_{\ell+1}}\right) \label{tmp_21}
\end{align}
using one skip connection.
This will require $s_{\ell+1}$ skip connected neurons.

Iterating this process, for every $\ell \in [L]$, we have $\sum_{r=1}^{\ell} s_r$ neurons with $s_{\ell}$ skip connected neurons in the $\ell$-th hidden layer, whose outputs are $(\mathbb{I}(b_{r}(x) \geq t))_{r \in [\ell], t \in [s_{r}]}.$ 
For $\ell=L,$ the outputs are $(\mathbb{I}(b_{\ell}(x) \geq t))_{\ell \in [L], t \in [s_{\ell}]}$. 
\end{proof}

\medskip

\begin{proof}[Proof of Theorem \ref{thm_sq}]
For simplicity, 
we set $s_1 := p_1$ and $S_\ell:=\prod_{r=1}^\ell (s_r+1)$.
For $x \in [0,1]$, we define $b_1(x), b_2(x), \ldots, b_{L-1}(x)$
as the mixed radix coefficients of $x$ for the radix vector $(s_1+1, \ldots, s_{L-1}+1)$.
Then, $b_\ell(x) \in \{0,\ldots,s_\ell\}$ for $\ell \in [L-1]$, and  
$$\widetilde{x} := \sum_{\ell=1}^{L-1} \frac{b_\ell(x)}{S_{\ell}}$$
satisfies $\widetilde{x} \leq x \leq \widetilde{x}+ 1/S_{L-1}.$ Thus
$$\sup_{x \in [0,1]} \, \left|\left(\widetilde{x}+\frac{1}{2S_{L-1}}\right)-x\right| \leq \frac{1}{2 S_{L-1}}.$$

First, we construct the the first $(L-1)$ hidden layers using Lemma \ref{lemma_bit_extract}.
For every $\ell \in [L-1]$, we have $\sum_{r=1}^{\ell} s_r$ neurons with $s_{\ell}$ skip connected neurons in the $\ell$-th hidden layer. The output of the $(L-1)$-st hidden layer is 
$(\mathbb{I}(b_{\ell}(x) \geq t))_{\ell \in [L-1], t \in [s_{\ell}]}.$

In the $L$-th layer, we construct 
$$\Big( \mathbb{I}(b_{\ell_1}(x) \geq t_1) \mathbb{I}(b_{\ell_2}(x) \geq t_2)  \Big)_{\ell_1,\ell_2 \in [L-1], t_1 \in [s_{\ell_1}], t_2 \in [s_{\ell_2}]}$$ 
by using the identity $bc = \mathbb{I}(b + c - 3/2)$ for $b,c \in \{0,1\}$.
Thus, the $L$-th hidden layer has $(\sum_{\ell = 1}^{L-1} s_{\ell})(\sum_{\ell = 1}^{L-1} s_{\ell}+1)/2$ neurons. Since
\begin{align*}
    f(x):=&\left(\widetilde{x}+\frac{1}{2S_{L-1}}\right)^2 \\
    = &\left(\sum_{\ell=1}^{L-1} \frac{b_\ell(x)}{S_{\ell}} 
+ \frac{1}{2 S_{L-1}}\right)^2 \\
    = &\left(\sum_{\ell=1}^{L-1} \frac{\sum_{t=1}^{s_\ell} \mathbb{I}(b_{\ell}(x) \geq t)}{S_{\ell}} 
+ \frac{1}{2 S_{L-1}}\right)^2
\end{align*}
is an affine function of $(\mathbb{I}(b_{\ell_1}(x) \geq t_1) \mathbb{I}(b_{\ell_2}(x) \geq t_2) )_{\ell_1,\ell_2 \in [L-1], t_1 \in [s_{\ell_1}], t_2 \in [s_{\ell_2}]}$, we have constructed a network $f$ satisfying 
$| x^2 - f(x)|
    = |(x-\widetilde{x}-1/(2S_{L-1}))(x+\widetilde{x}+1/(2S_{L-1}))|
    \leq 1/S_{L-1} = 1/S_L.$
\end{proof}

\subsection{Proof of Theorem \ref{thm_skip_vc_upper}} \label{app_proof_vc}
\begin{proof}[Proof of the upper bound in Theorem \ref{thm_skip_vc_upper}]
The considered class $\mathbb{I}(\DHN_{\skipp}(L,d\!:p\!:1,s))$ is a special case of the linear threshold networks in Section 6 of \cite{anthony1999neural}.
Let $W$ and $k$ denote the total number of parameters and computational nodes in $\mathbb{I}(\DHN_{\skipp}(L,d\!:p\!:1,s))$, respectively. 
Then, we have
$$W\leq (d+1)p+ (p^2+p+sd)(L-1)+(p+1), \qquad k=Lp+1.$$
By applying Theorem 6.1 of \cite{anthony1999neural}, we obtain that
\begin{align}
\operatorname{VC}\big(\DHN_{\skipp}(L,d:p:1,s)\big)&\leq2W\log\left(\frac{2k}{\ln{2}}\right).\label{a-p-bound}
\end{align}
Given $L\geq1$ and $p \geq d \vee s \vee 2$, it follows that
\begin{align}
    W&\leq(2p^2+p)(L-1)+(p+1)^2\nonumber\\
    &\leq(3p^2+3p+1)\max\{L-1,1\}\nonumber\\
    &\leq4.75p^2\max\{L-1,1\},\label{bound-forW}
\end{align}
and 
\begin{align}
    \log\left(\frac{2k}{\ln{2}}\right)
    &\leq\log\left(\frac{3}{\ln2}\right)+\log\left(Lp\right)
\leq3.12\log\left(Lp\right).\label{bound-fork}
\end{align}
Plugging \eqref{bound-forW} and \eqref{bound-fork} into \eqref{a-p-bound} yields
\begin{align*}
\operatorname{VC}\big(\DHN_{\skipp}(L,d:p:1,s)\big)\leq9.5p^2\max\{L-1,1\}\times3.12\log\left(Lp\right)
\leq30 Lp^2\log\left(Lp\right).
\end{align*}
\end{proof}

\medskip

The following lemma plays a central role in establishing both the lower bound on the VC dimension and the approximation of smooth functions (Theorem \ref{upper_smooth_skip}). We define 
$b_1(x), b_2(x), \ldots, b_\ell(x), \ldots \in \{0,1\}$
as the digits of a binary representation
\begin{align*}
    x = 0.b_1(x) b_2(x) b_3(x)  \ldots_{(2)}:= 
\sum_{\ell=1}^{\infty} 2^{-\ell} b_{\ell}(x).
\end{align*}
Further details can be found in Appendix \ref{app_proof_sq}.
In addition, for $\bx = (x_1,\ldots,x_d) \in [0,1]^d$,
we define 
\[b_\ell(\bx) := (b_\ell(x_i))_{i \in [d]}.\]

For any given positive integer $d$ and non-negative integers $m$ and $n$,
we define 
\begin{align}
    J:=2^{md}, \quad  K:=2^{md}, \quad \text{and} \ \   R:=2^{nd}.
    \label{eq.0fweni}
\end{align}
We partition $[0,1]^d$ into $\{U_j\}_{j \in [J]}$ such that $\bx$ and $\bz$ belong to the same partition element if and only if $b_\ell(\bx) = b_\ell(\bz)$ for all $\ell \in [m]$.       
Similarly, we partition $[0,1]^d$ into $\{V_k\}_{k \in [K]}$ such that $\bx$ and $\bz$ belong to the same partition element if and only if $b_\ell(\bx) = b_\ell(\bz)$ for all $\ell \in \{m+1, m+2, \ldots, 2m\}$. 
Finally, we partition $[0,1]^d$ into $\{W_r\}_{r \in [R]}$ such that $\bx$ and $\bz$ belong to the same partition element if and only if $b_\ell(\bx) = b_\ell(\bz)$ for all $\ell \in \{2m+1, 2m+2, \ldots, 2m+n\}$. 

For each $\bx \in [0,1]^d$, there exist unique $j(\bx) \in [J]$, $k(\bx) \in [K]$ and $r(\bx) \in [R]$ such that 
$$\bx \in U_{j(\bx)} \cap V_{k(\bx)} \cap W_{r(\bx)}.$$

\begin{lemma} \label{lemma_piece_bin_skip_ver3}

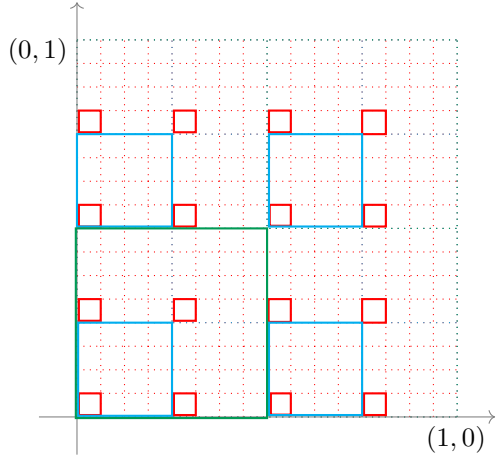
\begin{figure}[t]
    \centering
    \tikzset{every node/.style={scale=1}}
\begin{tikzpicture}[scale=5.0]
    \draw[step=0.0625cm,red,dotted] (0,0) grid (1,1); 
    \draw[step=0.25cm,cyan, dotted] (0,0) grid (1,1); 
    \draw[step=0.5cm,ForestGreen,dotted] (0,0) grid (1,1); 

    \draw[gray,->] (-0.1,0) -- (1.1,0);
    \draw[gray,->] (0,-0.1) -- (0,1.1);
    
    
    \draw[red, thick] (0.005,0.005) rectangle (0.0625,0.0625);
    \draw[red, thick] (0.255,0.005) rectangle (0.3125,0.0625);
    \draw[red, thick] (0.51,0.006) rectangle (0.5625,0.0625);
    \draw[red, thick] (0.755,0.005) rectangle (0.8125,0.0625);
    
    \draw[red, thick] (0.005,0.255) rectangle (0.0625,0.3125);
    \draw[red, thick] (0.255,0.255) rectangle (0.3125,0.3125);
    \draw[red, thick] (0.504,0.252) rectangle (0.5625,0.3125);
    \draw[red, thick] (0.75,0.25) rectangle (0.8125,0.3125);
    
    \draw[red, thick] (0.005,0.505) rectangle (0.0625,0.5625);
    \draw[red, thick] (0.255,0.505) rectangle (0.3125,0.5625);
    \draw[red, thick] (0.507,0.507) rectangle (0.5625,0.5625);
    \draw[red, thick] (0.755,0.505) rectangle (0.8125,0.5625);

    \draw[red, thick] (0.005,0.755) rectangle (0.0625,0.8125);
    \draw[red, thick] (0.255,0.755) rectangle (0.3125,0.8125);
    \draw[red, thick] (0.505,0.755) rectangle (0.5625,0.8125);
    \draw[red, thick] (0.75,0.75) rectangle (0.8125,0.8125);

    \draw[cyan, thick] (0.003,0.003) rectangle (0.25,0.25);
    \draw[cyan, thick] (0.505,0.005) rectangle (0.75,0.25);
    \draw[cyan, thick] (0,0.505) rectangle (0.25,0.75);
    \draw[cyan, thick] (0.505,0.505) rectangle (0.75,0.75);

    \draw[ForestGreen, thick] (-0.003,-0.003) rectangle (0.5,0.5);
    
    
    \node[anchor=north east] at (1.1,0) {$(1,0)$};
    \node[anchor=north east] at (0,1.03) {$(0,1)$};




\end{tikzpicture}
    \vspace{15pt}
    \caption{\textbf{Example of $U_j$, $V_k$ and $W_r$ in Lemma \ref{lemma_piece_bin_skip_ver3}.} In this case, $(d,m,n)=(2,1,2)$, and hence $(J,K,R)=(4,4,16)$. 
    The green, blue, and red solid squares represent $U_1$, $V_1$, and $W_1$, respectively. 
    For each $j \in [J]$, $k \in [K]$ and $r \in [R]$, $U_{j} \cap V_{k} \cap W_{r}$ is a hypercube with the side length $2^{-2m-n}$.
    }
    \label{fig:skip_UVW}
\end{figure}

Given arbitrary binary values $\{\eta_{j,k,r}\}_{j \in [J], k \in [K], r \in [R]} \in \{0,1\}^{JKR},$
\begin{enumerate}
    \item [(i)] for every $r \in [R]$, there exists a network 
    $$g_r \in \DHN\big(3,\ d(2m+n):(2K+1):1\big)$$ such that
    $$\mathbb{I} \Big( g_r\big(b_{1}(\bx), b_{2}(\bx), \ldots, b_{2m+n}(\bx) \big) \Big) = \eta_{j(\bx), k(\bx), r} \mathbb{I}(r(\bx)=r),$$

    \item [(ii)] and there exists a network 
    $$g \in \DHN_{\skipp}\big(R+3,\ d(2m+n):(6K+5):1, 2K+1 \big)$$ such that
    $$\mathbb{I} \Big( g\big(b_{1}(\bx), b_{2}(\bx), \ldots, b_{2m+n}(\bx) \big) \Big) = \eta_{j(\bx), k(\bx), r(\bx)}.$$
\end{enumerate}
\end{lemma}

\begin{proof}

    \begin{figure}[t]
    \centering
    \tikzset{every node/.style={scale=1.0}}
\begin{tikzpicture}[node distance=2cm, auto, thick, >=Stealth, scale=0.8]

\node[draw=white] (x) at (6, 0) { $\big(b_{\ell}(\bx)\big)_{\ell \in [2m+n]}$};
;

        \node[draw=white] (f11) at (0, 2) { $\big(\mathbb{I}(j(\bx)=j)\big)_{j \in [J]}$};

        \node[draw=white] (f12) at (6, 2) { $\big(\mathbb{I}(k(\bx)=k)\big)_{k \in [K]}$};

        \node[draw=white] (f13) at (12, 2) { $\mathbb{I}(r(\bx)=r)$};

        \node[draw=white] (f21) at (0, 4) { $\big(\eta_{j(\bx),k,r}\big)_{k \in [K]}$};

        \node[draw=white] (f22) at (6, 4) { $\big(\mathbb{I}(k(\bx)=k)\big)_{k \in [K]}$};

        \node[draw=white] (f23) at (12, 4) { $\mathbb{I}(r(\bx)=r)$};

        \node[draw=white] (f31) at (3, 6) { $\big(\eta_{j(\bx),k,r} \mathbb{I}(k(\bx)=k) \big)_{k \in [K]}$};

        \node[draw=white] (f32) at (9, 6) { $\mathbb{I}(r(\bx)=r)$};

        \node[draw=white] (f41) at (6, 8) { $\eta_{j(\bx),k(\bx),r} \mathbb{I}(r(\bx)=r)$};

        \draw[-, cyan, line width=0.2mm] (x.north) -- (f11.south); 
        \draw[-, cyan, line width=0.2mm] (x.north) -- (f12.south); 
        \draw[-, cyan, line width=0.2mm] (x.north) -- (f13.south); 

        \draw[-, cyan, line width=0.2mm] (f11.north) -- (f21.south) node[midway, right] {(\ref{tmp_2})};
        \draw[-, cyan, line width=0.2mm] (f12.north) -- (f22.south);
        \draw[-, cyan, line width=0.2mm] (f13.north) -- (f23.south);

        \draw[-, cyan, line width=0.2mm] (f21.north) -- (f31.south) node[near end, anchor=north west] {(\ref{tmp_3})};
        \draw[-, cyan, line width=0.2mm] (f22.north) -- (f31.south);
        \draw[-, cyan, line width=0.2mm] (f23.north) -- (f32.south);

        \draw[-, cyan, line width=0.2mm] (f31.north) -- (f41.south) node[near end, anchor=north west] {(\ref{tmp_4})};
        \draw[-, cyan, line width=0.2mm] (f32.north) -- (f41.south);

    \end{tikzpicture}
    \vspace{15pt}
    \caption{\textbf{Illustration of the proof of Lemma \ref{lemma_piece_bin_skip_ver3}-(i)} For each $r \in [R]$, we construct $g_r$ as the figure.
    }
    \label{fig:skip_make_pc_binary}
\end{figure}
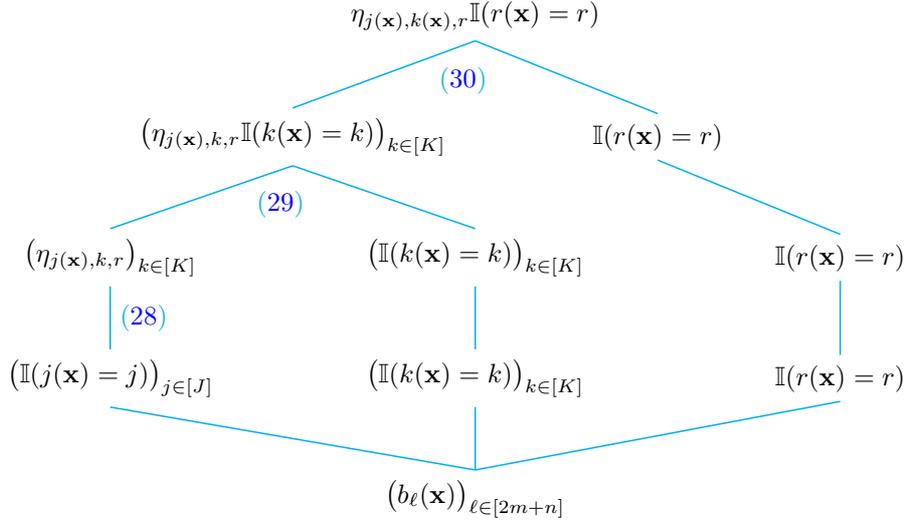

\textbf{(i)}
For each $j \in [J]$, there exists $\bm{b}_j = (b_{j, i, \ell})_{i \in [d], \ell \in [m]} \in  \{0,1\}^{md}$
such that $j(\bx)=j$ if and only if  
$(b_\ell(x_i))_{i \in [d], \ell \in [m]} = \bm{b}_j$.
Hence,
\begin{align*}
\mathbb{I}(j(\bx)=j) &= \prod_{i=1}^d \prod_{\ell=1}^m \mathbb{I}\left( b_\ell(x_i) = b_{j, i, \ell} \right)\\
&= \prod_{i=1}^d \prod_{\ell=1}^m \mathbb{I}\bigg( \big(2b_\ell(x_i)-1\big)\big(2b_{j, i, \ell}-1\big) =1 \bigg)\\
&= 
\mathbb{I}\left( \sum_{i=1}^d \sum_{\ell=1}^m \big(2b_\ell(x_i)-1\big)\big(2b_{j, i, \ell}-1\big) - \left(md-\frac{1}{2}\right) \right)
\end{align*}
is a linear threshold function of $(b_\ell(\bx))_{\ell \in [m]}$.
Similarly, $(\mathbb{I}(k(\bx)=k))_{k \in [K]}$ and $\mathbb{I}(r(\bx)=r)$ can be represented as a linear threshold function of $(b_\ell(\bx))_{\ell \in \{m+1,\ldots,2m\}}$ and $(b_\ell(\bx))_{\ell \in \{2m+1,\ldots,2m+n\}}$, respectively.  
Hence, we can represent $(\mathbb{I}(j(\bx)=j))_{j \in [J]}$,
$(\mathbb{I}(k(\bx)=k))_{k \in [K]}$ and
$\mathbb{I}(r(\bx)=r)$ in the first hidden layer, which are linear threshold functions of the input.

In the second hidden layer, for each $k \in [K]$,
we represent 
\begin{align}
    \eta_{j(\bx),k,r} = \sum_{j \in [J]} \eta_{j,k,r} \mathbb{I}(j(\bx)=j)
    = \mathbb{I}\left(\sum_{j \in [J]} \eta_{j,k,r} \mathbb{I}(j(\bx)=j) - \frac{1}{2}\right). \label{tmp_2}
\end{align}
In addition, we forward $(\mathbb{I}(k(\bx)=k))_{k \in [K]}$ and
$\mathbb{I}(r(\bx)=r)$
from the previous layer using the identity $b = \mathbb{I}(b - 1/2)$ for $b \in \{0,1\}$.

In the third hidden layer, for each $k \in [K]$, 
we represent
\begin{align}
    \eta_{j(\bx),k,r} \mathbb{I}(k(\bx)=k)   
    = \mathbb{I}\left(\eta_{j(\bx),k,r} + \mathbb{I}(k(\bx)=k) - \frac{3}{2}\right). \label{tmp_3}
\end{align}
In addition, we forward 
$\mathbb{I}(r(\bx)=r)$
from the previous layer using the identity $b = \mathbb{I}(b - 1/2)$ for $b \in \{0,1\}$.
In the output hidden layer, we represent
\begin{align}
    \eta_{j(\bx),k(\bx),r} \mathbb{I}(r(\bx)=r)  
    = \mathbb{I}\left(\sum_{k \in [K]} \eta_{j(\bx),k,r} \mathbb{I}(k(\bx)=k) + \mathbb{I}(r(\bx)=r) - \frac{3}{2}\right),\label{tmp_4}
\end{align}
where we use the identities $\eta_{j(\bx),k(\bx),r} = \sum_{k \in [K]} \eta_{j(\bx),k,r}$ and $bc = \mathbb{I}(b + c - 3/2)$ for $b,c \in \{0,1\}$.
Since $J=K$,
the width of $g_r$ is bounded above by $\max(J+K+1, 2K+1, K+1) = 2K+1$, and there are no skip connected neurons in any hidden layer.

\medskip \medskip 

\noindent
\textbf{(ii)}
For each $r \in [R]$, we construct the first, second, third, and output layers of $g_r$ in the $r$-th, $(r+1)$-st, $(r+2)$-nd, and $(r+3)$-rd hidden layers of $g$, respectively.
Then, the $(r+3)$-rd hidden layer contains a neuron that outputs $\eta_{j(\bx),k(\bx),r} \mathbb{I}(r(\bx)=r)$. Moreover, the second to $R$-th hidden layer include $2K+1$ skip connected neurons each.

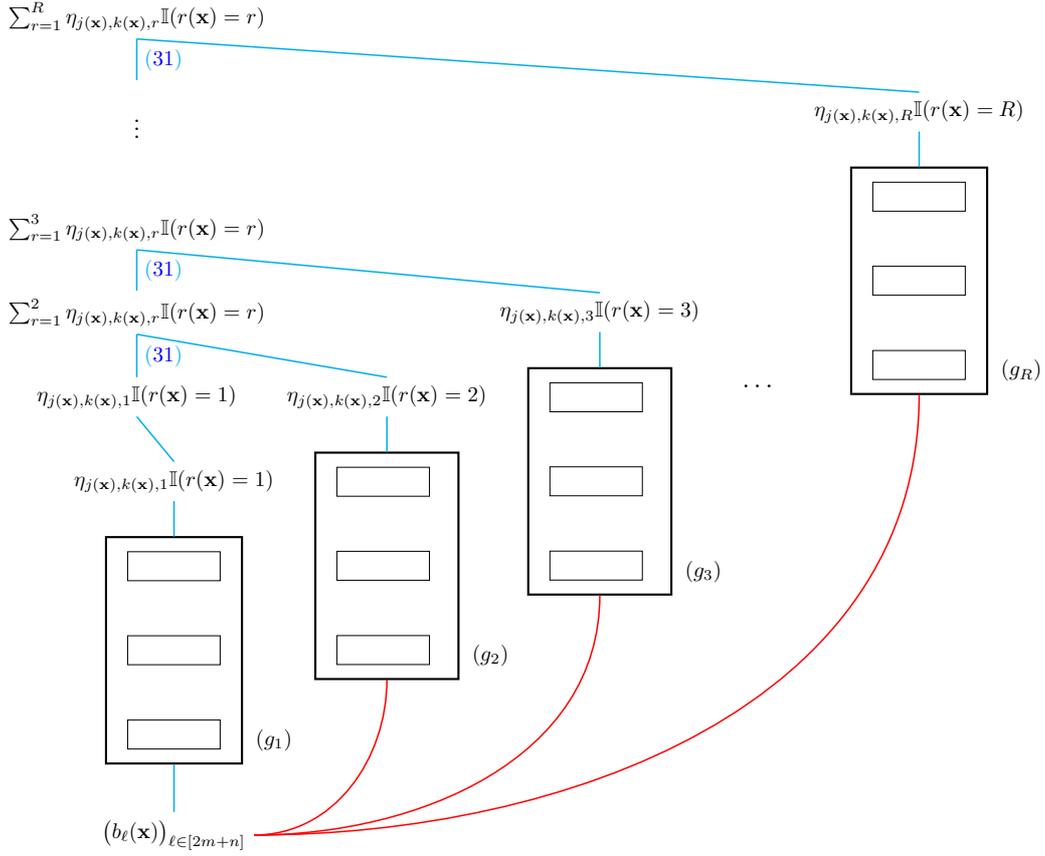
\begin{figure}[t]
\centering
\tikzset{every node/.style={scale=0.8}}
\begin{tikzpicture}[node distance=2cm, auto, thick, >=Stealth, scale=0.7]

\node[draw=white] (x) at (0, -0.5) { $\big(b_{\ell}(\bx)\big)_{\ell \in [2m+n]}$};

\node[draw=black, align=center, minimum height=3.5cm, minimum width=2cm] (g1) at (0, 3) {    \begin{tikzpicture}
        \matrix[column sep=0mm, row sep=3mm] {
            \node[draw=black, minimum height=0.6cm, minimum width=1.9cm] {}; \\ \\  \\
            \node[draw=black, minimum height=0.6cm, minimum width=1.9cm] {}; \\ \\ \\
            \node[draw=black, minimum height=0.6cm, minimum width=1.9cm] {}; \\  
        };
    \end{tikzpicture}};
\node[above right=0.1cm] at (g1.south east) {$(g_1)$};

\node[draw=white] (g1o) at (0, 6.2) { $\eta_{j(\bx),k(\bx),1} \mathbb{I}(r(\bx)=1)$};

\node[draw=black, align=center, minimum height=3.5cm, minimum width=2cm] (g2) at (4, 4.6) {\begin{tikzpicture}
        \matrix[column sep=0mm, row sep=3mm] {
            \node[draw=black, minimum height=0.6cm, minimum width=1.9cm] {}; \\ \\  \\
            \node[draw=black, minimum height=0.6cm, minimum width=1.9cm] {}; \\ \\ \\
            \node[draw=black, minimum height=0.6cm, minimum width=1.9cm] {}; \\  
        };
    \end{tikzpicture} };
\node[above right=0.1cm] at (g2.south east) {$(g_2)$};

\node[draw=white] (g2o) at (4, 7.8) { $\eta_{j(\bx),k(\bx),2} \mathbb{I}(r(\bx)=2)$};

\node[draw=black, align=center, minimum height=3.5cm, minimum width=2cm] (g3) at (8, 6.2) { \begin{tikzpicture}
        \matrix[column sep=0mm, row sep=3mm] {
            \node[draw=black, minimum height=0.6cm, minimum width=1.9cm] {}; \\ \\  \\
            \node[draw=black, minimum height=0.6cm, minimum width=1.9cm] {}; \\ \\ \\
            \node[draw=black, minimum height=0.6cm, minimum width=1.9cm] {}; \\  
        };
    \end{tikzpicture} };
\node[above right=0.1cm] at (g3.south east) {$(g_3)$};

\node[draw=white] (g3o) at (8, 9.4) { $\eta_{j(\bx),k(\bx),3} \mathbb{I}(r(\bx)=3)$};

\node[draw=black, align=center, minimum height=3.5cm, minimum width=2cm] (g4) at (14, 10) {\begin{tikzpicture}
        \matrix[column sep=0mm, row sep=3mm] {
            \node[draw=black, minimum height=0.6cm, minimum width=1.9cm] {}; \\ \\  \\
            \node[draw=black, minimum height=0.6cm, minimum width=1.9cm] {}; \\ \\ \\
            \node[draw=black, minimum height=0.6cm, minimum width=1.9cm] {}; \\  
        };
    \end{tikzpicture}};
\node[above right=0.1cm] at (g4.south east) {$(g_R)$};

\node[draw=white] (g4o) at (14, 13.2) { $\eta_{j(\bx),k(\bx),R} \mathbb{I}(r(\bx)=R)$};

\node[draw=white] (s1) at (-0.7, 7.8) { $\eta_{j(\bx),k(\bx),1} \mathbb{I}(r(\bx)=1)$};

\node[draw=white] (s2) at (-0.7, 9.4) { $\sum_{r=1}^2 \eta_{j(\bx),k(\bx),r} \mathbb{I}(r(\bx)=r)$};

\node[draw=white] (s3) at (-0.7, 11) { $\sum_{r=1}^3 \eta_{j(\bx),k(\bx),r} \mathbb{I}(r(\bx)=r)$};

\node[draw=white] (s4) at (-0.7, 15) { $\sum_{r=1}^R \eta_{j(\bx),k(\bx),r} \mathbb{I}(r(\bx)=r)$};

    \draw[-, cyan, line width=0.2mm] (x.north) -- (g1.south);

    \draw[-, cyan, line width=0.2mm] (g1.north) -- (g1o.south);

    \draw[-, cyan, line width=0.2mm] (g2.north) -- (g2o.south);

    \draw[-, cyan, line width=0.2mm] (g3.north) -- (g3o.south);

    \draw[-, cyan, line width=0.2mm] (g4.north) -- (g4o.south);

    \node at (11, 8) {\large $\cdots$};
    \node at (-0.7, 13) {\large $\vdots$};

    \draw[-, thick, red, line width=0.2mm] (x) to[out=0, in=270] (g2);
    \draw[-, thick, red, line width=0.2mm] (x) to[out=0, in=270] (g3);
    \draw[-, thick, red, line width=0.2mm] (x) to[out=0, in=270] (g4);

    \draw[-, cyan, line width=0.2mm] (g1o.north) -- (s1.south);
    \draw[-, cyan, line width=0.2mm] (s1.north) -- (s2.south) node[midway, right] {(\ref{tmp_5})};
    \draw[-, cyan, line width=0.2mm] (s2.north) -- (s3.south) node[midway, right] {(\ref{tmp_5})};
    \draw[-, cyan, line width=0.2mm] (g2o.north) -- (s2.south);
    \draw[-, cyan, line width=0.2mm] (g3o.north) -- (s3.south);

    \draw[-, cyan, line width=0.2mm] (-0.7,13.8) -- (s4.south) node[midway, right] {(\ref{tmp_5})};
    \draw[-, cyan, line width=0.2mm] (g4o.north) -- (s4.south);

    \end{tikzpicture}
\vspace{15pt}
\caption{\textbf{Illustration of the proof of Lemma \ref{lemma_piece_bin_skip_ver3}-(ii)} We construct $g$ as the figure. $g_1, g_2, \ldots, g_R \in \DHN(3,\ d(2m+n):(2K+1):1)$ are defined on Lemma \ref{lemma_piece_bin_skip_ver3}-(i). Each hidden layer of $g$ has at most $3(2K+1)+2$ neurons.
}
\label{fig:skip_make_pc_binary_part2}
\end{figure}

We forward $\eta_{j(\bx),k(\bx),1} \mathbb{I}(r(\bx)=1)$ from the fourth hidden layer to the fifth hidden layer using the identity $b = \mathbb{I}(b-1/2)$ for $b \in \{0,1\}$.
Assume that the $(\ell+4)$-th hidden layer with $1 \leq \ell \leq R-1$ has a neuron whose value is $\sum_{r=1}^{\ell} \eta_{j(\bx),k(\bx),r} \mathbb{I}(r(\bx)=r)$.
Then, for the $(\ell+5)$-th hidden layer, we represent
\begin{equation}
    \begin{aligned}
        &\sum_{r=1}^{\ell+1} \eta_{j(\bx),k(\bx),t} \mathbb{I}(r(\bx)=r)  \\
        &= \mathbb{I}\left( \sum_{r=1}^{\ell} \eta_{j(\bx),k(\bx),r} \mathbb{I}(r(\bx)=r) + \eta_{j(\bx),k(\bx),\ell+1} \mathbb{I}(r(\bx)=\ell+1) - \frac{1}{2} \right). \label{tmp_5}
    \end{aligned}
\end{equation}
Iterating this process, we obtain $\sum_{r=1}^{R} \eta_{j(\bx),k(\bx),r} \mathbb{I}(r(\bx)=r) = \eta_{j(\bx),k(\bx),r(\bx)}$ in the output layer.
Since each hidden layer has at most $3(2K+1) + 2 = 6K+5$ neurons, we obtain the assertion.
\end{proof}

\medskip

\begin{proof}[Proof of the lower bound in Theorem \ref{thm_skip_vc_upper}] 

It is enough to show the result for $d = 1$.
We set integers $m$ and $n$ such that $m \lesssim \log p$ and $n \lesssim \log L$, where their exact values will be specified in (\ref{tmp_41}). 
For $i \in [2^{2m+n}]$, we define 
    $$z_i := \frac{i}{2^{2m+n}} - \frac{1}{2^{2m+n+1}}.$$
    We now show that for an arbitrary function
    $$\lambda : \{z_1, z_2, \ldots, z_{2^{2m+n}} \} \mapsto \{0,1\},$$
    there exists $f \in \DHN_{\skipp}(L,1:p:1,1)$ such that $f(z_i) = \lambda(z_i)$ for every $i \in [2^{2m+n}]$. Thus, $\operatorname{VC}(\DHN_{\skipp}(L,1:p:1,s)) \geq 2^{2m+n}$.
    
    First, we construct a Heaviside network with a skip connection to the input, that extracts the first $2m+n$ bits of $x$, that is, $(b_\ell(x))_{\ell \in [2m+n]}$. 
    This is achieved by applying Lemma \ref{lemma_bit_extract} with $L := 2m+n$, $p_1 := 1$ and $\bm{s} := \bm{1}_{2m+n-1}$.
    For $\ell \in [2m+n],$ the $\ell$-th hidden layer has $\ell$ neurons with one skip connected neuron.
    The $(2m+n)$-th hidden layer outputs $(b_{\ell}(x))_{\ell \in [2m+n]}$. This network gives the construction of the network $f$ up to the $(2m+n)$-th layer.  

    We now describe the construction of the remaining $(R+4)$-hidden layers of the network $f.$ The main step is the construction of a subnetwork $g.$ As in \eqref{eq.0fweni} and the subsequent text (with $d=1$), define $J:=2^{m}$, $K:=2^{m}$, $R:=2^{n}$, along with the partitions
    $\{U_j\}_{j \in [J]}$, $\{V_k\}_{k \in [K]}$, $\{W_r\}_{r \in [R]}$  of $[0,1]$, and the index functions $j(x) \in [J]$, $k(x) \in [K]$, and $r(x) \in [R].$ By definition, there exists a bijective function 
    $\bm{\pi} = (\pi_1,\pi_2,\pi_3) : [2^{2m+n}] \mapsto [J] \times [K] \times [R]$
    such that 
    $\pi_1(i) = j(z_i)$, $\pi_2(i) = k(z_i)$ and $\pi_3(i) = r(z_i)$ for every $i \in [2^{2m+n}]$. 
    In fact, $z_{i}$ is the center of the interval $U_{\pi_1(i)} \cap V_{\pi_2(i)} \cap W_{\pi_3(i)}$.
    
    By applying Lemma \ref{lemma_piece_bin_skip_ver3}-(ii) with $\eta_{j,k,r} := \lambda(z_{\bm{\pi}^{-1}(j,k,r)})$, we obtain a DHN
    $$g \in \DHN_{\skipp}\big(R+3,\ (2m+n):(6K+5):1, 2K+1 \big)$$
    such that 
    $$\mathbb{I} \Big( g \big((b_{\ell}(x))_{\ell \in [2m+n]} \big) \Big) = \eta_{j(x),k(x),r(x)} = \lambda\big(z_{\bm{\pi}^{-1}(j(x),k(x),r(x))}\big).$$
    Replacing $x$ by $z_i$, we obtain
    \begin{align}
        \mathbb{I} \Big( g \big( (b_{\ell}(z_i))_{\ell \in [2m+n]} \big) \Big) = \lambda(z_{\bm{\pi}^{-1}(\pi_1(i),\pi_2(i),\pi_3(i))}) = \lambda(z_i), \quad \text{for any}  \ \ i \in [2^{2m+n}].\label{tmp_33}
    \end{align} 
We can apply Remark \ref{rem.1} to get rid of the skip connections by adding $2m+n$ hidden units to each hidden layer. This shows that there exists a network
    $$\widetilde g \in \DHN\big(R+3,\ (2m+n):(6K+5+2m+n):1 \big)$$
    such that \eqref{tmp_33} holds with $g$ replaced by $\widetilde g.$ By adding one layer, 
    $$\mathbb{I}(\widetilde g) \in \DHN\big(R+4,\ (2m+n):(6K+5+2m+n):1 \big).$$

    The network $f$ has been already constructed up to the $(2m+n)$-th hidden layer outputting the binary digits $(b_{\ell}(x))_{\ell \in [2m+n]}$. Stacking the network $\mathbb{I}(\widetilde g)$ on top results in the network 
    $$f \in \DHN_{\skipp}\big(2m+n+R+4, \ 1:(2m+n + 6K+5):1, \ 1\big)$$ 
    computing $f(z_i) = \lambda(z_i)$ for every $i \in [2^{2m+n}]$.
    Recall that $R=2^{n}$ and $K=2^{m}$.
    We set
    \begin{align}
        n := \bigg\lfloor \log\left(\frac{L}{2}\right)\bigg\rfloor, \qquad m := \bigg\lfloor \log\left(  \frac{p}{10} \right) \bigg\rfloor. \label{tmp_41}
    \end{align}
    Given that 
    $$2m+n \leq 2\log (Lp) \leq (L \land p)/4,$$
    we obtain 
    \begin{align*}
        2m+n+R+4 \leq \frac{L}{4} + \frac{L}{2} + 4 &\leq L,\\
        2m+n+6K+5 \leq \frac{p}{4} + \frac{3p}{5} + 5 &\leq p
    \end{align*}
    for sufficiently large $L$ and $p$ (e.g. if $\tfrac{3}{4}L+4 \leq L$ and $\tfrac{17}{20}p + 5 \leq p$, implying that the constant $c$ in the statement of the theorem can be taken to be $34$). Hence $f \in \DHN_{\skipp}(L,1:p:1,1)$.
    
    Given the inequalities
    $$n \geq \log\left(\frac{L}{2}\right)-1, \qquad m \geq \log\left(  \frac{p}{10} \right)-1,$$ 
    we obtain the lower bound
    $$\operatorname{VC}\big(\DHN_{\skipp}(L,1:p:1,1)\big) \geq 2^{2m+n} \geq \frac{1}{2^3 \cdot 2 \cdot 10^2} Lp^2.$$
\end{proof}

The above proof focuses on providing a clean rate with respect to $L$ and $p$, without attempting to minimize the constant $c$ or maximize the constant $C$.
Imposing additional constraints on $L$ and $p$ can yield significantly better constants.

\subsection{Proofs of Section \ref{sec_skip_smooth}} \label{app_proof_skip_smooth}

\begin{proof}[Proof of Proposition \ref{vc-connect-appro}]

In Definition \ref{hölder}, we introduced $\mathcal{H}_d^{\beta}(M) := \{ f : [0,1]^d \to \mathbb{R}, \| f\|_{\mathcal{C}^{\beta}} \leq M \}$ as the H\"older ball of $\beta$-smooth functions with radius $M.$ In this proof, we also define \[\mathcal{H}_{\mathbb{R}^d}^{\beta}(M)
:= \Big\{f:\mathbb{R}^d \to \mathbb{R}, \,  \ f\big|_{[0,1]^d} \in \mathcal{H}_d^{\beta}(M) \Big\}.\]

In the proof, we show that for any $d\geq1$, $\beta>0$, $\varepsilon>0$ and any function class $\mathcal{F}$ whose elements are defined on $[0,1]^d$,
\begin{equation}\label{app-accuracy}
    \sup_{f_0\in\mathcal{H}_{\mathbb{R}^d}^{\beta}(M)} \,  \inf_{f\in\mathcal{F}} \, \sup_{\bx \in [0,1]^d} |f(\bx)-f_0(\bx)| \leq\varepsilon,
\end{equation}
implies $\operatorname{VC}(\mathcal{F})\geq C_{d,\beta}(\varepsilon/M)^{-d/\beta},$
where $C_{d,\beta}>0$ is a constant only depending on $d$ and $\beta$.
The contrapositive of this statement establishes the assertion of Proposition \ref{vc-connect-appro}. 

The proof consists of two main steps. First, we construct a class of functions in $\mathcal{H}_{\mathbb{R}^d}^{\beta}(M)$ that can shatter $\geq C_{d,\beta}(\varepsilon/M)^{-d/\beta}$ points in $[0,1]^d$. Then, we show that a subset of $\mathcal{F}$ can also shatter the same points.

\medskip
\noindent {\textbf{Step 1}}: 
Let $K\geq 1$, to be specified in Step 2. We partition the interval $[0,1)$ into $K$ equal subintervals. Extending this construction to $d$ dimensions, consider ${\bm\ell}=(\ell_1,\ldots,\ell_{d})\in\{0,1,\ldots,(K-1)\}^d$, which yields a partition of $[0,1)^d$ into $K^d$ hypercubes $\{\mathcal{I}_{\bm\ell}\}_{{\bm\ell}\in\{0,1,\ldots,(K-1)\}^{d}}$, where $$\mathcal{I}_{\bm\ell}=\left[\frac{\ell_1}{K},\frac{\ell_1+1}{K}\right)\times\cdots\times\left[\frac{\ell_{d}}{K},\frac{\ell_{d}+1}{K}\right).$$

Take a smooth function $\widetilde\phi\in C^{\infty}(\mathbb{R}^d)$ such that $\widetilde\phi(\bm{0}_d)=1$ and $\widetilde\phi({\bx})=0$, for all $\|\bx\|_{2}\geq1/3.$ Such a function does exist; for instance, one can take the function $\widetilde g$ from the proof of Theorem 2.4 in \cite{lu2021deep}. 
Define  $\phi:=M\widetilde\phi/\|\widetilde\phi\|_{\mathcal{C}^{\beta}}\in\mathcal{H}_{\mathbb{R}^d}^{\beta}(M)$ and $$\Lambda:=\left\{\lambda,\;\lambda:\{0,\ldots,K-1\}^d\rightarrow\{-1,1\}\right\}.$$ 
For each ${\bm\ell},$ let $I_{{\bm\ell}}$ denote the geometric center of the hypercube $\mathcal{I}_{{\bm\ell}}$ and for any $\lambda\in\Lambda$, consider the function
$$f_{0,\lambda}(\bx):=\sum_{{\bm\ell}\in\{0,\ldots,K-1\}^d}\frac{K^{-\beta}}{2}\lambda({\bm\ell})\phi\big(K(\bx-I_{{\bm\ell}})\big).$$ 
We now verify that for all $\lambda\in\Lambda$, $f_{0,\lambda} \in\mathcal{H}_{\mathbb{R}^d}^{\beta}(M)$. First, since $\phi\in C^{\infty}(\mathbb{R}^d)$, and the support of $f_{0,\lambda}$ within each hypercube $\mathcal{I}_{{\bm\ell}}$ is contained in a ball centered at $I_{\bm\ell}$ with radius $1/(3K)$, it follows that $f_{0,\lambda}(\bx)\in C^{\infty}(\mathbb{R}^d)$, for all $\lambda \in \Lambda.$ 
Next, we check that $\|f_{0,\lambda}\|_{\mathcal{C}^{\beta}}\leq M$. Let $\beta = q+s$ for some $q \in \mathbb{N}$ and $s \in (0,1]$. For any $\bm{\alpha} \in \mathbb{N}^d$ with $\left\|\bm{\alpha}\right\|_1 \leq q$ and any $\bx\in\mathcal{I}_{{\bm\ell}}$, we have 
\begin{equation}\label{int-part-ratio}
|\partial^{\bm{\alpha}}f_{0,\lambda}(\bx)|=\frac{K^{\left\|\bm{\alpha}\right\|_1-\beta}}{2}|\partial^{\bm{\alpha}}\phi(K(\bx-I_{{\bm\ell}}))|\leq \frac{1}{2}|\partial^{\bm{\alpha}}\phi(K(\bx-I
_{{\bm\ell}}))|.
\end{equation}
Also, for any $\bm{\alpha} \in \mathbb{N}^d$ with $\left\|\bm{\alpha}\right\|_1=q$, and any $\bx_1,\bx_2\in\mathcal{I}_{{\bm\ell}}$, with $\bx_1\neq \bx_2,$
\begin{align}
\frac{|\partial^{\bm{\alpha}}f_{0,\lambda}(\bx_1)-\partial^{\bm{\alpha}}f_{0,\lambda}(\bx_2)|}{\|\bx_1-\bx_2\|_{\infty}^{s}}&=\frac{K^{q-\beta}}{2}\cdot\frac{|\partial^{\bm{\alpha}}\phi(K(\bx_1-I_{{\bm\ell}}))-\partial^{\bm{\alpha}}\phi(K(\bx_2-I_{{\bm\ell}}))|}{\|\bx_1-\bx_2\|_{\infty}^{s}}\nonumber\\
&=\frac{K^{s+q-\beta}}{2}\cdot\frac{|\partial^{\bm{\alpha}}\phi(\by_1)-\partial^{\bm{\alpha}}\phi(\by_2)|}{\|\by_1-\by_2\|_{\infty}^{s}}\nonumber\\
&=\frac{|\partial^{\bm{\alpha}}\phi(\by_1)-\partial^{\bm{\alpha}}\phi(\by_2)|}{2 \|\by_1-\by_2\|_{\infty}^{s}}\label{frac-ratio}
\end{align}
for $\by_1:=K(\bx_1-I_{{\bm\ell}}) $ and $\by_2:=K(\bx_2-I_{{\bm\ell}})$.
Thus, combining \eqref{int-part-ratio}, \eqref{frac-ratio} with the fact that $\phi \in\mathcal{H}_{\mathbb{R}^d}^{\beta}(M)$, we conclude that $f_{0,\lambda} \in\mathcal{H}_{\mathbb{R}^d}^{\beta}(M)$, for all $\lambda\in\Lambda$. Following the  construction of the function class $\{f_{0,\lambda},\;\lambda\in\Lambda\}$, we know that $\{f_{0,\lambda},\;\lambda\in\Lambda\}$ can shatter $\{I_{{\bm\ell}}\}_{{\bm\ell}\in\{0,1,\ldots,(K-1)\}^d}$, which are the centers of the $K^d$ hypercubes.

\medskip
\noindent
{\textbf{Step 2}}: Now, we focus on all centers $I_{{\bm\ell}}$, for ${\bm\ell}\in\{0,1,\ldots,(K-1)\}^d$. Given \eqref{app-accuracy}, for each $\lambda$, on the one hand, there exists a function $f_{\lambda}\in\mathcal{F}$ satisfying for all $I_{{\bm\ell}}$,
\begin{equation}\label{diff-dis}
|f_{\lambda}(I_{{\bm\ell}})-f_{0,\lambda}(I_{{\bm\ell}})|\leq\varepsilon.
\end{equation}
On the other hand, by taking $K := \lfloor(4\varepsilon\|\widetilde\phi\|_{\mathcal{C}^{\beta}}/M)^{-1/\beta}\rfloor$, we have
\begin{equation}\label{ab-dis}
    |f_{0,\lambda}(I_{{\bm\ell}})|=\frac{K^{-\beta}}{2}\cdot\frac{M}{\|\widetilde\phi\|_{\mathcal{C}^{\beta}}}\geq 2\varepsilon.
\end{equation}
Combining \eqref{diff-dis} and \eqref{ab-dis}, we see that $f_{0,\lambda}$ and $f_{\lambda}$ share the same sign at all $I_{{\bm\ell}}$. Consequently, the set $\{f_{\lambda},\;\lambda\in\Lambda\}$ in $\mathcal{F}$ also shatters the collection of centers $\{I_{{\bm\ell}}\}_{{\bm\ell}\in\{0,1,\ldots,(K-1)\}^d}$.

We set $C_{d,\beta}:=2^{-d}(4\|\widetilde\phi\|_{\mathcal{C}^{\beta}})^{-d/\beta}$.
If $\varepsilon\leq M(2^d C_{d,\beta})^{\beta/d}$, we have
$$K=\left\lfloor\left(\frac{4\varepsilon\|\widetilde\phi\|_{\mathcal{C}^{\beta}}}{M}\right)^{-1/\beta}\right\rfloor=\left\lfloor\left(\frac{\varepsilon}{M}\right)^{-1/\beta}\big(2^d C_{d,\beta}\big)^{1/d}\right\rfloor\geq1,$$
and hence
\begin{align*}
    \operatorname{VC}(\mathcal{F})\geq \operatorname{VC}\big(\{f_{\lambda},\;\lambda\in\Lambda\}\big)\geq K^d=\Big\lfloor\Big(\frac{4\varepsilon\|\widetilde\phi\|_{\mathcal{C}^{\beta}}}M\Big)^{-1/\beta}\Big\rfloor^d\geq2^{-d}\Big(\frac{4\varepsilon\|\widetilde\phi\|_{\mathcal{C}^{\beta}}}M\Big)^{-d/\beta},
\end{align*}
which implies
$$\operatorname{VC}(\mathcal{F})\geq C_{d,\beta}\left(\frac{\varepsilon}{M}\right)^{-d/\beta}.$$ 
If $\varepsilon>M(2^d C_{d,\beta})^{\beta/d}$, we immediately obtain
$$\operatorname{VC}(\mathcal{F})\geq1>C_{d,\beta}\left(\frac{\varepsilon}{M}\right)^{-d/\beta}.$$
\end{proof}

\medskip
\medskip

\begin{proof}[Proof of Theorem \ref{upper_smooth_skip}]

We introduce the integers $m$ and $n$ satisfying $m \lesssim \log p$ and $n \lesssim \log L$. Their exact values will be specified in (\ref{tmp_12}). 
In addition, we define $q = \lceil \beta \rceil(2m+n)$.

Any real number $x \in [0,1]$ can be expressed in a binary (base-2) representation as
\begin{align*}
    x=0.b_1(x)b_2(x)b_3(x)\dots_{(2)}.
\end{align*}
We approximate $x$ by the finite-bit representation
\begin{align}
\label{form1}
\widetilde{x}=0.b_1(x)b_2(x)b_3(x)\dots b_{2m+n}(x)_{(2)} + 2^{-2m-n-1},
\end{align}
where we introduce the additional bias term $2^{-2m-n-1}$ to reduce the overall approximation error between $x$ and $\widetilde{x}$.

For a vector $\bx = (x_1, \dots, x_d) \in [0,1]^d$, we define $b_{\ell}(\bx) := (b_{\ell}(x_i))_{i \in [d]}$.
In addition, we define the truncated approximation of $\bx$ by $\widetilde{\bx} = (\widetilde{x}_1, \dots, \widetilde{x}_d) \in [0,1]^d,$ where each $\widetilde{x}_i$ is defined as in \eqref{form1}. Therefore,
$$
    \|\bx - \widetilde{\bx}\|_{\infty} \leq 2^{-2m-n-1}.
$$

Our goal is to construct $f \in \DHN_{\skipp}(L,d:p:1,d)$ satisfying
\begin{align} 
    \sup_{\bx \in [0,1]^d} \Bigg|f(\bx) - \underbrace{\sum_{\bm{\alpha} \in \mathcal{A}(d,\beta)} (\partial^{\bm{\alpha}}f_0)(\widetilde{\bx}) \cdot \frac{(\bx-\widetilde{\bx})^{\bm{\alpha}}}{\bm{\alpha}!}}_{:= \widetilde{f}_0(\bx)} \Bigg| \lesssim   \left(\frac{\log^3 (Lp)}{Lp^2 }\right)^{\frac{\beta}{d}},
    \label{eq_skip_taylor}
\end{align}
with
$$\mathcal{A}(d,\beta) :=  \{ \bm{\alpha} \in \mathbb{N}^d  :\|\bm{\alpha}\|_1 < \beta \}.$$
Since Lemma \ref{lemma_taylor} with \eqref{tmp_12} implies
\begin{align*}
    \sup_{\bx \in [0,1]^d} \left| \widetilde{f}_0(\bx) -f_0(\bx) \right| 
    \leq  \left\| f_0 \right\|_{\mathcal{C}^{\beta}} \cdot \|\bx - \widetilde{\bx} \|_{\infty}^\beta \lesssim   \left(\frac{\log^3 (Lp)}{Lp^2 }\right)^{\frac{\beta}{d}},
\end{align*}
the assertion follows.
The proof proceeds via the following three steps:  
\medskip 

\noindent
\textit{Step 1}: Extracting the first $q$ bits of $\bx$ and $(\partial^{\bm{\alpha}}f_0)(\widetilde{\bx})$ for each $\bm{\alpha} \in \mathcal{A}(d,\beta)$.

\noindent
\textit{Step 2}: Approximating $\widetilde{f}_0(\bx)$.

\noindent
\textit{Step 3}: Verifying the size and error of the final network.

\begin{figure}[t]
    \centering
    \tikzset{every node/.style={scale=0.9}}
    \begin{tikzpicture}[node distance=2cm, auto, thick, >=Stealth]
        \node[draw=white] (x1) at (0, 0) { $x_1$};        

        \node[draw=white] (f11) at (0, 1.5) { $\big(b_{\ell}(x_1)\big)_{\ell \in [1]}$};
        
        \node[draw=white] (f21) at (0, 3) { $\big(b_{\ell}(x_1)\big)_{\ell \in [2]}$};

        \node[draw=white] (f31) at (0,4.5) { $\big(b_{\ell}(x_1)\big)_{\ell \in [3]}$};

        \node at (0, 5.8) {\huge $\vdots$};

        \node[draw=white] (f41) at (0, 7.5) { $\big(b_{\ell}(x_1)\big)_{\ell \in [q]}$};

    \draw[-, cyan, line width=0.2mm] (x1.north) -- (f11.south);
       \draw[-, cyan, line width=0.2mm] (f11.north) -- (f21.south);
       \draw[-, cyan, line width=0.2mm] (f21.north) -- (f31.south);

       \draw[-, cyan, line width=0.2mm] (0,6.4) -- (f41.south);

       \draw[-, thick, red, line width=0.2mm] (x1) to[out=140, in=220] (f21);
       \draw[-, thick, red, line width=0.2mm] (x1) to[out=140, in=220] (f31);
       \draw[-, thick, red, line width=0.2mm] (x1) to[out=140, in=220] (f41);


        \node[draw=white] (x2) at (3, 0) { $x_2$};        

        \node[draw=white] (f12) at (3, 1.5) { $\big(b_{\ell}(x_2)\big)_{\ell \in [1]}$};
        
        \node[draw=white] (f22) at (3, 3) { $\big(b_{\ell}(x_2)\big)_{\ell \in [2]}$};

        \node[draw=white] (f32) at (3,4.5) { $\big(b_{\ell}(x_2)\big)_{\ell \in [3]}$};

        \node at (3, 5.8) {\huge $\vdots$};

        \node[draw=white] (f42) at (3, 7.5) { $\big(b_{\ell}(x_2)\big)_{\ell \in [q]}$};

    \draw[-, cyan, line width=0.2mm] (x2.north) -- (f12.south);
       \draw[-, cyan, line width=0.2mm] (f12.north) -- (f22.south);
       \draw[-, cyan, line width=0.2mm] (f22.north) -- (f32.south);

       \draw[-, cyan, line width=0.2mm] (3,6.4) -- (f42.south);

       \draw[-, thick, red, line width=0.2mm] (x2) to[out=140, in=220] (f22);
       \draw[-, thick, red, line width=0.2mm] (x2) to[out=140, in=220] (f32);
       \draw[-, thick, red, line width=0.2mm] (x2) to[out=140, in=220] (f42);
       
    \node at (5.5, 4.5) {\huge $\cdots$};


        \node[draw=white] (x3) at (9, 0) { $x_d$};        

        \node[draw=white] (f13) at (9, 1.5) { $\big(b_{\ell}(x_d)\big)_{\ell \in [1]}$};
        
        \node[draw=white] (f23) at (9, 3) { $\big(b_{\ell}(x_d)\big)_{\ell \in [2]}$};

        \node[draw=white] (f33) at (9,4.5) { $\big(b_{\ell}(x_d)\big)_{\ell \in [3]}$};

        \node at (9, 5.8) {\huge $\vdots$};

        \node[draw=white] (f43) at (9, 7.5) { $\big(b_{\ell}(x_d)\big)_{\ell \in [q]}$};

    \draw[-, cyan, line width=0.2mm] (x3.north) -- (f13.south);
       \draw[-, cyan, line width=0.2mm] (f13.north) -- (f23.south);
       \draw[-, cyan, line width=0.2mm] (f23.north) -- (f33.south);

       \draw[-, cyan, line width=0.2mm] (9,6.4) -- (f43.south);

       \draw[-, thick, red, line width=0.2mm] (x3) to[out=140, in=220] (f23);
       \draw[-, thick, red, line width=0.2mm] (x3) to[out=140, in=220] (f33);
       \draw[-, thick, red, line width=0.2mm] (x3) to[out=140, in=220] (f43);

\node[draw=black, align=center, minimum height=9cm, minimum width=3.3cm] (box1) at (-0.6, 3.8) {};
\node[above=0.1cm] at (box1.north) {(Lemma \ref{lemma_bit_extract})};

\node[draw=black, align=center, minimum height=9cm, minimum width=3.3cm] (box2) at (2.5, 3.8) {};
\node[above=0.1cm] at (box2.north) {(Lemma \ref{lemma_bit_extract})};

\node[draw=black, align=center, minimum height=9cm, minimum width=3.3cm] (box3) at (8.5, 3.8) {};
\node[above=0.1cm] at (box3.north) {(Lemma \ref{lemma_bit_extract})};

    \end{tikzpicture}
    \vspace{15pt}
    \caption{
\textbf{The first part of Step 1.} Extracting the first $q$ bits of $\bx$.}
    \label{fig:skip_smooth_step1}
\end{figure}
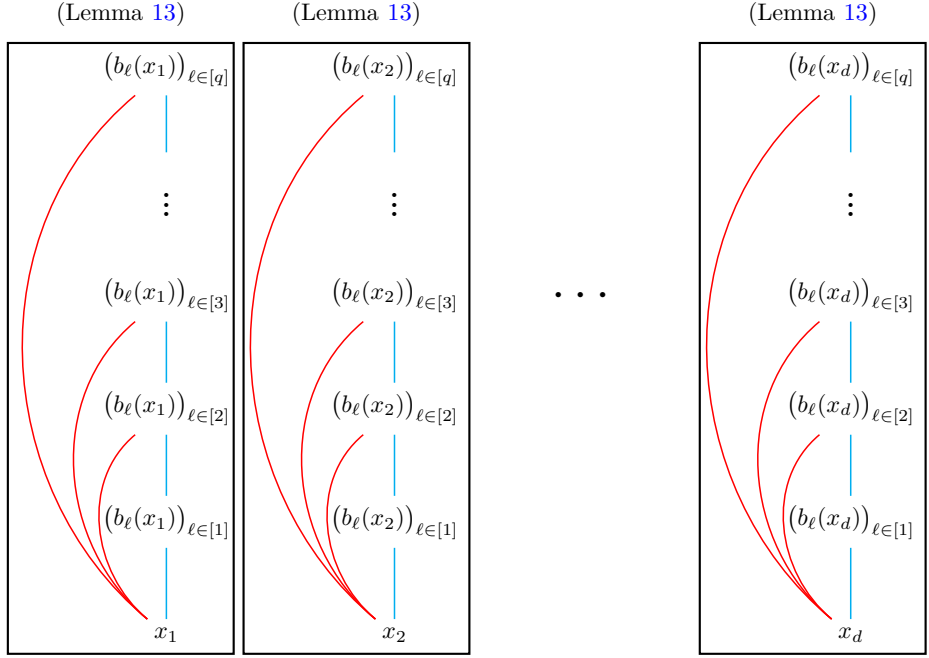

\medskip 
\medskip 

\noindent
\textbf{Step 1}: Extracting the first $q$ bits of $\bx$ and $(\partial^{\bm{\alpha}}f_0)(\widetilde{\bx})$ for each $\bm{\alpha} \in \mathcal{A}(d,\beta)$.

\medskip

\noindent
We start by extracting the $q$ bits of $\bx$ by applying Lemma \ref{lemma_bit_extract} with
$L := q$, $p_1 := 1$ and $\bm{s} := \bm{1}_{q-1}$ componentwise. 
For any  $2 \leq \ell \leq q$, the $\ell$-th hidden layer contains $d$ skip connected neurons with access to the $d$-dimensional input. 
The $q$-th hidden layer of $f$ outputs $(b_{\ell}(\bx))_{\ell \in [q]}$.
See Fig. \ref{fig:skip_smooth_step1} for an illustration.

We now extract the $q$ bits of $(\partial^{\bm{\alpha}}f_0)(\widetilde{\bx})$ for each $\bm{\alpha} \in \{\bm{\alpha}_1, \bm{\alpha}_2, \ldots, \bm{\alpha}_{|\mathcal{A}(d,\beta)|} \} := \mathcal{A}(d,\beta)$.
As in \eqref{eq.0fweni} and the subsequent text,
define $J:=2^{md}$, $K:=2^{md}$, $R:=2^{nd}$, along with the partitions
$\{U_j\}_{j \in [J]}$, $\{V_k\}_{k \in [K]}$, and $\{W_r\}_{r \in [R]}$  of $[0,1]^d$, and the index functions $j(\bx) \in [J]$, $k(\bx) \in [K]$, and $r(\bx) \in [R]$.
We write $\bm{c}_{j,k,r}$ the center of the hypercube $U_{j} \cap V_{k} \cap W_{r}$. Accordingly, the center of the hypercube including $\bx$ is
\begin{align*}
    \widetilde{\bx} = \bm{c}_{j(\bx),k(\bx),r(\bx)}.
\end{align*}

\begin{figure}[t]
    \centering
    \tikzset{every node/.style={scale=0.9}}
    \begin{tikzpicture}[node distance=2cm, auto, thick, >=Stealth]
        \node[draw=white] (x1) at (0, 0) { $\big(b_{\ell}(\bx)\big)_{\ell \in [q]}$};        

        \node[draw=white] (x2) at (0, 1.5) { $\big(b_{\ell}(\bx)\big)_{\ell \in [q]}$};
        
        \node[draw=white] (x3) at (0, 3) { $\big(b_{\ell}(\bx)\big)_{\ell \in [q]}$};

        \node[draw=white] (x4) at (0,4.5) { $\big(b_{\ell}(\bx)\big)_{\ell \in [q]}$};

        \node[draw=white] (x5) at (0, 7.5) { $\big(b_{\ell}(\bx)\big)_{\ell \in [q]}$};
        \node[draw=white] (x6) at (0, 9) { $\big(b_{\ell}(\bx)\big)_{\ell \in [q]}$};
        \node[draw=white] (x7) at (0, 10.5) { $\big(b_{\ell}(\bx)\big)_{\ell \in [q]}$};

    \draw[-, cyan, line width=0.2mm] (x1.north) -- (x2.south);
       \draw[-, cyan, line width=0.2mm] (x2.north) -- (x3.south);
       \draw[-, cyan, line width=0.2mm] (x3.north) -- (x4.south);
       \draw[-, cyan, line width=0.2mm] (x5.north) -- (x6.south);
       \draw[-, cyan, line width=0.2mm] (x6.north) -- (x7.south);



        \node at (0, 6.2) {\huge $\vdots$};
        \node at (2.5, 6.2) {\huge $\vdots$};
        \node at (5.5, 6.2) {\huge $\vdots$};
        \node at (10.5, 6.2) {\huge $\vdots$};

       \node[draw=black, line width=0.1mm, minimum height=9.5cm, minimum width=2cm] (f1) at (2.5, 5.3) { };
       \node[below, align=center] at (f1.south) {$g^{\bm{\alpha}_1, 1}$ \\
       (Lemma \ref{lemma_piece_bin_skip_ver3})};

       \node[draw=black, line width=0.1mm, draw=black, minimum height=0.6cm, minimum width=1.5cm] (f11) at (2.5, 1.5) { };
       \node[draw=black, line width=0.1mm, draw=black, minimum height=0.6cm, minimum width=1.5cm] (f12) at (2.5, 3) { };
       \node[draw=black, line width=0.1mm, draw=black, minimum height=0.6cm, minimum width=1.5cm] (f13) at (2.5, 4.5) { };
       \node[draw=black, line width=0.1mm, draw=black, minimum height=0.6cm, minimum width=1.5cm] (f14) at (2.5, 7.5) { };
       \node[draw=black, line width=0.1mm, draw=black, minimum height=0.6cm, minimum width=1.5cm] (f15) at (2.5, 9) { };
        \node[draw=white] (f16) at (2.5, 10.5) {$\eta^{(1)}_{\bm{\alpha}_1,j(\bx),k(\bx),r(\bx)}$ };

        \draw[-, cyan, line width=0.2mm] (f11.north) -- (f12.south);
       \draw[-, cyan, line width=0.2mm] (f12.north) -- (f13.south);
       \draw[-, cyan, line width=0.2mm] (f14.north) -- (f15.south);
       \draw[-, cyan, line width=0.2mm] (f15.north) -- (f16.south);

       \node[draw=black, line width=0.1mm, minimum height=9.5cm, minimum width=2cm] (f2) at (5.5, 5.3) { };
        \node[below, align=center] at (f2.south) {$g^{\bm{\alpha}_1, 2}$ \\
       (Lemma \ref{lemma_piece_bin_skip_ver3})};

        \node[draw=black, line width=0.1mm, draw=black, minimum height=0.6cm, minimum width=1.5cm] (f21) at (5.5, 1.5) { };
       \node[draw=black, line width=0.1mm, draw=black, minimum height=0.6cm, minimum width=1.5cm] (f22) at (5.5, 3) { };
       \node[draw=black, line width=0.1mm, draw=black, minimum height=0.6cm, minimum width=1.5cm] (f23) at (5.5, 4.5) { };
       \node[draw=black, line width=0.1mm, draw=black, minimum height=0.6cm, minimum width=1.5cm] (f24) at (5.5, 7.5) { };
       \node[draw=black, line width=0.1mm, draw=black, minimum height=0.6cm, minimum width=1.5cm] (f25) at (5.5, 9) { };
        \node[draw=white] (f26) at (5.5, 10.5) {$\eta^{(2)}_{\bm{\alpha}_1,j(\bx),k(\bx),r(\bx)}$ };

        \draw[-, cyan, line width=0.2mm] (f21.north) -- (f22.south);
       \draw[-, cyan, line width=0.2mm] (f22.north) -- (f23.south);
       \draw[-, cyan, line width=0.2mm] (f24.north) -- (f25.south);
       \draw[-, cyan, line width=0.2mm] (f25.north) -- (f26.south);

       \node[draw=black, line width=0.1mm, minimum height=9.5cm, minimum width=2cm] (f3) at (10.5, 5.3) { };
       \node[below, align=center] at (f3.south) 
       {$g^{\bm{\alpha}_{|\mathcal{A}(d,\beta)|}, q}$ \\
       (Lemma \ref{lemma_piece_bin_skip_ver3})};

        \node[draw=black, line width=0.1mm, draw=black, minimum height=0.6cm, minimum width=1.5cm] (f31) at (10.5, 1.5) { };
       \node[draw=black, line width=0.1mm, draw=black, minimum height=0.6cm, minimum width=1.5cm] (f32) at (10.5, 3) { };
       \node[draw=black, line width=0.1mm, draw=black, minimum height=0.6cm, minimum width=1.5cm] (f33) at (10.5, 4.5) { };
       \node[draw=black, line width=0.1mm, draw=black, minimum height=0.6cm, minimum width=1.5cm] (f34) at (10.5, 7.5) { };
       \node[draw=black, line width=0.1mm, draw=black, minimum height=0.6cm, minimum width=1.5cm] (f35) at (10.5, 9) { };
        \node[draw=white] (f36) at (10.5, 10.5) {$\eta^{(q)}_{\bm{\alpha}_{|\mathcal{A}(d,\beta)|},j(\bx),k(\bx),r(\bx)}$ };

        \draw[-, cyan, line width=0.2mm] (f31.north) -- (f32.south);
       \draw[-, cyan, line width=0.2mm] (f32.north) -- (f33.south);
       \draw[-, cyan, line width=0.2mm] (f34.north) -- (f35.south);
       \draw[-, cyan, line width=0.2mm] (f35.north) -- (f36.south);

    \node at (8, 5.5) {\huge $\cdots$};


\draw[-, cyan, line width=0.2mm] (x1.north) -- (f11.south);
\draw[-, cyan, line width=0.2mm] (x1.north) -- (f21.south);
\draw[-, cyan, line width=0.2mm] (x1.north) -- (f31.south);

\draw[-, cyan, line width=0.2mm] (x2.north) -- (f12.south);
\draw[-, cyan, line width=0.2mm] (x2.north) -- (f22.south);
\draw[-, cyan, line width=0.2mm] (x2.north) -- (f32.south);

\draw[-, cyan, line width=0.2mm] (x3.north) -- (f13.south);
\draw[-, cyan, line width=0.2mm] (x3.north) -- (f23.south);
\draw[-, cyan, line width=0.2mm] (x3.north) -- (f33.south);


\draw[-, cyan, line width=0.2mm] (x5.north) -- (f15.south);
\draw[-, cyan, line width=0.2mm] (x5.north) -- (f25.south);
\draw[-, cyan, line width=0.2mm] (x5.north) -- (f35.south);


    \end{tikzpicture}
    \vspace{15pt}
    \caption{\textbf{The second part of Step 1} Extracting the first $q$ bits of  $(\partial^{\bm{\alpha}}f_0)(\widetilde{\bx})$ for each $\bm{\alpha} \in \mathcal{A}(d,\beta) = \{\bm{\alpha}_1, \bm{\alpha}_2, \ldots, \bm{\alpha}_{|\mathcal{A}(d,\beta)|} \}$. 
    } \label{fig_skip_smooth_step2}
\end{figure}
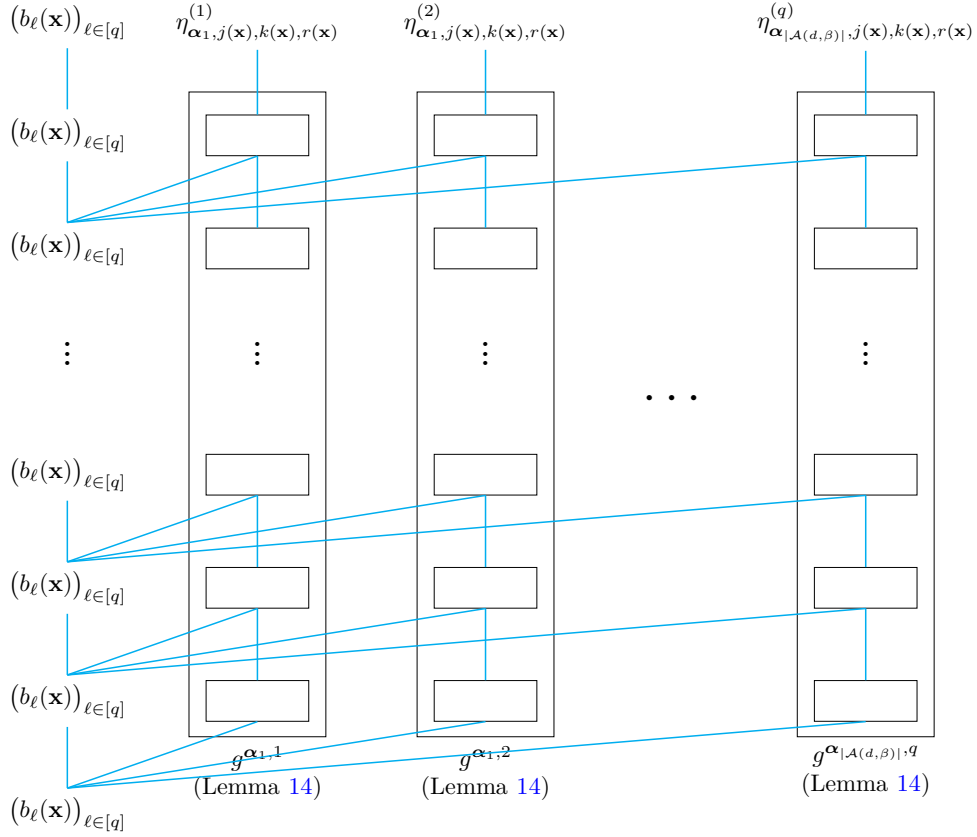

For each $\bm{\alpha} \in \mathcal{A}(d,\beta)$, 
$j \in [J]$, $k \in [K]$ and $r \in [R]$, 
we have
$$\frac{(\partial^{\bm{\alpha}}f_0)(\bm{c}_{j,k,r}) + \|\partial^{\bm{\alpha}}f_0\|_{\infty}}{2\|\partial^{\bm{\alpha}}f_0\|_{\infty}} \in [0,1].$$
Consequently, for suitable binary sequences $\{\eta^{(\ell)}_{\bm{\alpha},j,k,r}\}_{\ell \in \{1,2,\ldots\}} \in \{0,1\}^{\mathbb{N}},$
$$\frac{(\partial^{\bm{\alpha}}f_0)(\bm{c}_{j,k,r}) + \|\partial^{\bm{\alpha}}f_0\|_{\infty}}{2\|\partial^{\bm{\alpha}}f_0\|_{\infty}}
        = 0.\eta^{(1)}_{\bm{\alpha},j,k,r}\eta^{(2)}_{\bm{\alpha},j,k,r}\ldots \eta^{(q)}_{\bm{\alpha},j,k,r} \ldots_{(2)}.$$

For each $\bm{\alpha} \in \mathcal{A}(d,\beta)$ and $\ell \in [q]$,
we apply Lemma \ref{lemma_piece_bin_skip_ver3}-(ii) with $\eta_{j,k,r} := \eta^{(\ell)}_{\bm{\alpha},j,k,r}$ and obtain a DHN
$$g^{(\bm{\alpha}, \ell)} \in \DHN_{\skipp}\big(R+3,\ d(2m+n):(6K+5):1, 2K+1 \big)$$
such that
$$\mathbb{I} \Big( g^{(\bm{\alpha}, \ell)} \big((b_{\ell}(\bx))_{\ell \in [2m+n]} \big) \Big) = \eta^{(\ell)}_{\bm{\alpha},j(\bx),k(\bx),r(\bx)}.$$
Recall that the $q$-th hidden layer of $f$ outputs $(b_{\ell}(\bx))_{\ell \in [q]}$ with $q \geq 2m+n$. 
To save the $q$ bits of $\bx$ up to the $(q + R + 4)$-th hidden layer of $f$, we use the identity $b = \mathbb{I}(b - 1/2)$ for $b \in \{0,1\}$, resulting in $qd$ neurons per layer.
In addition, we construct $g^{(\bm{\alpha}, \ell)}((b_{\ell}(\bx))_{\ell \in [2m+n]})$ for each $\bm{\alpha} \in \mathcal{A}(d,\beta)$ and $\ell \in [q]$,  where the first hidden layer of $g^{(\bm{\alpha}, \ell)}$ is placed at the $(q+1)$-st hidden layer of $f$, with the output of $g^{(\bm{\alpha}, \ell)}$ at the $(q + R + 4)$-th hidden layer of $f$.
Then, the overall output of the $(q + R + 4)$-th hidden layer of $f$ is $(b_{\ell}(\bx))_{\ell \in [q]}$ and $(\eta^{(\ell)}_{\bm{\alpha},j(\bx),k(\bx),r(\bx)})_{\bm{\alpha} \in \mathcal{A}(d,\beta), \ell \in [q]}$.
Given that our network saves the values of $(b_{\ell}(\bx))_{\ell \in [2m+n]}$ until the $(q + R + 2)$-nd hidden layer, no skip connections in the construction of the networks $g^{(\bm{\alpha}, \ell)}$ are needed (see Remark \ref{rem.1}). 
Hence, each of the $(q+1)$-st to the $(q + R + 4)$-th hidden layers has
\begin{align}
    \leq dq + (6K+5)|\mathcal{A}(d,\beta)|q \label{tmp_26}
\end{align}
neurons without any skip connections. See Fig. \ref{fig_skip_smooth_step2} for an illustration. 
\medskip
\medskip

\noindent
\textbf{Step 2}: Approximating $\widetilde{f}_0(\bx)$.

\medskip

\noindent
Given the outputs $(b_{\ell}(\bx))_{\ell \in [q]}$ and
$(\eta^{(\ell)}_{\bm{\alpha},j(\bx),k(\bx),r(\bx)})_{\bm{\alpha} \in \mathcal{A}(d,\beta), \ell \in [q]}$ of the previous step in the $(q+R+4)$-th hidden layer, we now construct the network for the approximation of (\ref{eq_skip_taylor}).

For each $i \in [d]$, recall that 
\begin{align*}
    x_i - \widetilde{x}_i = \sum_{\ell=2m+n+1}^{\infty} \frac{b_{\ell}(x_i)}{2^{\ell}} - \frac{1}{2^{2m+n+1}}.  
\end{align*}
We define 
$$\operatorname{app}(x_i - \widetilde{x}_i):= 
\sum_{\ell=2m+n+1}^{q} \frac{b_{\ell}(x_i)}{2^{\ell}} + \frac{1}{2^{q+1}} - \frac{1}{2^{2m+n+1}},$$
satisfying
\begin{equation} \label{tmp_22} 
    \begin{aligned}
         \Big| \operatorname{app}(x_i - \widetilde{x}_i) - (x_i - \widetilde{x}_i) \Big| &\leq \frac{1}{2^{q+1}}, \\
         \Big| \operatorname{app}(x_i - \widetilde{x}_i) \Big| &\leq \frac{1}{2^{2m+n+1}}.
    \end{aligned}
\end{equation}
For each $\bm{\alpha} \in \mathcal{A}(d,\beta)$, recall that
$$(\partial^{\bm{\alpha}}f_0)(\widetilde{\bx}) := 2\|\partial^{\bm{\alpha}}f_0\|_{\infty} \Bigg(\sum_{\ell=1}^{\infty} \frac{\eta^{(\ell)}_{\bm{\alpha},j(\bx),k(\bx), r(\bx)}}{2^{\ell}} - \frac{1}{2} \Bigg).$$
We define
$$\operatorname{app}\big((\partial^{\bm{\alpha}}f_0)(\widetilde{\bx})\big) := 2\|\partial^{\bm{\alpha}}f_0\|_{\infty} \Bigg(\sum_{\ell=1}^{q} \frac{\eta^{(\ell)}_{\bm{\alpha},j(\bx),k(\bx), r(\bx)}}{2^{\ell}} + \frac{1}{2^{q+1}} - \frac{1}{2} \Bigg),$$
satisfying
\begin{equation} \label{tmp_23}
    \begin{aligned}
            \Big| \operatorname{app}\big((\partial^{\bm{\alpha}}f_0)(\widetilde{\bx})\big) - \big((\partial^{\bm{\alpha}}f_0)(\widetilde{\bx})\big) \Big| &\leq \frac{\|\partial^{\bm{\alpha}}f_0\|_{\infty}}{ 2^{q+1}},\\
            \Big| \operatorname{app}\big((\partial^{\bm{\alpha}}f_0)(\widetilde{\bx})\big) \Big| &\leq \|\partial^{\bm{\alpha}}f_0\|_{\infty}.
    \end{aligned}
\end{equation}

\begin{figure}[t]
    \centering
    \fbox{
        \begin{minipage}{0.9\linewidth}
            \centering
            \tikzset{every node/.style={scale=1.0}}
\begin{tikzpicture}[node distance=2cm, auto, thick, >=Stealth, scale=0.8]

        \node[draw=white] (x1) at (3, 0) { $\big(b_{\ell}(\bx)\big)_{\ell \in [q]}$};

        \node[draw=white] (x2) at (11, 0) { $\big( \eta^{(\ell)}_{\bm{\alpha},j(\bx),k(\bx),r(\bx)}\big)_{\bm{\alpha} \in \mathcal{A}(d,\beta), \ell \in [q]}$};
        
        \node[draw=white] (f11) at (7, 2) { $\left( \eta^{(\ell_0)}_{\bm{\alpha},j(\bx),k(\bx), r(\bx)} \cdot \prod_{\kappa=1}^{\|\bm{\alpha}\|_1} 
        b_{\ell_\kappa}(x_{\pi_{\bm{\alpha}}(\kappa)}) \right)_{\bm{\alpha} \in \mathcal{A}(d,\beta),\ 
        \ell_0, \ell_1, \ldots, \ell_{\|\bm{\alpha}\|_1} \in \{0,1,\ldots,q\}}$};

        \node[draw=white] (f21) at (7, 4) { $f(\bx) := \sum_{\bm{\alpha} \in \mathcal{A}(d,\beta)} 
        \frac{1}{\bm{\alpha}!} \cdot
        \operatorname{app}((\partial^{\bm{\alpha}}f_0)(\widetilde{\bx})) \cdot \prod_{\kappa=1}^{\|\bm{\alpha}\|_1} \operatorname{app}\left(
        x_{\pi_{\bm{\alpha}}(\kappa)} - \widetilde{x}_{\pi_{\bm{\alpha}}(\kappa)}\right)$};

        \draw[-, cyan, line width=0.2mm] (x1.north) -- (f11.south) node[below] {(\ref{tmp_29})};
        \draw[-, cyan, line width=0.2mm] (x2.north) -- (f11.south);
        \draw[-, cyan, line width=0.2mm] (f11.north) -- (f21.south);

    \end{tikzpicture}
        \end{minipage}
    }
    \vspace{15pt}
    \caption{\textbf{Step 2 of the proof of Theorem \ref{upper_smooth_skip}.}}
    \label{fig:skip_Taylor}
\end{figure}

For each $\bm{\alpha}:=(\alpha_1, \dots, \alpha_d) \in (\mathcal{A}(d,\beta) \setminus \{\bm{0}_d\})$, we have
\begin{align*}
    \bx^{\bm{\alpha}} = x_1^{\alpha_1} \cdots x_d^{\alpha_d} = \underbrace{x_1 \cdots x_1}_{\alpha_1}
    \underbrace{x_2 \cdots x_2}_{\alpha_2}
    \cdots \underbrace{x_d \cdots x_d}_{\alpha_d}.
\end{align*}
We choose $\pi_{\bm{\alpha}}: \{1,\ldots,\|\bm{\alpha}\|_1\} \to \{1, \dots, d\}$, such that $\pi_{\bm{\alpha}}(\kappa)$ extracts the index of $\bx^{\bm{\alpha}}$ at position $\kappa$, that is,
$$\bx^{\bm{\alpha}} = x_{\pi_{\bm{\alpha}}(1)} x_{\pi_{\bm{\alpha}}(2)} \ldots x_{\pi_{\bm{\alpha}}(\|\alpha\|_1)}.$$

We define $b_{0}(x_{i}) := 1$ for every $i \in [d]$ and $\eta^{(0)}_{\bm{\alpha},j(\bx),k(\bx), r(\bx)} := 1$ for every $\bm{\alpha} \in \mathcal{A}(d,\beta)$.
In the $(q + R + 5)$-th hidden layer (the last hidden layer), we represent
\begin{align}
    \Bigg( \eta^{(\ell_0)}_{\bm{\alpha},j(\bx),k(\bx), r(\bx)} \cdot \prod_{\kappa=1}^{\|\bm{\alpha}\|_1} 
    b_{\ell_\kappa}\left(x_{\pi_{\bm{\alpha}}(\kappa)}\right) \Bigg)_{\bm{\alpha} \in \mathcal{A}(d,\beta), \ \ell_0, \ell_1, \ldots, \ell_{\|\bm{\alpha}\|_1} \in \{0,1,\ldots,q\}}, \label{tmp_29}
\end{align}
using the identity
$\eta \prod_{\kappa=1}^{\|\bm{\alpha}\|_1} b_\kappa = \mathbb{I}(\eta + \sum_{\kappa=1}^{\|\bm{\alpha}\|_1} b_\kappa - \|\bm{\alpha}\|_1 - 1/2)$ for $\eta, b_1, \ldots,b_{\|\bm{\alpha}\|_1} \in \{0,1\}$. 
Then, the number of neurons needed for the last hidden layer is bounded above by
\begin{align}
    \sum_{\bm{\alpha} \in \mathcal{A}(d,\beta)} (q+1)^{\|\bm{\alpha}\|_1 + 1} \leq  |\mathcal{A}(d,\beta)| \cdot (q+1)^{\lceil \beta \rceil}. \label{tmp_27}
\end{align}

The output of the network is then given by
\begin{align}
    f(\bx) &:= \sum_{\bm{\alpha} \in \mathcal{A}(d,\beta)} \operatorname{app}((\partial^{\bm{\alpha}}f_0)(\widetilde{\bx})) \cdot \frac{\prod_{\kappa=1}^{\|\bm{\alpha}\|_1} \operatorname{app}\left(
        x_{\pi_{\bm{\alpha}}(\kappa)} - \widetilde{x}_{\pi_{\bm{\alpha}}(\kappa)}\right)}{\bm{\alpha}!} \label{tmp_99} \\
        &= \sum_{\bm{\alpha} \in \mathcal{A}(d,\beta)} \Bigg[ 2\|\partial^{\bm{\alpha}}f_0\|_{\infty} \Bigg(\sum_{\ell=1}^{q} \frac{\eta^{(\ell)}_{\bm{\alpha},j(\bx),k(\bx), r(\bx)}}{2^{\ell}} + \frac{1}{2^{q+1}} - \frac{1}{2} \Bigg) \nonumber\\
        & \qquad \qquad \qquad \cdot \frac{1}{\bm{\alpha}!} \prod_{\kappa=1}^{\|\bm{\alpha}\|_1} \Bigg( \sum_{\ell=2m+n+1}^{q} \frac{b_{\ell}(x_{\pi_{\bm{\alpha}}(\kappa)})}{2^{\ell}} + \frac{1}{2^{q+1}} - \frac{1}{2^{2m+n+1}} \Bigg) \Bigg],\nonumber
\end{align}
which is an affine function of the last hidden layer.   
See Fig. \ref{fig:skip_Taylor} for an illustration. 
\medskip \medskip

\noindent
\textbf{Step 3}: Verifying the size and error of the final network.

\medskip

\noindent
As the last step of our proof, we assign specific values to $n$ and $m$ and evaluate the size and precision of the network $f$.

We set
\begin{align}
    n := \left\lfloor \frac{1}{d} \log\left(\frac{L}{2}\right)\right\rfloor, \qquad m := \left\lfloor \frac{1}{d} \log\left(  \frac{p}{\log^{\frac{3}{2}}(Lp)} \right) \right\rfloor. \label{tmp_12}
\end{align}
Consequently, $R=2^{nd}\leq L/2$, $K=2^{md}\leq p/\log^{\frac{3}{2}}(Lp),$ and $q = \lceil \beta \rceil(2m+n) \leq (2\beta+2)\log(Lp)/d$. Given that by assumption $L \geq \log^2 (Lp)\geq \log(Lp)\log(L)$ and $L\geq c$, for a sufficiently large constant $c$ depending on $\beta$ and $d,$ the number of hidden layers of the network $f$ is bounded above by
\begin{align*}
    q+R+5 
    \leq \frac{(2\beta+2) \log(Lp)}{d} + \frac{L}{2} + 5 
    \leq L \, \frac{2\beta+2}{d \log(L)} + \frac{L}{2} + 5
    \leq L.
\end{align*}
According to (\ref{tmp_26}) and (\ref{tmp_27}), the maximal number of neurons in each hidden layer of $f$ is bounded by 
\begin{align}
     &\Big(dq + (6K+5)|\mathcal{A}(d,\beta)|q\Big) \vee \Big( |\mathcal{A}(d,\beta)| (q+1)^{\lceil \beta \rceil} \Big) \nonumber \\
     &\leq \Big(|\mathcal{A}(d,\beta)| (q + (6K+5)q) \Big) \vee \Big( |\mathcal{A}(d,\beta)| (q+1)^{\lceil \beta \rceil} \Big) \nonumber \\
     & \leq |\mathcal{A}(d,\beta)| \Big( 6q(K + 1) \vee (q+1)^{\lceil \beta \rceil}  \Big). \label{tmp_6}
\end{align}
Given that by assumption $p \geq \log^{\beta+2} (Lp)$ and $p\geq c$, for a sufficiently large constant $c$ depending on $\beta$ and $d,$ we have
\begin{align*}
    |\mathcal{A}(d,\beta)| \cdot 6q(K + 1) 
    &\leq \frac{12 |\mathcal{A}(d,\beta)|  (\beta+1) \log(Lp)}{d} \left( \frac{p}{\log^{\frac{3}{2}}(Lp) } + 1 \right) \\
    & \leq \frac{12 |\mathcal{A}(d,\beta)| (\beta+1)}{d} \left( \frac{p}{\log^{\frac{1}{2}}(p) } + \frac{p}{\log^{\beta+1} (p)} \right) \\
     &\leq p
\end{align*}
and
\begin{align*}
    |\mathcal{A}(d,\beta)| \cdot \big(q+1\big)^{\lceil \beta \rceil} 
    &\leq |\mathcal{A}(d,\beta)| \left( \frac{(2\beta+2)\log(Lp)}{d} + 1 \right)^{\beta+1}\\
    &\leq |\mathcal{A}(d,\beta)| \left(\frac{2\beta+3}{d}\right)^{\beta+1}\log^{\beta+1}(Lp) \\
    &\leq |\mathcal{A}(d,\beta)| \left(\frac{2\beta+3}{d}\right)^{\beta+1} \frac{p}{\log(p)} \\
    &\leq p.
\end{align*}
Hence, (\ref{tmp_6}) is bounded above by $p$.
Given that each hidden layer has at most $d$ skip connected neurons, we get
$f \in \DHN_{\skipp}(L,d:p:1,d)$.

\medskip \medskip

Next we bound the approximation error of the network $f$ defined in \eqref{tmp_99} for approximating $f_0$.
For each $\bm{\alpha} \in \mathcal{A}(d,\beta)$, recall that $(\bx-\widetilde{\bx})^{\bm{\alpha}} = \prod_{\kappa=1}^{\|\bm{\alpha}\|_1} 
        (x_{\pi_{\bm{\alpha}}(\kappa)} - \widetilde{x}_{\pi_{\bm{\alpha}}(\kappa)})$.
By (\ref{tmp_22}), (\ref{tmp_23}) and Lemma \ref{lemma_prod_diff}, we have 
\begin{align*}
        &\left|\operatorname{app}\big((\partial^{\bm{\alpha}}f_0)(\widetilde{\bx})\big) \cdot \prod_{\kappa=1}^{\|\bm{\alpha}\|_1} 
        \operatorname{app}\left(x_{\pi_{\bm{\alpha}}(\kappa)} - \widetilde{x}_{\pi_{\bm{\alpha}}(\kappa)}\right)
        - (\partial^{\bm{\alpha}}f_0)(\widetilde{\bx}) \cdot (\bx-\widetilde{\bx})^{\bm{\alpha}}\right| \\
        &= \left|\operatorname{app}\big((\partial^{\bm{\alpha}}f_0)(\widetilde{\bx})\big) \cdot \prod_{\kappa=1}^{\|\bm{\alpha}\|_1} 
        \operatorname{app}\left(x_{\pi_{\bm{\alpha}}(\kappa)} - \widetilde{x}_{\pi_{\bm{\alpha}}(\kappa)}\right)
        - (\partial^{\bm{\alpha}}f_0)(\widetilde{\bx}) \cdot \prod_{\kappa=1}^{\|\bm{\alpha}\|_1} 
        \left(x_{\pi_{\bm{\alpha}}(\kappa)} - \widetilde{x}_{\pi_{\bm{\alpha}}(\kappa)}\right) \right| \\
        &\leq \|\partial^{\bm{\alpha}} f_0\|_{\infty}
        \left( \frac{1}{2^{2m+n+1}} \right)^{\|\bm{\alpha}\|_1} 
        \left( \frac{1}{2^{q+1}} + \frac{\|\bm{\alpha}\|_1}{2^{q-2m-n}} \right)\\
        &\leq \frac{\|\partial^{\bm{\alpha}} f_0\|_{\infty}}{2^q}. 
\end{align*}
Hence, we get
\begin{align*}
    &\left|f(\bx) - \sum_{\bm{\alpha} \in \mathcal{A}(d,\beta)} (\partial^{\bm{\alpha}}f_0)(\widetilde{\bx}) \cdot \frac{(\bx-\widetilde{\bx})^{\bm{\alpha}}}{\bm{\alpha}!} \right| \\
    &\leq \sum_{\bm{\alpha} \in \mathcal{A}(d,\beta)}  \left| \operatorname{app}\big((\partial^{\bm{\alpha}}f_0)(\widetilde{\bx})\big) \cdot \frac{\prod_{\kappa=1}^{\|\bm{\alpha}\|_1} \operatorname{app}\left(
        x_{\pi_{\bm{\alpha}}(\kappa)} - \widetilde{x}_{\pi_{\bm{\alpha}}(\kappa)}\right)}{\bm{\alpha}!} - (\partial^{\bm{\alpha}}f_0)(\widetilde{\bx}) \cdot \frac{(\bx-\widetilde{\bx})^{\bm{\alpha}}}{\bm{\alpha}!} \right| \\    
    &\leq \sum_{\bm{\alpha} \in \mathcal{A}(d,\beta)} \frac{\|\partial^{\bm{\alpha}} f_0\|_{\infty}}{2^q} \\ 
    &\leq  \frac{\left\| f_0 \right\|_{\mathcal{C}^{\beta}}}{2^{ \beta (2m+n)}}.
\end{align*}
By Lemma \ref{lemma_taylor}, we further have
\begin{align*}
    \left| \sum_{\bm{\alpha} \in \mathcal{A}(d,\beta)} (\partial^{\bm{\alpha}}f_0)(\widetilde{\bx}) \cdot \frac{(\bx-\widetilde{\bx})^{\bm{\alpha}}}{\bm{\alpha}!} -f_0(\bx) \right| 
    \leq&  \left\| f_0 \right\|_{\mathcal{C}^{\beta}} \cdot \|\bx - \widetilde{\bx} \|_{\infty}^\beta \\
    \leq &   \frac{\left\| f_0 \right\|_{\mathcal{C}^{\beta}}}{2^{\beta(2m+n+1)}}. 
\end{align*}
In summary, for every $\bx \in [0,1]^d$,
\begin{align*}
    \big| f(\bx) - f_0(\bx) \big| 
    &\leq \frac{2 \left\| f_0 \right\|_{\mathcal{C}^{\beta}}}{2^{\beta(2m+n)}} \\
    &\leq 
    2 \cdot 8^{\beta} \cdot \left\| f_0 \right\|_{\mathcal{C}^{\beta}} \left( \frac{L}{2} \cdot \frac{p^2}{\log^3(Lp) } \right)^{-\frac{\beta}{d}} \\
    &\leq 2 \cdot 8^{\beta} \cdot 2^{\beta}\cdot \left\| f_0 \right\|_{\mathcal{C}^{\beta}} \left(\frac{\log^3 (Lp)}{Lp^2}\right)^{\frac{\beta}{d}},
\end{align*}
using for the second inequality that by \eqref{tmp_12}, 
$$n \geq \frac{1}{d} \log\left(\frac{L}{2}\right)-1, \qquad m \geq \frac{1}{d} \log\left(  \frac{p}{\log^{\frac{3}{2}}(Lp)} \right)-1.$$  
\end{proof}

\section{Proofs of Section \ref{sec_linear} }
\subsection{Proof of Proposition \ref{thm_lin_include}} \label{app_proof_lin_include}
\begin{proof}
To define the neural network output we use the layerwise definition in (\ref{DHN_linear}), (\ref{DHN_linear_L}) and (\ref{DHN_linear_others}).
For $\ell \in [L-1]$, we denote by $\bm{g}^{(\ell)}: \mathbb{R}^d \to \mathbb{R}^{p_{\ell}}$ and $\bm{h}^{(\ell)}: \mathbb{R}^d \to \mathbb{R}^{s}$ the vector-valued functions consisting of the first $p_{\ell}$ elements and the remaining $s$ elements of $\bm{f}^{(\ell)}$, respectively. Moreover,
$\bm{f}^{(1)}|_{[\bx_1, \bx_2]} : [0,1] \to \{0,1\}^{p_1+s}$ is defined by 
\begin{align*}
    \bm{f}^{(1)}\big|_{[\bx_1, \bx_2]}(t) :=
\sigma_s \Big(W_{0} \big((1-t)\bx_1 + t\bx_2\big) - \bm{b}_{0}\Big).
\end{align*}
For each $i \in [p_1]$, the $i$-th component of $\bm{f}^{(1)}|_{[\bx_1, \bx_2]}(t)$ can take values of the form $\mathbb{I}(t \geq c_i)$ or $\mathbb{I}(t \leq c_i)$ for real numbers $c_i$.
Hence, $\bm{g}^{(1)}|_{[\bx_1, \bx_2]}$ is a piecewise constant vector-valued function with at most $p_1$ breakpoints on $[0,1]$, and thus consists of at most $p_1+1$ pieces.
By definition, $\bm{h}^{(1)}|_{[\bx_1, \bx_2]}(t)$ is an affine vector-valued function of $t$. 

Assume that for $\ell \in [L-2]$,
there exists an interval partition $A_1, \ldots, A_{K}$ of $[0,1]$ with $K \leq \prod_{i=1}^{\ell} (p_i + 1)$ such that
$\bm{g}^{(\ell)}|_{[\bx_1, \bx_2]}(t)$ and $\bm{h}^{(\ell)}|_{[\bx_1, \bx_2]}(t)$ are constant and affine vector-valued functions of $t$, respectively, on each element of the partition.
For each $k \in [K]$,
$$W_{\ell} \begin{pmatrix}
\bm{g}^{(\ell)}|_{[\bx_1, \bx_2]}(t)  \\
\bm{h}^{(\ell)}|_{[\bx_1, \bx_2]}(t)
\end{pmatrix} - \bm{b}_{\ell} $$
is an affine vector-valued function of $t \in A_k$. 
For $t \in A_k$, 
each element of  
$\bm{g}^{(\ell+1)}|_{[\bx_1, \bx_2]}(t)$
takes values of the form 
$\mathbb{I}(t \geq c_i)$ or $\mathbb{I}(t \leq c_i)$ for real numbers $c_i.$ Hence, $\bm{g}^{(\ell+1)}|_{[\bx_1, \bx_2]}(t)$ is a piecewise constant vector-valued function with at most $p_{\ell+1}$ breakpoints on $A_k$, and thus consists of at most $p_{\ell+1}+1$ pieces.
In addition, 
$\bm{h}^{(\ell+1)}|_{[\bx_1, \bx_2]}(t)$
is an affine vector-valued function on $t \in A_k$.

As a result, there exists an interval partition $A'_1, \ldots, A'_{K'}$ of $[0,1]$ with $K' \leq (p_{\ell+1}+1)K \leq \prod_{i=1}^{\ell+1} (p_i + 1)$ such that $\bm{g}^{(\ell+1)}|_{[\bx_1, \bx_2]}(t)$ and $\bm{h}^{(\ell+1)}|_{[\bx_1, \bx_2]}(t)$ are constant and affine vector-valued functions of $t$, respectively, on each element of the partition.
Iterating this process,  
we can show that $\bm{f}^{(L)}|_{[\bx_1, \bx_2]}$ is a piecewise constant vector-valued function with at most $\prod_{\ell=1}^L (p_{\ell} + 1)$ pieces, and $f|_{[\bx_1, \bx_2]}$ is a piecewise constant function with the same bound on the number of pieces.    
\end{proof}

\subsection{Bit extraction using lin-DHNs}
In this section, 
$b_1(x), b_2(x), \ldots, b_\ell(x), \ldots \in \{0,1\}$
as the binary digits of $x \in [0,1]$, that is,
\begin{align*}
    x = 0.b_1(x) b_2(x) b_3(x)  \ldots_{(2)}:= 
\sum_{\ell=1}^{\infty} 2^{-\ell} b_{\ell}(x).
\end{align*}
Further details can be found in Appendix \ref{app_proof_sq}.
For $\bx = (x_1,\ldots,x_d) \in [0,1]^d$,
we set 
$b_\ell(\bx) := (b_\ell(x_i))_{i \in [d]}$. 

By  Lemma \ref{lemma_bit_extract}, there exists $f \in \DHN_{\skipp}(L, 1:L:1, 1)$ computing  in the last hidden layer $[0,1] \ni x \mapsto (b_{\ell}(x))_{\ell \in [L]}.$ By \eqref{skip_lin_include}, $\DHN_{\skipp}(L, 1:L:1, 1) \subseteq \DHN_{\lin}(L, 1:L:1, 1),$ and this proves

\medskip

\begin{lemma} \label{lemma_bit_extract_linear_ver1}
     For any integer $L \geq 1,$ 
    there exists a neural network $f \in \DHN_{\lin}(L, 1:L:1, 1)$ computing in the last hidden layer $[0,1] \ni x \mapsto (b_{\ell}(x))_{\ell \in [L]}.$
\end{lemma}

\medskip

\begin{lemma} \label{lemma_bit_extract_linear_ver2}
     For any integer $L \geq 1$, 
    there exists a network
\begin{align*}
    f \in \DHN_{\lin}(L, 1:1:1, 1)
\end{align*}
such that for input $x \in [0,1]$ and any $\ell \in [L]$, the $\ell$-th hidden layer contains a neuron with activation $b_{\ell}(x)$.
{  For $\ell\in[L-1]$, the remaining neuron in the $\ell$-th hidden layer computes $x - \sum_{i=1}^{\ell-1} 2^{-i} b_{i}(x)$.} 
\end{lemma}
\begin{proof}
    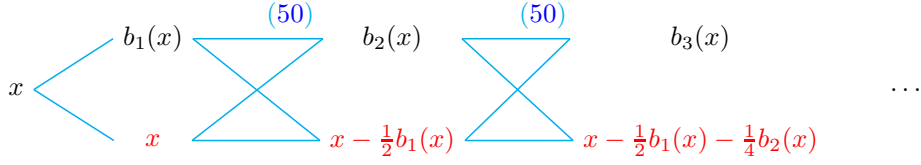
\begin{figure}[t]
    \centering
        \begin{tikzpicture}[node distance=2cm, auto, thick, >=Stealth, scale=0.9]
        \node[draw=white] (x) at (-0.5, 3.75) {$x$};        

        \node[draw=white, align=center] (f11) at (1.5, 4.5) {$b_1(x)$};

        \node[draw=white, align=center, minimum width=1cm] (f12) at (1.5, 3) {\color{red} $x$};
        
        \node[draw=white, align=center, minimum width=1.8cm] (f21) at (5, 4.5) {$b_2(x)$};

        \node[draw=white, align=center] (f22) at (5, 3) {\color{red} $x-\frac{1}{2}b_1(x)$};
        
        \node[draw=white, align=center, minimum width=3.4cm] (f31) at (9.5, 4.5) {$b_3(x)$};

        \node[draw=white, align=center] (f32) at (9.5, 3) {\color{red} $x-\frac{1}{2}b_1(x)-\frac{1}{4}b_2(x)$};

        \node at (12.5, 3.75) {$\ldots$};

        \draw[-, cyan, line width=0.2mm] (x.east) -- (f11.west);
        \draw[-, cyan, line width=0.2mm] (x.east) -- (f12.west);

        \draw[-, cyan, line width=0.2mm] (f11.east) -- (f21.west) node[near end, above] {(\ref{tmp_28})};
        \draw[-, cyan, line width=0.2mm] (f11.east) -- (f22.west);
        \draw[-, cyan, line width=0.2mm] (f12.east) -- (f21.west) ;        
        \draw[-, cyan, line width=0.2mm] (f12.east) -- (f22.west);

        \draw[-, cyan, line width=0.2mm] (f21.east) -- (f31.west) node[near end, above] {(\ref{tmp_28})};
        \draw[-, cyan, line width=0.2mm] (f22.east) -- (f31.west);
        \draw[-, cyan, line width=0.2mm] (f21.east) -- (f32.west);
        \draw[-, cyan, line width=0.2mm] (f22.east) -- (f32.west);



    \end{tikzpicture} \vspace{15pt}
    \caption{\textbf{Illustration for the proof of Lemma \ref{lemma_bit_extract_linear_ver2}.} Black neurons use the Heaviside activation function, while red neurons use the linear activation function.}
\end{figure}

In the first hidden layer, we assign $2$ neurons to the respective values $b_{1}(x) = \mathbb{I}(x-1/2)$ and $x$.
Assume that the $\ell$-th hidden layer with $1 \leq \ell \leq L-2$ has $2$ neurons whose respective values are 
$b_{\ell}(x)$ and $x - \sum_{i=1}^{\ell-1} 2^{-i} b_{i}(x)$. 
Then, for the $(\ell+1)$-st hidden layer, we can represent
\begin{align}
    b_{\ell+1}(x) 
    = \mathbb{I} \left( \left( x - \sum_{i=1}^{\ell-1} \frac{b_{i}(x)}{2^i}\right) - \frac{b_{\ell}(x)}{2^{\ell}} - \frac{1}{2^{\ell+1}}\right) \label{tmp_28}
\end{align}
and $x - \sum_{i=1}^{\ell} 2^{-i} b_{i}(x)$. 

Iterating this process, for every $\ell \in [L-1]$, we have one Heaviside neuron and one linear neuron in the $\ell$-th hidden layer, whose respective outputs are $b_{\ell}(x)$ and $x - \sum_{i=1}^{\ell-1} 2^{-i} b_{i}(x)$.
Finally, we construct one neuron in the last hidden layer that outputs
\begin{align*}
    b_{L}(x) 
    = \mathbb{I} \left(\left( x - \sum_{\ell=1}^{L-2} \frac{b_{\ell}(x)}{2^\ell}\right) - \frac{b_{L-1}(x)}{2^{L-1}} - \frac{1}{2^{L}}\right).
\end{align*}

\end{proof}

\subsection{Proof of Theorem \ref{VCdim-linear}} \label{app_proof_lin_VC}

\begin{proof}[Proof of the upper bound in Theorem \ref{VCdim-linear}]
The proof is inspired by the proof of Theorem 7 in \cite{peter2}. Let $W$ be the total number of varying parameters in $\DHN_{\lin}(L,d\!:p:\!1,s).$
We denote $f_{\bw}$ as the function in $\DHN_{\lin}(L,d\!:p\!:1,s)$ induced by the parameter vector ${\bw}\in\mathbb{R}^{W}$. To study variations in ${\bw},$ we introduce the notation $h_{\bx}(\bw):=f_{\bw}({\bx}).$

Consider fixed $m\geq W,$ and arbitrary points ${\bx}_{1},\ldots,{\bx}_{m}\in\mathcal{X}$. In the following steps, we will first derive an upper bound for the growth function
$$\Pi_{L,d:p:1,s}(m)=\Big|\Big\{\Big(\mathbb{I}\big(h_{{\bx}_{1}}({\bw})\big),\ldots,\mathbb{I}\big(h_{{\bx}_{m}}({\bw})\big)\Big),\;{{\bw}}\in\mathbb{R}^{W}\Big\}\Big|.$$ This quantity represents the number of distinct sign patterns that the neural network can produce for the inputs ${\bx}_{1},\ldots,{\bx}_{m}$ and varying parameters.

The key to bound the growth function is to construct a partition $\mathcal{S}$ of the parameter space $\mathbb{R}^{W}$ such that, within each region $S\in\mathcal{S}$, the sign patterns of the functions $h_{{\bx}_{j}}({\bw})$, for all $j\in[m]$, remain unchanged as ${\bw}$ varies in $S\subseteq\mathbb{R}^{W}$. Let $\mathcal{S}_0 = {\mathbb{R}^W}$. Set $p_0:=d$, $p_{\ell}:=p$ for $\ell\in[L]$, $s_0:=0$, $s_{\ell}:=s$ for $\ell\in[L-1]$, and $s_{L}:=0$. Starting with $\mathcal{S}_0$, more refined partitions $\mathcal{S}_{1},\ldots,\mathcal{S}_{L}$ are built layer by layer such that the following two properties hold:
\begin{enumerate}
\item[(i)] For any $\ell\in[L]$, $S\in\mathcal{S}_{\ell-1}$, ${\bx}_{j}$ with $j\in[m]$, and $k\in[p_{\ell}+s_{\ell}]$,
$$S \ni \bw \mapsto h^{\ell}_{\bx_{j},k}({\bw})$$
returns the input of the $k$-th node in the $\ell$-th hidden layer in response to ${\bx}_{j}$ and parameter $\bw.$ This function is a polynomial of total degree $\leq \ell$.
\item[(ii)] For any $\ell\in[L]$, 
\begin{equation}\label{partition-number}
|\mathcal{S}_{\ell}|\leq2\left(\frac{2em\ell p_{\ell}}{z_{\ell}}\right)^{z_{\ell}}|\mathcal{S}_{\ell-1}|,
\end{equation}
where $$z_{\ell}:=p_{\ell}(p_{\ell-1}+s_{\ell-1}+1)+\sum_{i=1}^{\ell-1}s_{i}(p_{i-1}+s_{i-1}+1),$$ and the convention that the empty sum is 0.
\end{enumerate}

We now detail the inductive construction of $\mathcal{S}_{1},\ldots,\mathcal{S}_{L}$. As before, we can assume that in each hidden layer $\ell\in[L-1]$, the last $s_{\ell}$ nodes use a linear activation function. For each ${\bx}_{j}$, $j=1,\ldots,m,$ the input of any node in the first hidden layer is linear in the parameters and thus a polynomial of degree at most 1 in $\bw$. This verifies (i) when $\ell=1$. To see that (ii) also holds when $\ell=1$, observe that the last $s_1$ nodes in the first hidden layer use a linear activation function, so the output polynomials remain unchanged as ${\bw}$ varies. The input of any node with the Heaviside activation function in the first hidden layer (the first $p_1$ nodes) is a polynomial of degree at most 1 in at most $z_1$ variables of ${\bw}$. Applying Lemma~\ref{poly-vc} to the collection of polynomials $\mathcal{S}_0=\mathbb{R}^W \ni \bw \mapsto h^1_{{\bx}_{j},k}({\bw})$ with $k\in[p_1]$ and $j\in[m]$, we conclude that there are at most $$N_{1}=2\left(\frac{2emp_1}{z_{1}}\right)^{z_{1}}$$ distinct signs patterns when ${\bw}$ varies. We obtain $\mathcal{S}_1$ by refining $\mathcal{S}_0$ so that within each region $S \in \mathcal{S}_1$, the aforementioned polynomials have fixed signs as ${\bw}$ varies. This gives $$|\mathcal{S}_1|\leq N_{1}\cdot|\mathcal{S}_0|,$$ implying (ii) for $
\ell=1$.

For the induction step, suppose we perform successive partitions up to $\ell-1$ for $\ell \in\{2,\ldots,L\}$, obtaining refinements $\mathcal{S}_1, \ldots, \mathcal{S}_{\ell-1}$. We now define $\mathcal{S}_\ell$. By the induction hypothesis, for any $S\in\mathcal{S}_{\ell-1}$, $j\in[m]$ and $k\in[p_{\ell-1}+s_{\ell-1}]$, $S\ni\bw \mapsto h^{\ell-1}_{{\bx}_{j},k}({\bw})$ is a polynomial of total degree $\leq\ell-1$. Passing through the $p_{\ell-1}+s_{\ell-1}$ nodes in the $(\ell-1)$-st hidden layer, the polynomials remain unchanged by the definition of $S$. Consequently, $S\ni\bw \mapsto h^{\ell}_{{\bx}_{j},k}({\bw})$ is a polynomial of total degree $\leq \ell$, for all $k\in[p_{\ell}+s_{\ell}]$ and $j\in[m]$. This verifies (i). Moreover, for all $k\in[p_{\ell}]$, $S\ni\bw \mapsto h^{\ell}_{{\bx}_{j},k}({\bw})$ is a polynomial in at most $z_{\ell}$ variables of ${\bw}$. Applying Lemma~\ref{poly-vc} to the collection of polynomials $S\ni\bw \mapsto h^{\ell}_{{\bx}_{j},k}({\bw})$ with $k\in[p_{\ell}]$ and $j\in[m]$, we know that there are $$\leq N_{\ell}=2\left(\frac{2em\ell p_{\ell}}{z_{\ell}}\right)^{z_\ell}$$ distinct signs patterns when ${\bw}$ varies in any $S\in\mathcal{S}_{\ell-1}$. The successive partition is then based on a refinement of $\mathcal{S}_{\ell-1}$ such that within each region $S'\in\mathcal{S}_{\ell}$, all the above mentioned polynomials have fixed signs when ${\bw}$ varies. Thus, for each region $S\in\mathcal{S}_{\ell-1}$, we partition it into at most $N_{\ell}$ subregions, resulting in a refined partition $\mathcal{S}_{\ell}$ that validates (ii). 

Finally, we consider the output layer. By construction, for any $S\in\mathcal{S}_{L}$, the output of each node in the $L$-th hidden layer is fixed (either 0 or 1). Therefore, $h_{\bx_j}$ is a linear combination of at most $p+1$ variables. Applying Lemma~\ref{poly-vc}, 
\begin{equation}\label{output-patterns}
\leq N_{L+1}=2\left(\frac{2em}{p+1}\right)^{p+1}   
\end{equation}
distinct sign patterns can occur. 

For any $m$, arbitrarily chosen ${\bx}_{1},\ldots, {\bx}_{m}\in\mathcal{X}$, the growth function can be bounded by \eqref{partition-number} and \eqref{output-patterns} as
\begin{align}
\Pi_{L,d:p:1,s}(m)&\leq N_{L+1}\cdot|\mathcal{S}_{L}|\leq\prod_{\ell=1}^{L+1}N_{\ell}=2^{L+1}\left[\prod_{\ell=1}^{L+1}\left(\frac{2em\ell p_{\ell}}{z_{\ell}}\right)^{z_{\ell}}\right],\label{b-cardinality-neural}
\end{align}
where we set $z_{L+1}:=p+1$ and, for this formula, $p_{L+1}:=1/(L+1)$ such that $(L+1)p_{L+1}=1.$ For $\ell\in[L+1]$, define $$\kappa_{\ell}:=\frac{2em\ell p_{\ell}}{z_{\ell}},\quad\lambda_{\ell}:=\frac{z_{\ell}}{\sum_{\ell=1}^{L+1}z_{\ell}}.$$
Applying the weighted AM-GM Inequality (see, e.g., \cite{hardy1934inequalities}) to \eqref{b-cardinality-neural} yields
\begin{align*}
\Pi_{L,d:p:1,s}(m)&\leq2^{L+1}\left(\prod_{\ell=1}^{L+1}\kappa_{\ell}^{\lambda_{\ell}}\right)^{\sum_{\ell=1}^{L+1}z_{\ell}}\\
&\leq2^{L+1}\left(\sum_{\ell=1}^{L+1}\lambda_{\ell}\kappa_{\ell}\right)^{\sum_{\ell=1}^{L+1}z_{\ell}}\\
&\leq2^{L+1}\left(\frac{2em\sum_{\ell=1}^{L+1}\ell p_{\ell}}{\sum_{\ell=1}^{L+1}z_{\ell}}\right)^{\sum_{\ell=1}^{L+1}z_{\ell}}.
\end{align*}
According to the definition of the VC-dimension (Definition~\ref{vc-def}), it is necessary to have
$$2^{\operatorname{VC}\left(\DHN_{\lin}(L,d:p:1,s)\right)}\leq2^{L+1}\left[\frac{2e(\sum_{\ell=1}^{L+1}\ell p_{\ell})\operatorname{VC}\left(\DHN_{\lin}(L,d:p:1,s)\right)}{\sum_{\ell=1}^{L+1}z_{\ell}}\right]^{\sum_{\ell=1}^{L+1}z_{\ell}}.$$ Provided $L\geq1$ and $p\geq2$, we have $\sum_{\ell=1}^{L+1}\ell p_{\ell}\geq3$ so that $2e(\sum_{\ell=1}^{L+1}\ell p_{\ell})\geq6e\geq16$. Lemma~18 of \cite{peter2} shows that if $2^{a}\leq2^{b}(ar/w)^{w}$ for some $r\geq16$ and $a\geq w\geq b\geq0$, then, $a\leq b+w\log (2r\log r)$. Applying this inequality with $a=\operatorname{VC}\left(\DHN_{\lin}(L,d:p:1,s)\right)$, $b=L+1$, $r=2e(\sum_{\ell=1}^{L+1}\ell p_{\ell})$ and $w=\sum_{\ell=1}^{L+1}z_{\ell}$,we obtain
\begin{align}
\operatorname{VC}\big(\DHN_{\lin}(L,d:p:1,s)\big)\leq L+1+\Big(\sum_{\ell=1}^{L+1}z_{\ell}\Big)\log \big(2r\log(r)\big).\label{vc-initial}
\end{align}
Since $r\geq 6e$, $L\geq1$ and $p\geq2$, we have
\begin{align}
\log \big(2r\log(r)\big) \leq\log (2r^{\frac{3}{2}})=1+\frac{3}{2}\log r \leq 1.75 \log r \leq7.06\log(pL),\label{bound-for-log}
\end{align}
using for the last inequality $(L+1)p_{L+1}=1$ and
$$r = 2e\Big( 1 + p\sum_{\ell=1}^L \ell \Big) =  2e + eL(L+1)p \leq 2e + e (Lp)^2 \leq (Lp)^{7.06}.$$
Additionally, with $L\geq1$, and $p\geq s\vee d\vee 2$, 
\begin{align}
\sum_{\ell=1}^{L+1}z_{\ell}&=(d+1)[p+(L-1)s]+(L-1)p(p+s+1)+s(p+s+1)\Big(\sum_{\ell=0}^{L-2}\ell\Big)+(p+1)\nonumber\\
    &\leq(p+1)^2+(L-1)(2ps+s+p^2+p)+\frac{(L-1)(L-2)}{2}s(p+s+1)\nonumber\\
    &\leq\frac{5}{4}L^2ps+\frac{11}{4}Lp^2.\label{bound-for-sumz}
\end{align}
Plugging \eqref{bound-for-log} and \eqref{bound-for-sumz} into \eqref{vc-initial}, and using $2L \leq Lp$, we finally obtain
\begin{align*}
\operatorname{VC}\left(\DHN_{\lin}(L,d:p:1,s)\right)&\leq 2L+7.06\left(\frac{5}{4}L^2ps+\frac{11}{4}Lp^2\right)\log(pL)\\
&\leq 30 \left(L^2ps \vee Lp^2\right)\log(pL).
\end{align*}
\end{proof}

We now establish a version of Lemma \ref{lemma_piece_bin_skip_ver3} for linear neurons augmented DHNs. The notation is mostly the same. In particular, 
\begin{align*}
    x = 0.b_1(x) b_2(x) b_3(x)  \ldots_{(2)}:= 
\sum_{\ell=1}^{\infty} 2^{-\ell} b_{\ell}(x).
\end{align*}
defines a binary representation of $x$ and 
for $\bx = (x_1,\ldots,x_d) \in [0,1]^d$,
\[b_\ell(\bx) := (b_\ell(x_i))_{i \in [d]}.\]

For any given positive integer $d$ and non-negative integers $m, n$ and $t$,
we define 
\begin{align}
    J:=2^{(m+t)d}, \quad K:=2^{nd}, \quad \text{and} \ \ R:=2^{td}.
    \label{eq.0fweni_2}
\end{align}
We partition $[0,1]^d$ into $\{U_j\}_{j \in [J]}$ such that $\bx$ and $\bz$ belong to the same partition element if and only if $b_\ell(\bx) = b_\ell(\bz)$ for all $\ell \in [m+t]$.  
Similarly, we partition $[0,1]^d$ into $\{V_k\}_{k \in [K]}$ such that $\bx$ and $\bz$ belong to the same partition element if and only if $b_\ell(\bx) = b_\ell(\bz)$ for all $\ell \in \{m+t+1, m+t+2, \ldots, m+n+t\}$. 
Finally, we partition $[0,1]^d$ into $\{W_r\}_{r \in [R]}$ such that $\bx$ and $\bz$ belong to the same partition element if and only if $b_\ell(\bx) = b_\ell(\bz)$ for all $\ell \in \{m+n+t+1, m+n+t+2, \ldots, m+n+2t\}$. 

For each $\bx \in [0,1]^d$, there exist unique $j(\bx) \in [J]$, $k(\bx) \in [K]$ and $r(\bx) \in [R]$ such that 
$$\bx \in U_{j(\bx)} \cap V_{k(\bx)} \cap W_{r(\bx)}.$$ 
\begin{lemma} \label{lemma_piece_bin_lin_ver3}
For any positive integer $d$, any non-negative integers $t,m,n,$ and any binary values $\{\eta_{j,k,r}\}_{j \in [J], k \in [K], r \in [R]} \in \{0,1\}^{JKR}$,
there exists a network 
    $$h \in \DHN_{\lin}\Big(2R+2,\ d(m+n+2t):d(m+n+2t) + \big(2^{md} \vee (3K+1)\big):1, K \Big)$$ such that 
    $$h\big((b_{\ell}(\bx))_{\ell \in [m+n+2t]}\big) = \eta_{j(\bx), k(\bx), r(\bx)}.$$
\end{lemma}

\begin{proof}
For each $j \in [J]$ and $k \in [K]$, we define
$$\lambda_{j,k} := {0.\eta_{j,k,1} \eta_{j,k,2} \ldots \eta_{j,k,R}}_{(2)}.$$
Let $\bz(\bx) := (b_{\ell}(\bx))_{\ell \in [m+n+2t]}$. The proof proceeds in the following three steps.  

\clearpage

\noindent
\textbf{Step 1}: Representing $\bz(\bx)$
and $(\lambda_{j(\bx),k})_{k \in [K]}$, in the $(R+1)$-st hidden layer of $h(\bz(\bx))$.

\medskip

        \begin{figure}[t]
    \centering
    \tikzset{every node/.style={scale=0.8}}
\begin{tikzpicture}[node distance=2cm, auto, thick, >=Stealth, scale=0.8]

\node[draw=white] (x) at (0.5, 0) { $\bz(\bx) := \big(b_{\ell}(\bx)\big)_{\ell \in [m+n+2t]}$};
;

        \node[draw=white] (f11) at (0.5, 2.5) { $\bz(\bx)$};

        \node[draw=white] (f12) at (5.5, 2.5) { $\big(\mathbb{I}(j(\bx)=j)\big)_{j \in \{1,\ldots,2^{md}\}}$};

        \node[draw=white] (f21) at (0.5, 5) { $\bz(\bx)$};

        \node[draw=white] (f22) at (5.5, 5) { $\big(\mathbb{I}(j(\bx)=j)\big)_{j \in \{2^{md} + 1,\ldots, 2\cdot 2^{md}\}}$};

        \node[draw=white] (f23) at (12, 5) { \color{red} $\left(\sum_{j=1}^{2^{md}} \lambda_{j,k} \mathbb{I}(j(\bx)=j)\right)_{k \in [K]} $};

        \node[draw=white] (f31) at (0.5, 7.5) { $\bz(\bx)$};

        \node[draw=white] (f32) at (5.5, 7.5) { $\big(\mathbb{I}(j(\bx)=j)\big)_{j \in \{2\cdot2^{md} + 1,\ldots, 3\cdot 2^{md}\}}$};

        \node[draw=white] (f33) at (12, 7.5) { \color{red} $\left(\sum_{j=1}^{2\cdot 2^{md}} \lambda_{j,k} \mathbb{I}(j(\bx)=j)\right)_{k \in [K]} $};

        \node at (5.5, 9.5) {\Huge $\vdots$};

        \node[draw=white] (f41) at (0.5, 11.5) { $\bz(\bx)$};

        \node[draw=white] (f42) at (5.5, 11.5) { $\big(\mathbb{I}(j(\bx)=j)\big)_{j \in \{(R-1)\cdot2^{md} + 1,\ldots, R \cdot 2^{md}\}}$};

        \node[draw=white] (f43) at (12, 11.5) { \color{red} $\left(\sum_{j=1}^{(R-1)\cdot 2^{md}} \lambda_{j,k} \mathbb{I}(j(\bx)=j)\right)_{k \in [K]} $};

        \node[draw=white] (f51) at (0.5, 14) { $\bz(\bx)$};

        \node[draw=white] (f53) at (12, 14) { \color{red} $\left(\lambda_{j(\bx),k}\right)_{k \in [K]} $};

        \draw[-, cyan, line width=0.2mm] (x.north) -- (f11.south); 
        \draw[-, cyan, line width=0.2mm] (x.north) -- (f12.south) node[below] {(\ref{tmp_40})};

        \draw[-, cyan, line width=0.2mm] (f11.north) -- (f21.south);
        \draw[-, cyan, line width=0.2mm] (f11.north) -- (f22.south) node[below] {(\ref{tmp_40})};
        \draw[-, cyan, line width=0.2mm] (f12.north) -- (f23.south);

        \draw[-, cyan, line width=0.2mm] (f21.north) -- (f31.south);
        \draw[-, cyan, line width=0.2mm] (f21.north) -- (f32.south) node[below] {(\ref{tmp_40})};
        \draw[-, cyan, line width=0.2mm] (f22.north) -- (f33.south);
        \draw[-, cyan, line width=0.2mm] (f23.north) -- (f33.south) node[midway, left] {(\ref{tmp_24})};

        \draw[-, cyan, line width=0.2mm] (f41.north) -- (f51.south);
        \draw[-, cyan, line width=0.2mm] (f42.north) -- (f53.south);
        \draw[-, cyan, line width=0.2mm] (f43.north) -- (f53.south) node[midway, left] {(\ref{tmp_24})};



    \end{tikzpicture}
    \vspace{15pt}
    \caption{\textbf{Illustration of Step 1 of the proof of Lemma \ref{lemma_piece_bin_lin_ver3}.} Black neurons use the Heaviside activation function, while red neurons use linear activation functions.
    } \label{fig_lemma_bin_step1}
\end{figure}

\noindent
For each $j \in [J]$, there exists $\bm{b}_j = (b_{j, i, \ell})_{i \in [d], \ell \in [m+t]} \in  \{0,1\}^{(m+t)d}$
such that $j(\bx)=j$ if and only if  
$(b_\ell(x_i))_{i \in [d], \ell \in [m+t]} = \bm{b}_j$.
Hence,
\begin{align}
\mathbb{I}(j(\bx)=j) &= \prod_{i=1}^d \prod_{\ell=1}^{m+t} \mathbb{I}\left( b_\ell(x_i) = b_{j, i, \ell} \right) \nonumber \\
&= \prod_{i=1}^d \prod_{\ell=1}^{m+t} \mathbb{I}\bigg( \big(2b_\ell(x_i)-1\big)\big(2b_{j, i, \ell}-1\big) =1 \bigg) \nonumber \\
&= 
\mathbb{I}\left( \sum_{i=1}^d \sum_{\ell=1}^{m+t} \big(2b_\ell(x_i)-1\big)\big(2b_{j, i, \ell}-1\big) - \left(d(m+t)-\frac{1}{2}\right) \right) \label{tmp_40}
\end{align}
is a linear threshold function of $(b_\ell(\bx))_{\ell \in [m+t]}$.

To make the input $\bz(\bx)$ available up to the $(R+1)$-st hidden layer of $h$, we use the identity $b = \mathbb{I}(b - 1/2)$ for $b \in \{0,1\}$, resulting in $d(m+n+2t)$ Heaviside neurons per layer.
In addition, for each $\ell \in [R] = [2^{td}]$, given that the $(\ell-1)$-st hidden layer of $h$ saves the values of $(b_{\ell}(\bx))_{\ell \in [m+t]}$, we construct 
$(\mathbb{I}(j(\bx)=j))_{j \in \{(\ell-1)\cdot2^{md}+1,\ldots,\ell \cdot 2^{md}\}}$ in the $\ell$-th hidden layer using (\ref{tmp_40}), resulting in $2^{md}$ Heaviside neurons per layer.

Given that the first hidden layer of $h$ saves the values of $(\mathbb{I}(j(\bx)=j))_{j \in \{1,\ldots,2^{md}\}}$, we construct
$(\sum_{j\in [2^{md}]} \lambda_{j,k} \mathbb{I}(j(\bx)=j))_{k \in [K]}$ in the second hidden layer using $K$ linear neurons.
Assume that the $\ell$-th hidden layer with $2 \leq \ell \leq R$ contains a neuron with the value of $\sum_{j\in [(\ell-1) \cdot 2^{md}]} \lambda_{j,k} \mathbb{I}(j(\bx)=j)$, for every $k \in [K]$.
Given that the $\ell$-th hidden layer also contains the values of $(\mathbb{I}(j(\bx)=j))_{j \in \{(\ell-1)\cdot2^{md}+1,\ldots,\ell \cdot 2^{md}\}}$,
we construct
\begin{equation}
    \begin{aligned}
    \sum_{j=1}^{\ell \cdot 2^{md}} \lambda_{j,k} \mathbb{I}(j(\bx)=j)
    = \sum_{j=1}^{(\ell-1)\cdot 2^{md}} \lambda_{j,k} \mathbb{I}(j(\bx)=j)
    + \sum_{j=(\ell-1)\cdot 2^{md} + 1}^{\ell \cdot 2^{md}} \lambda_{j,k} \mathbb{I}(j(\bx)=j).
    \label{tmp_24}
    \end{aligned}
\end{equation}
in the $(\ell+1)$-st hidden layer, for every $k \in [K]$.
This results in $K$ linear neurons per layer.
Iterating this process, the $\ell$-th hidden layer for every $2 \leq \ell \leq R+1$ contains the values of $\sum_{j\in [(\ell-1) \cdot 2^{md}]} \lambda_{j,k} \mathbb{I}(j(\bx)=j)$. 
For $\ell=R+1,$ the overall output is 
$\bz(\bx)$
and
$$\left(\sum_{j=1}^{R \cdot 2^{md}} \lambda_{j,k} \mathbb{I}(j(\bx)=j)\right)_{k \in [K]} = \left(\lambda_{j(\bx),k}\right)_{k \in [K]}.$$
By construction, each of the first to the $(R+1)$-st hidden layers of $h$ has at most $d(m+n+2t) + 2^{md} $
Heaviside neurons and $K$ linear neurons.
See Fig. \ref{fig_lemma_bin_step1} for an illustration of Step 1. 

\medskip \medskip 

\noindent
\textbf{Step 2}: 
Representing $(\mathbb{I}(k(\bx)=k, r(\bx)=\ell))_{k \in [K]}$ and $(\eta_{j(\bx),k,\ell})_{k \in [K]}$ for each $\ell \in [R]$, in the $(R+\ell+1)$-st hidden layer of $h(\bz(\bx))$.

\medskip

\begin{figure}[t]
    \centering
    \tikzset{every node/.style={scale=0.9}}
\begin{tikzpicture}[node distance=2cm, auto, thick, >=Stealth, scale=0.8]

\node[draw=white] (x) at (8, -4) { $\bz(\bx) := \big(b_{\ell}(\bx)\big)_{\ell \in [m+n+2t]}$};
;

       \node[draw=black, line width=0.1mm, draw=black, minimum height=0.6cm, minimum width=6cm] (g1) at (8, -2.5) { (First hidden layer) };

       \node at (8, -0.8) {\Huge $\vdots$};

       \node[draw=black, line width=0.1mm, draw=black, minimum height=0.6cm, minimum width=6cm] (g2) at (8, 0.5) { ($R$-th hidden layer) };

       \node[draw=black, line width=0.1mm, minimum height=5cm, minimum width=7.5cm] (g) at (8, -0.25) { };
        \node[left] at (g.west) 
        {\textbf{(Step 1)}};

        \node[draw=white] (f11) at (6, 2) { $\bz(\bx)$};

        \node[draw=white] (f13) at (10.5, 2) { \color{red} $\big(\lambda_{j(\bx),k} \big)_{k \in [K]} $};

        \node[draw=white] (f21) at (1, 4) { $\bz(\bx)$};

        \node[draw=white] (f22) at (5.5, 4) { $\big(\mathbb{I}(k(\bx)=k, r(\bx)=1)\big)_{k \in [K]}$};

        \node[draw=white] (f23) at (10.5, 4) {  $\big(\eta_{j(\bx),k,1} \big)_{k \in [K]} $};

        \node[draw=red, minimum height=0.6cm] (f24) at (13.5, 4) { \color{red}  K neurons};

        \node[draw=white] (f31) at (1, 6) { $\bz(\bx)$};

        \node[draw=white] (f32) at (5.5, 6) { $\big(\mathbb{I}(k(\bx)=k, r(\bx)=2)\big)_{k \in [K]}$};

        \node[draw=white] (f33) at (10.5, 6) {  $\big(\eta_{j(\bx),k,2} \big)_{k \in [K]} $};

        \node[draw=red, minimum height=0.6cm] (f34) at (13.5, 6) { \color{red}  K neurons};

        \node[draw=white] (f41) at (1, 8) { $\bz(\bx)$};

        \node[draw=white] (f42) at (5.5, 8) { $\big(\mathbb{I}(k(\bx)=k, r(\bx)=3)\big)_{k \in [K]}$};

        \node[draw=white] (f43) at (10.5, 8) {  $\big(\eta_{j(\bx),k,3} \big)_{k \in [K]} $};

        \node[draw=red, minimum height=0.6cm] (f44) at (13.5, 8) { \color{red}  K neurons};

        \node at (4.5, 10) {\Huge $\vdots$};

        \node at (12, 10) {\Huge $\vdots$};

        \node[draw=white] (f51) at (1, 11.5) { $\bz(\bx)$ };

        \node[draw=white] (f52) at (5.5, 11.5) { $\big(\mathbb{I}(k(\bx)=k, r(\bx)=R)\big)_{k \in [K]}$};

        \node[draw=white] (f53) at (10.5, 11.5) {  $\big(\eta_{j(\bx),k,R} \big)_{k \in [K]} $};

        \draw[-, cyan, line width=0.2mm] (x.north) -- (g.south); 
        \draw[-, cyan, line width=0.2mm] (g2.north) -- (f11.south); 
        \draw[-, cyan, line width=0.2mm] (g2.north) -- (f13.south); 

        \draw[-, cyan, line width=0.2mm] (f11.north) -- (f21.south); 
        \draw[-, cyan, line width=0.2mm] (f11.north) -- (f22.south) node[right] {(\ref{tmp_140})};
        \draw[-, cyan, line width=0.2mm] (f13.north) -- (f23.south);
        \draw[-, cyan, line width=0.2mm] (f13.north) -- (f24.south);

        \draw[-, cyan, line width=0.2mm] (f21.north) -- (f31.south); 
        \draw[-, cyan, line width=0.2mm] (f21.north) -- (f32.south) node[below] {(\ref{tmp_140})};
        \draw[-, cyan, line width=0.2mm] (f23.north) -- (f33.south);
        \draw[-, cyan, line width=0.2mm] (f23.north) -- (f34.south);
        \draw[-, cyan, line width=0.2mm] (f24.north) -- (f33.south);
        \draw[-, cyan, line width=0.2mm] (f24.north) -- (f34.south);
        \draw[-, cyan, line width=0.2mm];

        \draw[-, cyan, line width=0.2mm] (f31.north) -- (f41.south); 
        \draw[-, cyan, line width=0.2mm] (f31.north) -- (f42.south) node[below] {(\ref{tmp_140})};
        \draw[-, cyan, line width=0.2mm] (f33.north) -- (f43.south);
        \draw[-, cyan, line width=0.2mm] (f33.north) -- (f44.south);
        \draw[-, cyan, line width=0.2mm] (f34.north) -- (f43.south);
        \draw[-, cyan, line width=0.2mm] (f34.north) -- (f44.south);
        \draw[-, cyan, line width=0.2mm]; 

       \node[draw=black, line width=0.1mm, minimum height=8cm, minimum width=5cm] (g) at (12, 7.75) { };
        \node[above] at (g.north) {(Lemma \ref{lemma_bit_extract_linear_ver2})$\times K$ times  };

    \end{tikzpicture}
    \vspace{15pt}
    \caption{\textbf{Illustration of Step 2 of the proof of Lemma \ref{lemma_piece_bin_lin_ver3}.} Black neurons use the Heaviside activation function, while red neurons use linear activation functions.
    } \label{fig_lemma_bin_step2}
\end{figure}

\noindent
For each $k \in [K]$ and $r \in [R]$, there exists $\bm{b}_j = (b_{j, i, \ell})_{i \in [d], \ell \in [n+t]} \in  \{0,1\}^{(n+t)d}$
such that $k(\bx)=k$ and $r(\bx)=r$ if and only if  
$(b_\ell(x_i))_{i \in [d], \ell \in \{m+t+1,\ldots, m+n+2t\}} = \bm{b}_j$.
Hence, similarly as in \eqref{tmp_40},
\begin{align}
\mathbb{I}\big(k(\bx)=k, r(\bx)=r\big) &= \prod_{i=1}^d \prod_{\ell=1}^{n+t} \mathbb{I}\left( b_{m+t+\ell}(x_i) = b_{j, i, \ell} \right) \nonumber \\
&= \prod_{i=1}^d \prod_{\ell=1}^{n+t} \mathbb{I}\bigg( \big(2b_{m+t+\ell} (x_i)-1\big)\big(2b_{j, i, \ell}-1\big) =1 \bigg) \nonumber \\
&= 
\mathbb{I}\left( \sum_{i=1}^d \sum_{\ell=1}^{n+t} \big(2 b_{m+t+\ell}(x_i)-1\big)\big(2b_{j, i, \ell}-1\big) - \left(d(n+t)-\frac{1}{2}\right) \right) \label{tmp_140}
\end{align}
is a linear threshold function of $(b_\ell(\bx))_{\ell \in \{m+t+1,\ldots, m+n+2t\}}$.

The network $h$ has been already constructed up to the $(2m+n)$-th hidden layer, such that the output of the $(R+1)$-st hidden layer of $h(\bz(\bx))$ is 
$\bz(\bx)$
and
$\left(\lambda_{j(\bx),k}\right)_{k \in [K]}.$
To save $\bz(\bx)$ up to the $(2R+1)$-st hidden layer of $h$, we use the identity $b = \mathbb{I}(b - 1/2)$ for $b \in \{0,1\}$, resulting in $d(m+n+2t)$ Heaviside neurons per layer.
In addition, for each $\ell \in [R]$, given that the $(R+\ell)$-th hidden layer of $h$ saves the values of $(b_\ell(\bx))_{\ell \in \{m+t+1,\ldots, m+n+2t\}}$, we construct 
$(\mathbb{I}(k(\bx)=k, r(\bx)=\ell))_{k \in [K]}$ in the $(R+\ell+1)$-st hidden layer using (\ref{tmp_140}), resulting in $K$ Heaviside neurons in each layer.

Recall that $\lambda_{j(\bx),k} := {0.\eta_{j(\bx),k,1} \eta_{j(\bx),k,2} \ldots \eta_{j(\bx),k,R}}_{(2)}$ for each $k \in [K]$.
We apply Lemma \ref{lemma_bit_extract_linear_ver2} with $L:=R$ and obtain a DHN $f \in \DHN_{\lin}(R, 1:1:1, 1)$ such that the $\ell$-th hidden layer (with $\ell \in [R]$) of $f(\lambda_{j(\bx),k})$ contains a neuron with the value of $\eta_{j(\bx),k,\ell}$.
For each $k \in [K]$, we construct $f(\lambda_{j(\bx),k})$ 
using $\lambda_{j(\bx),k}$ in the $(R+1)$-st hidden layer of $h$.
Then, for each $\ell \in [R]$, the $(R + \ell + 1)$-st hidden layer of $h$ contains the values of $(\eta_{j(\bx),k,\ell})_{k \in [K]}$.  
This requires $K$ Heaviside neurons and $K$ linear neurons in each layer. 
See Fig. \ref{fig_lemma_bin_step2} for an illustration of Step 2. 

\medskip \medskip 

\noindent
\textbf{Step 3}: Representing $\eta_{j(\bx), k(\bx), r(\bx)}$ in the output layer of $h(\bz(\bx))$.

\medskip 

\noindent
For each $k \in [K]$ and $r \in [R]$, we define
$$\rho_{k,r}(\bx) := \eta_{j(\bx),k,r} \mathbb{I}\big(k(\bx)=k, r(\bx)=r\big).$$
For each $r \in [R]$, given that $(R + r + 1)$-st hidden layer of $h$ contains the values of $(\mathbb{I}(k(\bx)=k, r(\bx)=r))_{k \in [K]}$ and $(\eta_{j(\bx),k,r})_{k \in [K]}$, we construct
\begin{align}
    \big( \rho_{k,r}(\bx) \big)_{k \in [K]} = \left(\mathbb{I}\left( \eta_{j(\bx),k,r} + \mathbb{I}\big(k(\bx)=k, r(\bx)=r \big) - \frac{1}{2}\right)\right)_{k \in [K]} \label{tmp_30}
\end{align}
in the $(R + r + 2)$-nd hidden layer.

\clearpage

\begin{figure}[t]
    \centering
    \tikzset{every node/.style={scale=0.9}}
\begin{tikzpicture}[node distance=2cm, auto, thick, >=Stealth, scale=0.8]

       \node[draw=black, line width=0.1mm, minimum height=14cm, minimum width=8.8cm] (g) at (7, 4) { };
        \node[above] at (g.north) 
        {\textbf{(Step 2)}};

\node[draw=white] (x) at (7, -4.7) { $\bz(\bx) := \big(b_{\ell}(\bx)\big)_{\ell \in [m+n+2t]}$};
;

       \node[draw=black, line width=0.1mm, draw=black, minimum height=0.8cm, minimum width=8.5cm] (g1) at (7, -3) {(First hidden layer) };

       \node at (7, -1.5) {\Huge $\vdots$};

       \node[draw=black, line width=0.1mm, draw=black, minimum height=0.8cm, minimum width=8.5cm] (g2) at (7, 0) {$((R+1)$-st hidden layer) };


       \node[draw=black, line width=0.1mm, draw=black, minimum height=0.8cm, minimum width=8.5cm] (g3) at (7, 2) {\textbf{...} \hspace{6.5cm}  };

        \node[draw=white] (f31) at (6.5, 2) { $\big(\mathbb{I}(k(\bx)=k, r(\bx)=1)\big)_{k \in [K]}$};

        \node[draw=white] (f32) at (10.4, 2) { $\big(\eta_{j(\bx),k,1} \big)_{k \in [K]} $};


       \node[draw=black, line width=0.1mm, draw=black, minimum height=0.8cm, minimum width=8.5cm] (g4) at (7, 4) {\textbf{...} \hspace{6.5cm}  };

        \node[draw=white] (f41) at (6.5, 4) { $\big(\mathbb{I}(k(\bx)=k, r(\bx)=2)\big)_{k \in [K]}$};

        \node[draw=white] (f42) at (10.4, 4) { $\big(\eta_{j(\bx),k,2} \big)_{k \in [K]} $};

        \node[draw=white] (f43) at (13.5, 4) { $\big(\rho_{k,1} (\bx) \big)_{k \in [K]}$};


       \node[draw=black, line width=0.1mm, draw=black, minimum height=0.8cm, minimum width=8.5cm] (g5) at (7, 6) {\textbf{...} \hspace{6.5cm}  };

        \node[draw=white] (f51) at (6.5, 6) { $\big(\mathbb{I}(k(\bx)=k, r(\bx)=3)\big)_{k \in [K]}$};

        \node[draw=white] (f52) at (10.4, 6) { $\big(\eta_{j(\bx),k,3} \big)_{k \in [K]} $};

        \node[draw=white] (f53) at (13.5, 6) { $\big(\rho_{k,2} (\bx) \big)_{k \in [K]}$};        

        \node[draw=white] (f54) at (16.8, 6) { $\sum_{k=1}^K \rho_{k,1} (\bx)$};


       \node[draw=black, line width=0.1mm, draw=black, minimum height=0.8cm, minimum width=8.5cm] (g6) at (7, 8) {\textbf{...} \hspace{6.5cm}  };

        \node[draw=white] (f61) at (6.5, 8) { $\big(\mathbb{I}(k(\bx)=k, r(\bx)=4)\big)_{k \in [K]}$};

        \node[draw=white] (f62) at (10.4, 8) { $\big(\eta_{j(\bx),k,4} \big)_{k \in [K]} $};

        \node[draw=white] (f63) at (13.5, 8) { $\big(\rho_{k,3} (\bx) \big)_{k \in [K]}$};        

        \node[draw=white] (f64) at (16.8, 8) { $\sum_{r=1}^2\sum_{k=1}^K \rho_{k,r} (\bx)$};


        \node at (7, 9.5) {\Huge $\vdots$};

        \node at (16, 9.5) {\Huge $\vdots$};


       \node[draw=black, line width=0.1mm, draw=black, minimum height=0.8cm, minimum width=8.5cm] (g7) at (7, 11) {\textbf{...} \hspace{6.5cm}  };

        \node[draw=white] (f71) at (6.5, 11) { $\big(\mathbb{I}(k(\bx)=k, r(\bx)=R)\big)_{k \in [K]}$};

        \node[draw=white] (f72) at (10.4, 11) { $\big(\eta_{j(\bx),k,R} \big)_{k \in [K]} $};

        \node[draw=white] (f73) at (13.5, 11) { $\big(\rho_{k,R-1}(\bx) \big)_{k \in [K]}$};        

        \node[draw=white] (f74) at (16.8, 11) { $\sum_{r=1}^{R-2}\sum_{k=1}^K \rho_{k,r}(\bx)$};


        \node[draw=white] (f83) at (13.5, 13) { $\big(\rho_{k,R} (\bx) \big)_{k \in [K]}$};        

        \node[draw=white] (f84) at (16.8, 13) { $\sum_{r=1}^{R-1}\sum_{k=1}^K \rho_{k,r} (\bx)$};
        

        \node[draw=white] (f94) at (16.8, 15) { $h(\bz(\bx)) = \eta_{j(\bx),k(\bx),r(\bx)}$};


        \draw[-, cyan, line width=0.2mm] (x.north) -- (g.south) ;

        %

        \draw[-, cyan, line width=0.2mm] (f31.north) -- (f43.south) ;
        \draw[-, cyan, line width=0.2mm] (f32.north) -- (f43.south)node[below] {(\ref{tmp_30})};

        %
        
        \draw[-, cyan, line width=0.2mm] (f41.north) -- (f53.south) ;
        \draw[-, cyan, line width=0.2mm] (f42.north) -- (f53.south)node[below] {(\ref{tmp_30})};
        \draw[-, cyan, line width=0.2mm] (f43.north) -- (f54.south)node[below] {(\ref{tmp_31})};

        %
        
        \draw[-, cyan, line width=0.2mm] (f51.north) -- (f63.south) ;
        \draw[-, cyan, line width=0.2mm] (f52.north) -- (f63.south)node[below] {(\ref{tmp_30})};
        \draw[-, cyan, line width=0.2mm] (f53.north) -- (f64.south);
        \draw[-, cyan, line width=0.2mm] (f54.north) -- (f64.south)node[near end, right] {(\ref{tmp_25})};

        \draw[-, cyan, line width=0.2mm] (f71.north) -- (f83.south) ;
        \draw[-, cyan, line width=0.2mm] (f72.north) -- (f83.south)node[below] {(\ref{tmp_30})};
        \draw[-, cyan, line width=0.2mm] (f73.north) -- (f84.south);
        \draw[-, cyan, line width=0.2mm] (f74.north) -- (f84.south)node[near end, right] {(\ref{tmp_25})};

        \draw[-, cyan, line width=0.2mm] (f83.north) -- (f94.south);
        \draw[-, cyan, line width=0.2mm] (f84.north) -- (f94.south)node[near end, right] {(\ref{tmp_32})};
    \end{tikzpicture}
    \vspace{15pt}
    \caption{\textbf{Illustration of Step 3 of the proof of Lemma \ref{lemma_piece_bin_lin_ver3}.}
    } \label{fig_lemma_bin_step3}
\end{figure}
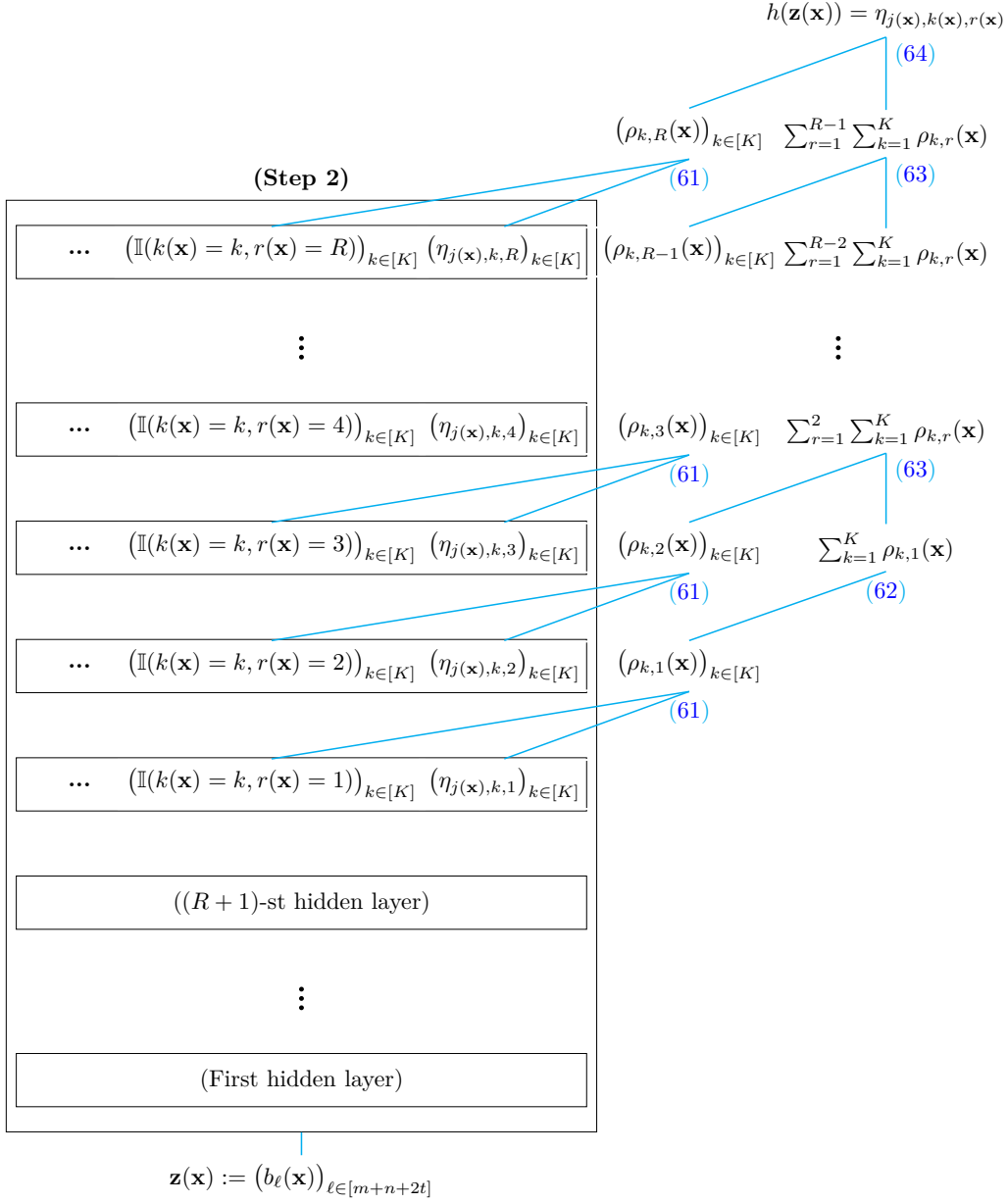

Given that the $(R+3)$-rd hidden layer saves the values of $(\rho_{k,1}(\bx))_{k \in [K]}$, we construct 
\begin{align}
    \sum_{k=1}^K \rho_{k,1} = \mathbb{I}\left(\sum_{k=1}^K \rho_{k,1} - \frac{1}{2}\right) \label{tmp_31}
\end{align}
in the $(R+4)$-th hidden layer.
Assume that the $(R+\ell+3)$-rd hidden layer with $1 \leq \ell \leq R-2$ contains the value of $\sum_{r=1}^\ell \sum_{k=1}^K \rho_{k,r}$.
Given that the $(R+\ell+3)$-rd hidden layer also contains the values of $\big( \rho_{k,\ell+1}(\bx) \big)_{k \in [K]}$, we construct
\begin{align}
    \sum_{r=1}^{\ell+1} \sum_{k=1}^K \rho_{k,r} = \mathbb{I}\left(\sum_{r=1}^{\ell} \sum_{k=1}^K \rho_{k,r} + \sum_{k=1}^K \rho_{k,\ell+1} - \frac{1}{2}\right) \label{tmp_25}
\end{align}
in the $(R+\ell+4)$-th hidden layer.
Iterating this process, the $(R+\ell+3)$-rd hidden layer for every $1 \leq \ell \leq R-1$ contains the values of $\sum_{r=1}^\ell \sum_{k=1}^K \rho_{k,r}$. 

Given that the $(2R+2)$-nd hidden layer contains the values of $\big( \rho_{k,R}(\bx) \big)_{k \in [K]}$ and $\sum_{r=1}^{R-1} \sum_{k=1}^K \rho_{k,r}$, we construct the final output as
\begin{align}
    h\big((b_{\ell}(\bx))_{\ell \in [m+n+2t]}\big) 
    &= \sum_{r=1}^{R-1} \sum_{k=1}^K \rho_{k,r} + \sum_{k=1}^K \rho_{k,R} \notag \\
    & = \sum_{r=1}^R \sum_{k=1}^K \eta_{j(\bx),k,r} \mathbb{I}\big(k(\bx)=k, r(\bx)=r\big)  \notag \\
    & = \eta_{j(\bx),k(\bx),r(\bx)}. \label{tmp_32}
\end{align}
See Fig. \ref{fig_lemma_bin_step3} for an illustration of Step 3.
Given that the network $h$ has $2R+2$ hidden layers and each hidden layer has 
$\leq d(m+n+2t) + (2^{md} \vee (3K+1))$ Heaviside activated neurons and $\leq K$ linearly activated neurons, we obtain the assertion.

\end{proof}

\begin{proof}[Proof of the lower bound in Theorem \ref{VCdim-linear}] 
It is enough to show the result for $d = 1$.
From (\ref{skip_lin_include}), we have
$$\DHN_{\lin}(L,1:p:1, s) \supseteq \DHN_{\skipp}(L,1:p:1, 1).$$
By Theorem \ref{thm_skip_vc_upper}, we obtain
\begin{align}
    \operatorname{VC}\big(\DHN_{\lin}(L,1:p:1, s)\big) \geq 
\operatorname{VC}\big(\DHN_{\skipp}(L,1:p:1, 1)\big) \geq
c' \cdot Lp^2, \label{tmp_34}
\end{align}
where $c'$ is the universal constant defined on Theorem \ref{thm_skip_vc_upper}.  

Now we consider the case $L^2 ps > Lp^2$.
We introduce integers $m$, $n$ and $t$ such that $m \lesssim \log p$, $n \lesssim \log s$ and $t \lesssim \log L$. Their exact values will be specified in (\ref{tmp_42}).
For $i \in [2^{m+n+2t}]$, we define 
$$z_i := \frac{i}{2^{m+n+2t}} - \frac{1}{2^{m+n+2t+1}}.$$
    We now show that for arbitrary function
    $$\lambda : \{z_1, z_2, \ldots, z_{2^{m+n+2t}} \} \mapsto \{0,1\},$$
    there exists $f \in \DHN_{\lin}(L,1:p:1,s)$ such that $f(z_i) = \lambda(z_i)$ for every $i \in [2^{m+n+2t}]$. Thus, $\operatorname{VC}(\DHN_{\lin}(L,1:p:1,s)) \geq 2^{m+n+2t}$.

    First, we construct a Heaviside network with one linear neuron, that extracts the first $m+n+2t$ bits of $x$, that is, $(b_\ell(x))_{\ell \in [m+n+2t]}$.   
    This is achieved by applying Lemma \ref{lemma_bit_extract_linear_ver1} with $L := m+n+2t$.
    For $\ell \in [m+n+2t],$ the $\ell$-th hidden layer of $f$ has $\ell$ Heaviside neurons and one linear neuron.
    The $(m+n+2t)$-th hidden layer outputs $(b_{\ell}(x))_{\ell \in [m+n+2t]}$. This network is used to construct the network $f$ up to the $(m+n+2t)$-th layer. 

    We now describe the construction of the remaining $(2R+2)$-hidden layers of the network $f.$
    The main step is the construction of a subnetwork $h.$
    As in \eqref{eq.0fweni_2} and the subsequent text (with $d=1$), define $J:=2^{m+t}$, $K:=2^{n}$, $R:=2^{t}$, along with the partitions
    $\{U_j\}_{j \in [J]}$, $\{V_k\}_{k \in [K]}$, and $\{W_r\}_{r \in [R]}$  of $[0,1]$, and the index functions $j(x) \in [J]$, $k(x) \in [K]$, and $r(x) \in [R]$.
    By definition, there exists a bijective function 
    $\bm{\pi} = (\pi_1,\pi_2,\pi_3) : [2^{m+n+2t}] \mapsto [J] \times [K] \times [R]$
    such that 
    $\pi_1(i) = j(z_i)$, $\pi_2(i) = k(z_i)$ and $\pi_3(i) = r(z_i)$ for every $i \in [2^{m+n+2t}]$. 
    In fact, $z_{i}$ is the center of the interval $U_{\pi_1(i)} \cap V_{\pi_2(i)} \cap W_{\pi_3(i)}$.

    By applying Lemma \ref{lemma_piece_bin_lin_ver3} with $\eta_{j,k,r} := \lambda(z_{\bm{\pi}^{-1}(j,k,r)})$, we obtain a DHN with $K$ linear neurons in each layer
    $$h \in \DHN_{\lin}\Big(2R+2,\ m+n+2t:m+n+2t + \big(2^{m} \vee (3K+1) \big):1, \ K \Big)$$
    such that 
    $$\mathbb{I} \Big( h \big((b_{\ell}(x))_{\ell \in [m+n+2t]} \big) \Big) = \eta_{j(x),k(x),r(x)} = \lambda(z_{\bm{\pi}^{-1}(j(x),k(x),r(x))}).$$
    Replacing $x$ by $z_i$, we obtain
    \begin{align*}
        \mathbb{I} \Big( h \big( (b_{\ell}(z_i))_{\ell \in [m+n+2t]} \big) \Big) = \lambda(z_{\bm{\pi}^{-1}(\pi_1(i),\pi_2(i),\pi_3(i))}) = \lambda(z_i). 
    \end{align*} 

    The network $f$ has been already constructed up to the $(m+n+2t)$-th hidden layer outputting the binary digits $(b_{\ell}(x))_{\ell \in [m+n+2t]}$.
    Stacking the network $h$ on top results in the network 
$$f \in \DHN_{\lin}\Big(m+n+2t+2R+2, \ 1:(m+n+2t + \big(2^{m} \vee (3K+1) \big):1, \ K\Big)$$
        computing $f(z_i) = \lambda(z_i)$ for every $i \in [2^{m+n+2t}]$.
    Recall that $R:=2^{t}$ and $K:=2^{n}$. 
    We set
    \begin{align}
        t := \bigg\lfloor \log\left(\frac{L}{3}\right)\bigg\rfloor, 
    \quad m := \bigg\lfloor \log\left(  \frac{p}{2} \right) \bigg\rfloor, 
    \quad \text{and} \ \ n := \left\lfloor \log \left( s \land \frac{p}{5} \right) \right\rfloor. \label{tmp_42}
    \end{align}
    Given that 
    $$m+n+2t \leq 2 \log (Lp) \leq \frac{L \land p}4,$$
    we obtain
$$m+n+2t+2R+2 \leq \frac{L}{4} + \frac{2L}{3}+2 \leq L,$$
$$m+n+2t + \big(2^{m} \vee (3K+1) \big) \leq \frac{p}{4} + \left( \frac{p}{2} \vee \left(\frac{3p}{5}+1\right) \right) \leq p$$
for sufficiently large $L$ and $p$ (e.g., if $\tfrac{11}{12}L+2 \leq L$ and $\tfrac{17}{20}p + 1 \leq p$, implying that the constant $c$ in the statement of the theorem can be taken to be $24$).
    We also have $K \leq s$, and hence    
    $f \in \DHN_{\lin}(L,1:p:1,s)$.

    Given the inequalities
    $$t \geq \log\left(\frac{L}{3}\right)-1, \quad m \geq \log\left(  \frac{p}{2} \right)-1,
    \quad \text{and} \ \ n \geq \log\left(  \frac{s}{5} \right)-1,$$ 
    we obtain the lower bound
    \begin{align}
        \operatorname{VC}\big(\DHN_{\lin}(L,1:p:1,1)\big) \geq 2^{m+n+2t} \geq \frac{1}{2^4 \cdot 3^2 \cdot 2 \cdot 5} L^2 ps. \label{tmp_35}
    \end{align}
    By (\ref{tmp_34}) and (\ref{tmp_35}), we obtain the assertion.
    
\end{proof}

\subsection{Proof of Theorem \ref{upper_smooth_lin}}
\begin{proof} 
By (\ref{skip_lin_include}), we have
$$\DHN_{\lin}(L,1:p:1, s) \supseteq \DHN_{\skipp}(L,1:p:1, d).$$
Hence, by Theorem \ref{upper_smooth_skip}, we obtain
\begin{align}
    \sup_{f_0 \in \mathcal{H}_{d}^{\beta}(M)} \, \inf_{f \in \DHN_{\lin}(L,d:p:1,s)} \left\| f - f_0 \right\|_{\infty}
    &\leq 
    \sup_{f_0 \in \mathcal{H}_{d}^{\beta}(M)} \, \inf_{f \in \DHN_{\skipp}(L,d:p:1,d)} \left\| f - f_0 \right\|_{\infty} \nonumber \\ 
    & \leq c' M \cdot  \left(\frac{\log^3 (Lp)}{Lp^2 }\right)^{\frac{\beta}{d}}, \label{tmp_200}
\end{align}
with $c'$ the constant defined in Theorem \ref{upper_smooth_skip}, which depends only on $\beta$.  

Now we consider the case $L^2 ps > Lp^2$.
The proof strategy is similar to that of Theorem \ref{upper_smooth_skip}, replacing Lemma \ref{lemma_piece_bin_skip_ver3} with Lemma \ref{lemma_piece_bin_lin_ver3} and thereby changing the rate.
We introduce the integers $m$, $n$ and $t$ satisfying
$m \lesssim \log p$, $n \lesssim \log s$ and $t \lesssim \log L$. 
Their exact values will be specified in (\ref{tmp_112}).  
In addition, we define $q = \lceil \beta \rceil(m+n+2t)$.

Any real number $x \in [0,1]$ can be expressed in a binary representation as
\begin{align*}
    x=0.b_1(x)b_2(x)b_3(x)\dots_{(2)}.
\end{align*}
As in \eqref{form1}, we introduce the finite-bit representation \[\widetilde{x}=0.b_1(x)b_2(x)b_3(x)\dots b_{m+n+2t}(x)_{(2)} + 2^{-m-n-2t-1}.\]

For a vector $\bx = (x_1, \dots, x_d) \in [0,1]^d$, we define $b_{\ell}(\bx) := (b_{\ell}(x_i))_{i \in [d]}$.
In addition, we define the truncated approximation of $\bx$ as $\widetilde{\bx} = (\widetilde{x}_1, \dots, \widetilde{x}_d) \in [0,1]^d$ where each $\widetilde{x}_i$ is defined as in \eqref{form1}. Therefore,
$$
    \|\bx - \widetilde{\bx}\|_{\infty} \leq 2^{-m-n-2t-1}.
$$

Our goal is to construct $f \in \DHN_{\lin}(L,d:p:1,s)$ satisfying
\begin{align} 
    \sup_{\bx \in [0,1]^d} \Bigg|f(\bx) - \underbrace{\sum_{\bm{\alpha} \in \mathcal{A}(d,\beta)} (\partial^{\bm{\alpha}}f_0)(\widetilde{\bx}) \cdot \frac{(\bx-\widetilde{\bx})^{\bm{\alpha}}}{\bm{\alpha}!}}_{:= \widetilde{f}_0(\bx)} \Bigg| \lesssim \left(\frac{\log^3 (Lp)}{L^2 ps}\right)^{\frac{\beta}{d}},
    \label{eq_skip_taylor_2}
\end{align}
with
\[\mathcal{A}(d,\beta) :=  \{ \bm{\alpha} \in \mathbb{N}^d  :\|\bm{\alpha}\|_1 < \beta \}.\]
Combining Lemma \ref{lemma_taylor} and (\ref{tmp_112}) implies
\[
    \sup_{\bx \in [0,1]^d} \left| \widetilde{f}_0(\bx) -f_0(\bx) \right| 
    \leq  \left\| f_0 \right\|_{\mathcal{C}^{\beta}} \cdot \|\bx - \widetilde{\bx} \|_{\infty}^\beta \lesssim \left(\frac{\log^3 (Lp)}{L^2 ps }\right)^{\frac{\beta}{d}},
\]
and the assertion follows.
The proof is split into the following three steps:
\medskip 

\noindent
\textit{Step 1}: Extracting the first $q$ bits of $\bx$ and $(\partial^{\bm{\alpha}}f_0)(\widetilde{\bx})$ for each $\bm{\alpha} \in \mathcal{A}(d,\beta)$.

\noindent
\textit{Step 2}: Approximating $\widetilde{f}_0(\bx)$.

\noindent
\textit{Step 3}: Verifying the size and error of the final network.

\medskip 
\medskip 

\noindent
\textbf{Step 1}: Extracting the first $q$ bits of $\bx$ and $(\partial^{\bm{\alpha}}f_0)(\widetilde{\bx})$ for each $\bm{\alpha} \in \mathcal{A}(d,\beta)$.

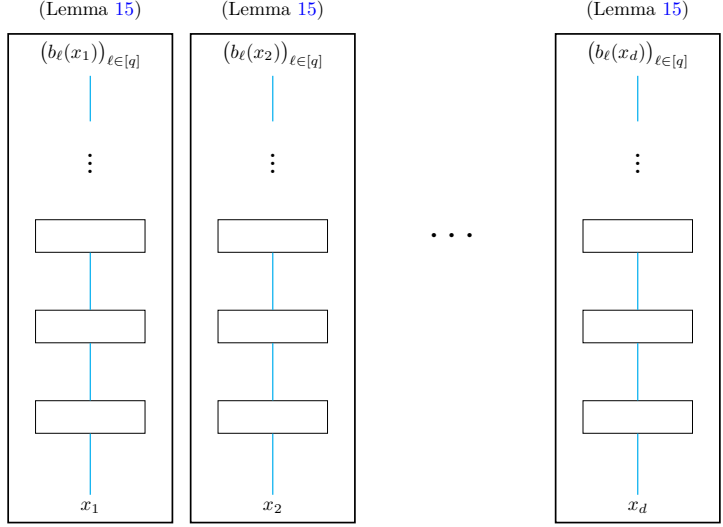
\begin{figure}[t]
    \centering
    \scalebox{0.8}{
    \tikzset{every node/.style={scale=0.9}}
    \begin{tikzpicture}[node distance=2cm, auto, thick, >=Stealth]
        \node[draw=white] (x1) at (0, 0) { $x_1$};        

        \node[line width=0.1mm, draw=black, minimum height=0.6cm, minimum width=2cm] (f11) at (0, 1.5) {};
        
        \node[line width=0.1mm, draw=black, minimum height=0.6cm, minimum width=2cm] (f21) at (0, 3) { };

        \node[line width=0.1mm, draw=black, minimum height=0.6cm, minimum width=2cm] (f31) at (0,4.5) { };

        \node at (0, 5.8) {\huge $\vdots$};

        \node[draw=white] (f41) at (0, 7.5) { $\big(b_{\ell}(x_1)\big)_{\ell \in [q]}$};

    \draw[-, cyan, line width=0.2mm] (x1.north) -- (f11.south);
       \draw[-, cyan, line width=0.2mm] (f11.north) -- (f21.south);
       \draw[-, cyan, line width=0.2mm] (f21.north) -- (f31.south);

       \draw[-, cyan, line width=0.2mm] (0,6.4) -- (f41.south);


        \node[draw=white] (x2) at (3, 0) { $x_2$};        

        \node[line width=0.1mm, draw=black, minimum height=0.6cm, minimum width=2cm] (f12) at (3, 1.5) { };
        
        \node[line width=0.1mm, draw=black, minimum height=0.6cm, minimum width=2cm] (f22) at (3, 3) { };

        \node[line width=0.1mm, draw=black, minimum height=0.6cm, minimum width=2cm] (f32) at (3,4.5) { };

        \node at (3, 5.8) {\huge $\vdots$};

        \node[draw=white] (f42) at (3, 7.5) { $\big(b_{\ell}(x_2)\big)_{\ell \in [q]}$};

    \draw[-, cyan, line width=0.2mm] (x2.north) -- (f12.south);
       \draw[-, cyan, line width=0.2mm] (f12.north) -- (f22.south);
       \draw[-, cyan, line width=0.2mm] (f22.north) -- (f32.south);

       \draw[-, cyan, line width=0.2mm] (3,6.4) -- (f42.south);

    \node at (6, 4.5) {\huge $\cdots$};


        \node[draw=white] (x3) at (9, 0) { $x_d$};        

        \node[line width=0.1mm, draw=black, minimum height=0.6cm, minimum width=2cm] (f13) at (9, 1.5) { };
        
        \node[line width=0.1mm, draw=black, minimum height=0.6cm, minimum width=2cm] (f23) at (9, 3) { };

        \node[line width=0.1mm, draw=black, minimum height=0.6cm, minimum width=2cm] (f33) at (9,4.5) { };

        \node at (9, 5.8) {\huge $\vdots$};

        \node[draw=white] (f43) at (9, 7.5) { $\big(b_{\ell}(x_d)\big)_{\ell \in [q]}$};

    \draw[-, cyan, line width=0.2mm] (x3.north) -- (f13.south);
       \draw[-, cyan, line width=0.2mm] (f13.north) -- (f23.south);
       \draw[-, cyan, line width=0.2mm] (f23.north) -- (f33.south);

       \draw[-, cyan, line width=0.2mm] (9,6.4) -- (f43.south);

\node[draw=black, align=center, minimum height=9cm, minimum width=3cm] (box1) at (0, 3.8) {};
\node[above=0.1cm] at (box1.north) {(Lemma \ref{lemma_bit_extract_linear_ver1})};

\node[draw=black, align=center, minimum height=9cm, minimum width=3cm] (box2) at (3, 3.8) {};
\node[above=0.1cm] at (box2.north) {(Lemma \ref{lemma_bit_extract_linear_ver1})};

\node[draw=black, align=center, minimum height=9cm, minimum width=3cm] (box3) at (9, 3.8) {};
\node[above=0.1cm] at (box3.north) {(Lemma \ref{lemma_bit_extract_linear_ver1})};

    \end{tikzpicture}
    }
    \vspace{15pt}
    \caption{\textbf{The first part of Step 1.} We represent the first $q$ bits of $\bx$.
    }
    \label{fig:lin_smooth_step1}
\end{figure}

\medskip

\noindent
We now construct a linear neurons augmented DHN $f$, which can approximate $f_0$ well. 
We start by extracting the $q$ bits of $\bx$  by applying Lemma \ref{lemma_bit_extract_linear_ver1} with
$L := q$ componentwise. 
For any  $\ell \in [q]$, the $\ell$-th hidden layer contains at most $dq$ Heaviside neurons and $d$ linear neurons.  
The $q$-th hidden layer of $f$ outputs $(b_{\ell}(\bx))_{\ell \in [q]}$.
See Fig. \ref{fig:lin_smooth_step1} for an illustration. 

We now extract the $q$ bits of $(\partial^{\bm{\alpha}}f_0)(\widetilde{\bx})$ for each $\bm{\alpha} \in \{\bm{\alpha}_1, \bm{\alpha}_2, \ldots, \bm{\alpha}_{|\mathcal{A}(d,\beta)|} \} := \mathcal{A}(d,\beta)$.
As in \eqref{eq.0fweni_2} and the subsequent text,
define $J:=2^{(m+t)d}$, $K:=2^{nd}$, $R:=2^{td}$, along with the partitions
$\{U_j\}_{j \in [J]}$, $\{V_k\}_{k \in [K]}$, and $\{W_r\}_{r \in [R]}$  of $[0,1]^d$, and the index functions $j(\bx) \in [J]$, $k(\bx) \in [K]$, and $r(\bx) \in [R]$.
We write $\bm{c}_{j,k,r}$ as the center of the hypercube $U_{j} \cap V_{k} \cap W_{r}$. Accordingly, the center of the hypercube including $\bx$ is
\begin{align*}
    \widetilde{\bx} = \bm{c}_{j(\bx),k(\bx),r(\bx)}.
\end{align*}
For each $\bm{\alpha} \in \mathcal{A}(d,\beta)$, 
$j \in [J]$, $k \in [K],$ and $r \in [R]$, 
we have
$$\frac{(\partial^{\bm{\alpha}}f_0)(\bm{c}_{j,k,r}) + \|\partial^{\bm{\alpha}}f_0\|_{\infty}}{2\|\partial^{\bm{\alpha}}f_0\|_{\infty}} \in [0,1].$$
Consequently, for suitable binary sequences $\{\eta^{(\ell)}_{\bm{\alpha},j,k,r}\}_{\ell \in \{1,2,\ldots\}} \in \{0,1\}^{\mathbb{N}}$,
$$\frac{(\partial^{\bm{\alpha}}f_0)(\bm{c}_{j,k,r}) + \|\partial^{\bm{\alpha}}f_0\|_{\infty}}{2\|\partial^{\bm{\alpha}}f_0\|_{\infty}}
= 0.\eta^{(1)}_{\bm{\alpha},j,k,r}\eta^{(2)}_{\bm{\alpha},j,k,r}\ldots \eta^{(q)}_{\bm{\alpha},j,k,r} \ldots_{(2)}.$$

For each $\bm{\alpha} \in \mathcal{A}(d,\beta)$ and $\ell \in [q]$,
we apply Lemma \ref{lemma_piece_bin_lin_ver3} with $\eta_{j,k,r} := \eta^{(\ell)}_{\bm{\alpha},j,k,r}$ and obtain a DHN
$$h^{(\bm{\alpha}, \ell)} \in \DHN_{\lin}\big(2R+2,\ d(m+n+2t):d(m+n+2t) + \big(2^{md} \vee (3K+1) \big):1, K \big)$$
such that 
$$\mathbb{I} \Big( h^{(\bm{\alpha}, \ell)} \big((b_{\ell}(\bx))_{\ell \in [m+n+2t]} \big) \Big) = \eta^{(\ell)}_{\bm{\alpha},j(\bx),k(\bx),r(\bx)}.$$
We now continue to construct $f$ starting from the $(q+1)$-st hidden layer.
Recall that the $q$-th hidden layer of $f$ outputs $\bm{z}(\bx) := (b_{\ell}(\bx))_{\ell \in [q]}$ with $q \geq m+n+2t$. 
To save the $q$ bits of $\bx$ up to the $(q + (2R + 2)|\mathcal{A}(d,\beta)|q+1)$-st hidden layer of $f$, we use the identity $b = \mathbb{I}(b - 1/2)$ for $b \in \{0,1\}$, resulting in $dq$ Heaviside neurons per layer.
In addition, starting from the $q$-th hidden layer, we construct sub-networks $\mathbb{I}(h^{(\bm{\alpha}, \ell)}(\bm{z}(\bx)))$ for $\bm{\alpha} \in \mathcal{A}(d,\beta)$ and $\ell \in [q]$ in such a way that each new sub-network starts at the layer where the corresponding previous sub-network outputs its result. 
In addition, we forward all outputs of these sub-networks to the $(q + (2R + 2)|\mathcal{A}(d,\beta)|q+1)$-st hidden layer.     
See Fig. \ref{fig_lin_smooth_step2} for an illustration.

Then, the overall output of the $(q + (2R + 2)|\mathcal{A}(d,\beta)|q+1)$-st hidden layer of $f$ is $(b_{\ell}(\bx))_{\ell \in [q]}$ and $(\eta^{(\ell)}_{\bm{\alpha},j(\bx),k(\bx),r(\bx)})_{\bm{\alpha} \in \mathcal{A}(d,\beta), \ell \in [q]}$. Hence, each of the $(q+1)$-st to the $(q + (2R + 2)|\mathcal{A}(d,\beta)|q+1)$-st hidden layers of $f$ has
\begin{align}
    \leq dq + |\mathcal{A}(d,\beta)|q
    + d(m+n+2t) + \big(2^{md} \vee (3K+1) \big) \label{tmp_126}
\end{align}
Heaviside neurons and $K$ linear neurons.

\begin{figure}[t] 
    \centering
    \scalebox{0.95}{
    \tikzset{every node/.style={scale=0.9}}
    \begin{tikzpicture}[node distance=2cm, auto, thick, >=Stealth]
        \node[draw=white] (x1) at (0, 0) { $\big(b_{\ell}(\bx)\big)_{\ell \in [q]}$};        

        \node[draw=white] (x2) at (0, 1) { $\big(b_{\ell}(\bx)\big)_{\ell \in [q]}$};

       \node[draw=white, line width=0.1mm, minimum height=0.6cm, minimum width=1.5cm] (x3) at (0, 2) { };

       \node at (0, 2.1) {\huge $\vdots$};
        

        \node[draw=white] (x4) at (0,3) { $\big(b_{\ell}(\bx)\big)_{\ell \in [q]}$};

        \node[draw=white] (x5) at (0, 4) { $\big(b_{\ell}(\bx)\big)_{\ell \in [q]}$};

       \node[draw=white, line width=0.1mm, minimum height=0.6cm, minimum width=1.5cm] (x6) at (0, 5) { };
       \node at (0, 5.1) {\huge $\vdots$};
       
        \node[draw=white] (x7) at (0, 6) { $\big(b_{\ell}(\bx)\big)_{\ell \in [q]}$};
        \node[draw=white] (x8) at (0, 7) { $\big(b_{\ell}(\bx)\big)_{\ell \in [q]}$};

    \draw[-, cyan, line width=0.2mm] (x1.north) -- (x2.south);
       \draw[-, cyan, line width=0.2mm] (x2.north) -- (x3.south);
       \draw[-, cyan, line width=0.2mm] (x3.north) -- (x4.south);
       \draw[-, cyan, line width=0.2mm] (x4.north) -- (x5.south);
       \draw[-, cyan, line width=0.2mm] (x5.north) -- (x6.south);
       \draw[-, cyan, line width=0.2mm] (x6.north) -- (x7.south);
       \draw[-, cyan, line width=0.2mm] (x7.north) -- (x8.south);

       \node at (0, 10.5) {\Huge $\vdots$};
       \node[draw=white, line width=0.1mm, minimum height=0.6cm, minimum width=1.5cm] (x8) at (0, 13) { };
        \node[draw=white] (x9) at (0, 14) { $\big(b_{\ell}(\bx)\big)_{\ell \in [q]}$};
        \draw[-, cyan, line width=0.2mm] (x8.north) -- (x9.south);



       \node[draw=black, line width=0.1mm, minimum height=3cm, minimum width=2cm] (f1) at (2.5, 2) { };
       \node[right, align=center] at (f1.east) {$h^{\bm{\alpha}_1, 1}$ \\ (Lemma \ref{lemma_piece_bin_lin_ver3})};

       \node[draw=black, line width=0.1mm, draw=black, minimum height=0.6cm, minimum width=1.5cm] (f11) at (2.5, 1) { };
       \node[draw=white, line width=0.1mm, minimum height=0.6cm, minimum width=1.5cm] (f12) at (2.5, 2) { };

       \node at (2.5, 2.1) {\huge $\vdots$};
       
       \node[draw=black, line width=0.1mm, draw=black, minimum height=0.6cm, minimum width=1.5cm] (f13) at (2.5, 3) { };

       \node[draw=white] (f14) at (2.5, 4) {$\eta^{(1)}_{\bm{\alpha}_1,j(\bx),k(\bx),r(\bx)}$ };

       \node[draw=white, line width=0.1mm, minimum height=0.6cm, minimum width=1.5cm] (f15) at (2.5, 5) { };
       \node at (2.5, 5.1) {\huge $\vdots$};
       
        \node[draw=white] (f16) at (2.5, 6) { $\eta^{(1)}_{\bm{\alpha}_1,j(\bx),k(\bx),r(\bx)}$};
        \node[draw=white] (f17) at (2.5, 7) { $\eta^{(1)}_{\bm{\alpha}_1,j(\bx),k(\bx),r(\bx)}$};

       \draw[-, cyan, line width=0.2mm] (f11.north) -- (f12.south);
       \draw[-, cyan, line width=0.2mm] (f12.north) -- (f13.south);
       \draw[-, cyan, line width=0.2mm] (f13.north) -- (f14.south);
       \draw[-, cyan, line width=0.2mm] (f14.north) -- (f15.south);
       \draw[-, cyan, line width=0.2mm] (f15.north) -- (f16.south);
       \draw[-, cyan, line width=0.2mm] (f16.north) -- (f17.south);

       \node at (2.5, 10.5) {\Huge $\vdots$};
       \node[draw=white, line width=0.1mm, minimum height=0.6cm, minimum width=1.5cm] (f18) at (2.5, 13) { };
       \node[draw=white] (f19) at (2.5, 14) {$\eta^{(1)}_{\bm{\alpha}_1,j(\bx),k(\bx),r(\bx)}$ };
        \draw[-, cyan, line width=0.2mm] (f18.north) -- (f19.south);


       \node[draw=black, line width=0.1mm, minimum height=3cm, minimum width=2cm] (f2) at (5, 5) { };
       \node[right, align=center] at (f2.east) {$h^{\bm{\alpha}_1, 2}$ \\ (Lemma \ref{lemma_piece_bin_lin_ver3})};

       \node[draw=black, line width=0.1mm, draw=black, minimum height=0.6cm, minimum width=1.5cm] (f21) at (5, 4) { };
       \node[draw=white, line width=0.1mm, minimum height=1cm, minimum width=1.5cm] (f22) at (5, 5) { };

       \node at (5, 5.1) {\huge $\vdots$};
       
       \node[draw=black, line width=0.1mm, draw=black, minimum height=0.6cm, minimum width=1.5cm] (f23) at (5, 6) { };

       \node[draw=white] (f24) at (5, 7) {$\eta^{(2)}_{\bm{\alpha}_1,j(\bx),k(\bx),r(\bx)}$ };


       \draw[-, cyan, line width=0.2mm] (f21.north) -- (f22.south);
       \draw[-, cyan, line width=0.2mm] (f22.north) -- (f23.south);
       \draw[-, cyan, line width=0.2mm] (f23.north) -- (f24.south);

       \node at (5, 10.5) {\Huge $\vdots$};
       \node[draw=white, line width=0.1mm, minimum height=0.6cm, minimum width=1.5cm] (f25) at (5, 13) { };
       \node[draw=white] (f26) at (5, 14) {$\eta^{(2)}_{\bm{\alpha}_1,j(\bx),k(\bx),r(\bx)}$ };
        \draw[-, cyan, line width=0.2mm] (f25.north) -- (f26.south);



       \node[draw=black, line width=0.1mm, minimum height=3cm, minimum width=2cm] (f3) at (9.5, 12) { };
       \node[right, align=center] at (f3.east) {$h^{\bm{\alpha}_{|\mathcal{A}(d,\beta)|}, q}$ \\ (Lemma \ref{lemma_piece_bin_lin_ver3})};

       \node[draw=black, line width=0.1mm, draw=black, minimum height=0.6cm, minimum width=1.5cm] (f31) at (9.5, 11) { };
       \node[draw=white, line width=0.1mm, minimum height=1cm, minimum width=1.5cm] (f32) at (9.5, 12) { };

       \node at (9.5, 12.1) {\huge $\vdots$};
       
       \node[draw=black, line width=0.1mm, draw=black, minimum height=0.6cm, minimum width=1.5cm] (f33) at (9.5, 13) { };

       \node[draw=white] (f34) at (9.5, 14) {$\eta^{(q)}_{\bm{\alpha}_{|\mathcal{A}(d,\beta)|},j(\bx),k(\bx),r(\bx)}$ };

       \draw[-, cyan, line width=0.2mm] (f31.north) -- (f32.south);
       \draw[-, cyan, line width=0.2mm] (f32.north) -- (f33.south);
       \draw[-, cyan, line width=0.2mm] (f33.north) -- (f34.south);




    \node at (7.25, 8.5) {\huge $\cdots$};
       \node[draw=white] at (7.25, 13.9) {\large $\ldots$};


\draw[-, cyan, line width=0.2mm] (x1.north) -- (f11.south);
\draw[-, cyan, line width=0.2mm] (x4.north) -- (f21.south);

    \end{tikzpicture}
    }
    \vspace{15pt}
    \caption{\textbf{The second part of Step 1.} We represent the first $q$ bits of  $(\partial^{\bm{\alpha}}f_0)(\widetilde{\bx})$ for each $\bm{\alpha} \in \mathcal{A}(d,\beta) = \{\bm{\alpha}_1, \bm{\alpha}_2, \ldots, \bm{\alpha}_{|\mathcal{A}(d,\beta)|} \}$. 
    }  \label{fig_lin_smooth_step2}
\end{figure}

\clearpage

\noindent
\textbf{Step 2}: Approximating $\widetilde{f}_0(\bx)$.

\medskip

\noindent
Given the outputs $(b_{\ell}(\bx))_{\ell \in [q]}$ and
$(\eta^{(\ell)}_{\bm{\alpha},j(\bx),k(\bx),r(\bx)})_{\bm{\alpha} \in \mathcal{A}(d,\beta), \ell \in [q]}$ of the previous step in the $(q + (2R + 2)|\mathcal{A}(d,\beta)|q+1)$-st hidden layer, we now construct the network for the approximation of (\ref{eq_skip_taylor_2}).

For each $i \in [d]$, recall that 
\begin{align*}
    x_i - \widetilde{x}_i = \sum_{\ell=m+n+2t+1}^{\infty} \frac{b_{\ell}(x_i)}{2^{\ell}} - \frac{1}{2^{m+n+2t+1}}.  
\end{align*}
We define 
$$\operatorname{app}(x_i - \widetilde{x}_i):= 
\sum_{\ell=m+n+2t+1}^{q} \frac{b_{\ell}(x_i)}{2^{\ell}} + \frac{1}{2^{q+1}} - \frac{1}{2^{m+n+2t+1}},$$
satisfying
\begin{equation} \label{tmp_122} 
    \begin{aligned}
         \Big| \operatorname{app}(x_i - \widetilde{x}_i) - (x_i - \widetilde{x}_i) \Big| &\leq \frac{1}{2^{q+1}}, \\
         \Big| \operatorname{app}(x_i - \widetilde{x}_i) \Big| &\leq \frac{1}{2^{m+n+2t+1}}.
    \end{aligned}
\end{equation}
For each $\bm{\alpha} \in \mathcal{A}(d,\beta)$, recall that
$$(\partial^{\bm{\alpha}}f_0)(\widetilde{\bx}) := 2\|\partial^{\bm{\alpha}}f_0\|_{\infty} \Bigg(\sum_{\ell=1}^{\infty} \frac{\eta^{(\ell)}_{\bm{\alpha},j(\bx),k(\bx), r(\bx)}}{2^{\ell}} - \frac{1}{2} \Bigg).$$
We define
$$\operatorname{app}\big((\partial^{\bm{\alpha}}f_0)(\widetilde{\bx})\big) := 2\|\partial^{\bm{\alpha}}f_0\|_{\infty} \Bigg(\sum_{\ell=1}^{q} \frac{\eta^{(\ell)}_{\bm{\alpha},j(\bx),k(\bx), r(\bx)}}{2^{\ell}} + \frac{1}{2^{q+1}} - \frac{1}{2} \Bigg),$$
satisfying
\begin{equation} \label{tmp_123}
    \begin{aligned}
            \Big| \operatorname{app}\big((\partial^{\bm{\alpha}}f_0)(\widetilde{\bx})\big) - \big((\partial^{\bm{\alpha}}f_0)(\widetilde{\bx})\big) \Big| &\leq \frac{\|\partial^{\bm{\alpha}}f_0\|_{\infty}}{ 2^{q+1}},\\
            \Big| \operatorname{app}\big((\partial^{\bm{\alpha}}f_0)(\widetilde{\bx})\big) \Big| &\leq \|\partial^{\bm{\alpha}}f_0\|_{\infty}.
    \end{aligned}
\end{equation}

For each $\bm{\alpha}:=(\alpha_1, \dots, \alpha_d) \in (\mathcal{A}(d,\beta) \setminus \{\bm{0}_d\})$, we have
\begin{align*}
    \bx^{\bm{\alpha}} = x_1^{\alpha_1} \cdots x_d^{\alpha_d} = \underbrace{x_1 \cdots x_1}_{\alpha_1}
    \underbrace{x_2 \cdots x_2}_{\alpha_2}
    \cdots \underbrace{x_d \cdots x_d}_{\alpha_d}.
\end{align*}
We choose $\pi_{\bm{\alpha}}: \{1,\ldots,\|\bm{\alpha}\|_1\} \to \{1, \dots, d\}$, such that $\pi_{\bm{\alpha}}(\kappa)$ extracts the index of $\bx^{\bm{\alpha}}$ at position $\kappa$, that is,
$$\bx^{\bm{\alpha}} = x_{\pi_{\bm{\alpha}}(1)} x_{\pi_{\bm{\alpha}}(2)} \ldots x_{\pi_{\bm{\alpha}}(\|\alpha\|_1)}.$$

We define $b_{0}(x_{i}) := 1$ for every $i \in [d]$ and $\eta^{(0)}_{\bm{\alpha},j(\bx),k(\bx), r(\bx)} := 1$ for every $\bm{\alpha} \in \mathcal{A}(d,\beta)$.
In the $(q + (2R + 2)|\mathcal{A}(d,\beta)|q+2)$-nd  hidden layer (the last hidden layer), we represent
\begin{align*}
    \Bigg( \eta^{(\ell_0)}_{\bm{\alpha},j(\bx),k(\bx), r(\bx)} \cdot \prod_{\kappa=1}^{\|\bm{\alpha}\|_1} 
    b_{\ell_\kappa}\left(x_{\pi_{\bm{\alpha}}(\kappa)}\right) \Bigg)_{\bm{\alpha} \in \mathcal{A}(d,\beta), \ \ell_0, \ell_1, \ldots, \ell_{\|\bm{\alpha}\|_1} \in \{0,1,\ldots,q\}},
\end{align*}
using the identity
$\eta \prod_{\kappa=1}^{\|\bm{\alpha}\|_1} b_\kappa = \mathbb{I}(\eta + \sum_{\kappa=1}^{\|\bm{\alpha}\|_1} b_\kappa - \|\bm{\alpha}\|_1 - 1/2)$ for $\eta, b_1, \ldots,b_{\|\bm{\alpha}\|_1} \in \{0,1\}$. 
Then, the number of neurons needed for the last hidden layer is bounded above by
\begin{align}
    \sum_{\bm{\alpha} \in \mathcal{A}(d,\beta)} (q+1)^{\|\bm{\alpha}\|_1 + 1} \leq  |\mathcal{A}(d,\beta)| \cdot (q+1)^{\lceil \beta \rceil}. \label{tmp_127}
\end{align}

The output of the network is then given by
\begin{align}
    f(\bx) &:= \sum_{\bm{\alpha} \in \mathcal{A}(d,\beta)} \operatorname{app}((\partial^{\bm{\alpha}}f_0)(\widetilde{\bx})) \cdot \frac{\prod_{\kappa=1}^{\|\bm{\alpha}\|_1} \operatorname{app}\left(
        x_{\pi_{\bm{\alpha}}(\kappa)} - \widetilde{x}_{\pi_{\bm{\alpha}}(\kappa)}\right)}{\bm{\alpha}!} \label{tmp_199} \\
        &= \sum_{\bm{\alpha} \in \mathcal{A}(d,\beta)} \Bigg[ 2\|\partial^{\bm{\alpha}}f_0\|_{\infty} \left(\sum_{\ell=1}^{q} \frac{\eta^{(\ell)}_{\bm{\alpha},j(\bx),k(\bx), r(\bx)}}{2^{\ell}} + \frac{1}{2^{q+1}} - \frac{1}{2} \right) \nonumber\\
        & \qquad \qquad \qquad \cdot \frac{1}{\bm{\alpha}!} \prod_{\kappa=1}^{\|\bm{\alpha}\|_1} \left( \sum_{\ell=m+n+2t+1}^{q} \frac{b_{\ell}(x_{\pi_{\bm{\alpha}}(\kappa)})}{2^{\ell}} + \frac{1}{2^{q+1}} - \frac{1}{2^{m+n+2t+1}} \right) \Bigg],\nonumber
\end{align}
which is an affine function of the last hidden layer.   

\medskip
\medskip

\noindent
\textbf{Step 3}: Verifying the size and error of the final network.

\medskip

\noindent
As the last step of our proof, we assign specific values to $t$, $m$ and $n$ and evaluate the size and precision of our network $f$.

We set
\begin{equation}
    \begin{aligned}
    t := \left\lfloor \frac{1}{d} \log\left(  \frac{L}{
    \log^{\frac{3}{2}}(Lp)} \right) \right\rfloor, \quad  
    m := \left\lfloor \frac{1}{d} \log\left(  \frac{p}{2} \right) \right\rfloor, \quad \text{and} \ \ 
    n := \left\lfloor \frac{1}{d} \log \left(s \land \frac{p}{7} \right) \right\rfloor. \label{tmp_112}
    \end{aligned}
\end{equation}
Consequently, $R=2^{td} \leq L/\log^{\frac{3}{2}}(Lp)$, $K=2^{nd} \leq s \land p/7$ and $q = \lceil \beta \rceil(m+n+2t) \leq (2\beta+2)\log(Lp)/d$.
Given that by assumption $L \geq \log^2 (Lp) \geq \log(Lp) \log(L)$
and $L\geq c$, for a sufficiently large constant $c$ depending on $\beta$ and $d,$
the number of hidden layers of the network $f$ is bounded above by
\begin{align*}
    &q + (2R + 2)|\mathcal{A}(d,\beta)|q+2 \\
    &\leq \frac{(2\beta+2)\log(Lp)}{d} + 
    \bigg(\frac{ 2L}{\log^{\frac{3}{2}}(Lp)}+2\bigg) |\mathcal{A}(d,\beta)|\frac{(2\beta+2)\log(Lp)}{d}+2
    \\
    &\leq \frac{(2\beta+2)L}{d \log(L)} + \frac{L}{2} + 2 \\
    &\leq L.
\end{align*}
According to (\ref{tmp_126}) and (\ref{tmp_127}), the maximal number of neurons in each hidden layer of $f$ is bounded by 
\begin{align}
     &\Big( dq + |\mathcal{A}(d,\beta)|q
    + d(m+n+2t) + \big(2^{md} \vee (3K+1) \big) \Big) \vee \Big( |\mathcal{A}(d,\beta)| (q+1)^{\lceil \beta \rceil} \Big) \nonumber \\
     &\leq \Big(q(2d+|\mathcal{A}(d,\beta)|) + \frac{p}{2} \Big) \vee \Big( |\mathcal{A}(d,\beta)| (q+1)^{\lceil \beta \rceil} \Big). \label{tmp_106}
\end{align}
Given that by assumption $p \geq \log^{\beta+2} (Lp)$ and $L\geq c$, for a sufficiently large constant $c$ depending on $\beta$ and $d,$ 
we have moreover
\begin{align*}
    q\big(2d+|\mathcal{A}(d,\beta)|\big) + \frac{p}{2}
    &\leq (2\beta+2) \big(2d+|\mathcal{A}(d,\beta)|\big)\log(Lp)  + \frac{p}{2} \\
    &\leq (2\beta+2) \big(2d+|\mathcal{A}(d,\beta)|\big)\frac{p}{\log^{\beta+1}(p)}  + \frac{p}{2} \\
    &\leq p
\end{align*}
and 
\begin{align*}
    |\mathcal{A}(d,\beta)| (q+1)^{\lceil \beta \rceil}
    &\leq |\mathcal{A}(d,\beta)| \big( (2\beta+2) \log(Lp) + 1  \big)^{\beta+1}\\
    &\leq |\mathcal{A}(d,\beta)| (2\beta+3)^{\beta+1}  \log^{\beta+1}(Lp) \\
    &\leq |\mathcal{A}(d,\beta)| (2\beta+3)^{\beta+1}  \frac{p}{\log(p)} \\
    &\leq p.
\end{align*}
Hence, (\ref{tmp_106}) is bounded above by $p$.
Given that each hidden layer has at most $d \vee K \leq s$ linear neurons, we conclude
$f \in \DHN_{\lin}(L,d:p:1,s)$.

\medskip \medskip

Next we bound the approximation error of the network $f$ defined in \eqref{tmp_199} for approximating $f_0$.
For each $\bm{\alpha} \in \mathcal{A}(d,\beta)$, recall that $(\bx-\widetilde{\bx})^{\bm{\alpha}} = \prod_{\kappa=1}^{\|\bm{\alpha}\|_1} 
        (x_{\pi_{\bm{\alpha}}(\kappa)} - \widetilde{x}_{\pi_{\bm{\alpha}}(\kappa)})$.
By (\ref{tmp_122}), (\ref{tmp_123}) and Lemma \ref{lemma_prod_diff}, we have 
\begin{align*}
        &\left|\operatorname{app}\big((\partial^{\bm{\alpha}}f_0)(\widetilde{\bx})\big) \cdot \prod_{\kappa=1}^{\|\bm{\alpha}\|_1} 
        \operatorname{app}\left(x_{\pi_{\bm{\alpha}}(\kappa)} - \widetilde{x}_{\pi_{\bm{\alpha}}(\kappa)}\right)
        - (\partial^{\bm{\alpha}}f_0)(\widetilde{\bx}) \cdot (\bx-\widetilde{\bx})^{\bm{\alpha}}\right| \\
        &= \left|\operatorname{app}\big((\partial^{\bm{\alpha}}f_0)(\widetilde{\bx})\big) \cdot \prod_{\kappa=1}^{\|\bm{\alpha}\|_1} 
        \operatorname{app}\left(x_{\pi_{\bm{\alpha}}(\kappa)} - \widetilde{x}_{\pi_{\bm{\alpha}}(\kappa)}\right)
        - (\partial^{\bm{\alpha}}f_0)(\widetilde{\bx}) \cdot \prod_{\kappa=1}^{\|\bm{\alpha}\|_1} 
        \left(x_{\pi_{\bm{\alpha}}(\kappa)} - \widetilde{x}_{\pi_{\bm{\alpha}}(\kappa)}\right) \right| \\
        &\leq \|\partial^{\bm{\alpha}} f_0\|_{\infty}
        \left( \frac{1}{2^{m+n+2t+1}} \right)^{\|\bm{\alpha}\|_1} 
        \left( \frac{1}{2^{q+1}} + \frac{\|\bm{\alpha}\|_1}{2^{q-m-n-2t}} \right)\\
        &\leq \frac{\|\partial^{\bm{\alpha}} f_0\|_{\infty}}{2^q}. 
\end{align*}
Hence, we get
\begin{align*}
    &\left|f(\bx) - \sum_{\bm{\alpha} \in \mathcal{A}(d,\beta)} (\partial^{\bm{\alpha}}f_0)(\widetilde{\bx}) \cdot \frac{(\bx-\widetilde{\bx})^{\bm{\alpha}}}{\bm{\alpha}!} \right| \\
    &\leq \sum_{\bm{\alpha} \in \mathcal{A}(d,\beta)}  \left| \operatorname{app}\big((\partial^{\bm{\alpha}}f_0)(\widetilde{\bx})\big) \cdot \frac{\prod_{\kappa=1}^{\|\bm{\alpha}\|_1} \operatorname{app}\left(
        x_{\pi_{\bm{\alpha}}(\kappa)} - \widetilde{x}_{\pi_{\bm{\alpha}}(\kappa)}\right)}{\bm{\alpha}!} - (\partial^{\bm{\alpha}}f_0)(\widetilde{\bx}) \cdot \frac{(\bx-\widetilde{\bx})^{\bm{\alpha}}}{\bm{\alpha}!} \right| \\    
    &\leq \sum_{\bm{\alpha} \in \mathcal{A}(d,\beta)} \frac{\|\partial^{\bm{\alpha}} f_0\|_{\infty}}{2^q} \\ 
    &\leq  \frac{\left\| f_0 \right\|_{\mathcal{C}^{\beta}}}{2^{ \beta (m+n+2t)}}.
\end{align*}
By Lemma \ref{lemma_taylor}, we further have
\begin{align*}
    \left| \sum_{\bm{\alpha} \in \mathcal{A}(d,\beta)} (\partial^{\bm{\alpha}}f_0)(\widetilde{\bx}) \cdot \frac{(\bx-\widetilde{\bx})^{\bm{\alpha}}}{\bm{\alpha}!} -f_0(\bx) \right| 
    \leq&  \left\| f_0 \right\|_{\mathcal{C}^{\beta}} \cdot \|\bx - \widetilde{\bx} \|_{\infty}^\beta \\
    \leq &   \frac{\left\| f_0 \right\|_{\mathcal{C}^{\beta}}}{2^{\beta(m+n+2t+1)}}. 
\end{align*}
In summary, for every $\bx \in [0,1]^d$,
\begin{align*}
    \big| f(\bx) - f_0(\bx) \big| 
    &\leq \frac{2 \left\| f_0 \right\|_{\mathcal{C}^{\beta}}}{2^{\beta(m+n+2t)}} \\
    &\leq 
    2 \cdot 16^{\beta} \cdot \left\| f_0 \right\|_{\mathcal{C}^{\beta}} \left( \frac{L^2}{\log^{3}(Lp)} \cdot \frac{p}{2} \cdot \frac{s}{7} \right)^{-\frac{\beta}{d}} \\
    &\leq 2 \cdot 16^{\beta} \cdot 14^{\beta} \cdot \left\| f_0 \right\|_{\mathcal{C}^{\beta}} \left(\frac{\log^3 (Lp)}{L^2 ps}\right)^{\frac{\beta}{d}},
\end{align*}
using for the second inequality that by \eqref{tmp_112}, 
\begin{align*}
    t \geq \frac{1}{d} \log\left(  \frac{L}{
    \log^{\frac{3}{2}}(Lp)} \right) - 1, \quad 
    m \geq \frac{1}{d} \log\left(  \frac{p}{2} \right)-1
    , \quad
    n \geq \frac{1}{d} \log \left(\frac{s}{7} \right)-1. 
\end{align*} 

Hence, we obtain 
\begin{align}
    \sup_{f_0 \in \mathcal{H}_{d}^{\beta}(M)} \, \inf_{f \in \DHN_{\lin}(L,d:p:1,s)} \left\| f - f_0 \right\|_{\infty}
    & \leq c'' M \cdot  \left(\frac{\log^3 (Lp)}{L^2 ps }\right)^{\frac{\beta}{d}} \label{tmp_201}
\end{align}
with $c''>0$ a constant only depending on $\beta$. 
By (\ref{tmp_200}) and (\ref{tmp_201}), we obtain the assertion of Theorem \ref{upper_smooth_lin}.

\end{proof}

\section{Auxiliary lemmas}
\label{app_proofs_lemmas}

\begin{lemma}  \label{lemma_taylor}
For any $\beta$-Hölder smooth function 
    $f : [0,1]^d \to \mathbb{R}$ and any $\bx, \bx_0 \in [0,1]^d$,
    \begin{align*}
        \left| f(\bx) - \sum_{\substack{\bm{\alpha} \in \mathbb{N}^d \\ \left\|\bm{\alpha}\right\|_1 < \beta}} (\partial^{\bm{\alpha}}f)(\bx_0) \frac{(\bx-\bx_0)^{\bm{\alpha}}}{\bm{\alpha}!} \right| \leq \left\| f \right\|_{\mathcal{C}^{\beta}} \cdot \left\|\bx - \bx_0 \right\|_{\infty}^\beta.
    \end{align*}
\end{lemma}

\begin{proof}
    Let $\beta = q+s$ for some $q \in \mathbb{N}$ and $s \in (0,1]$.
    By Taylor's theorem, there exists $\xi \in [0,1]$ such that
    $$f(\bx)=\sum_{\substack{\bm{\alpha} \in \mathbb{N}^d \\ \left\|\bm{\alpha}\right\|_1 < q-1}}
    (\partial^{\bm{\alpha}}f)(\bx_0) \frac{(\bx-\bx_0)^{\bm{\alpha}}}{\bm{\alpha}!}
    +\sum_{\substack{\bm{\alpha} \in \mathbb{N}^d \\ \left\|\bm{\alpha}\right\|_1 = q}} \left(\partial^{\bm{\alpha}} f\right) \big(\bx_0+\xi(\bx-\bx_0)\big) \frac{(\bx-\bx_0)^{\bm{\alpha}}}{\bm{\alpha}!}.$$
    Hence,
    \begin{align*}
        &\left| f(\bx) - \sum_{\substack{\bm{\alpha} \in \mathbb{N}^d \\ \left\|\bm{\alpha}\right\|_1 < \beta}} (\partial^{\bm{\alpha}}f)(\bx_0) \frac{(\bx-\bx_0)^{\bm{\alpha}}}{\bm{\alpha}!} \right| \\
        &= \left| f(\bx) - \sum_{\substack{\bm{\alpha} \in \mathbb{N}^d \\ \left\|\bm{\alpha}\right\|_1 \leq q}} (\partial^{\bm{\alpha}}f)(\bx_0) \frac{(\bx-\bx_0)^{\bm{\alpha}}}{\bm{\alpha}!} \right| \\
        &= \left|  \sum_{\substack{\bm{\alpha} \in \mathbb{N}^d \\ \left\|\bm{\alpha}\right\|_1 = q}} \Big[  \left(\partial^{\bm{\alpha}} f\right)(\bx_0+\xi(\bx-\bx_0)) - \left(\partial^{\bm{\alpha}} f\right)(\bx_0)  \Big] \cdot \frac{(\bx-\bx_0)^{\bm{\alpha}}}{\bm{\alpha}!} \right| \\
        & \leq \left\| f \right\|_{\mathcal{C}^{\beta}} \cdot \| \bx - \bx_0 \|_{\infty}^{s}  \cdot \| \bx - \bx_0 \|_{\infty}^{q}\\
        & =\left\| f \right\|_{\mathcal{C}^{\beta}} \cdot \| \bx - \bx_0 \|_{\infty}^{\beta}.
    \end{align*}
\end{proof}

\medskip

\begin{lemma} \label{lemma_prod_diff}
    For any $x_1,\ldots, x_d, x_1',\ldots, x_d' \in \mathbb{R}$ such that $|x_i|, |x'_i| \leq a_i$ for each $i \in [d]$, 
    $$\left| \prod_{i=1}^d x_i - \prod_{i=1}^d x'_i \right| \leq \bigg(\prod_{i=1}^d a_i \bigg) \cdot  \sum_{i=1}^d \frac{\left|x_i - x'_i\right|}{a_i}.$$    
\end{lemma}

\begin{proof}
\begin{align*}
    \left| x_1 x_2 \ldots x_d - x'_1 x'_2 \ldots x'_d  \right| 
    &\leq  \left| x_1 x_2 \ldots x_d - x'_1 x_2 \ldots x_d  \right|
    + \left| x'_1 x_2 \ldots x_d - x'_1 x'_2 \ldots x_d  \right| \\
    & \hspace{3.5cm} + \ldots   + \left| x'_1 x'_2 \ldots x'_{d-1} x_d - x'_1 x'_2 \ldots x'_{d-1} x'_d  \right| \\
    &\leq \sum_{j=1}^d \bigg( \left|x_j - x'_j\right| \cdot \prod_{k \neq j} a_k \bigg)\\
    &\leq \bigg(\prod_{j=1}^d a_j \bigg) \cdot \sum_{j=1}^d \frac{\left|x_j - x'_j\right|}{a_j}.
\end{align*}
\end{proof}

\medskip

\begin{lemma}[Lemma~1 of \cite{peter1}]\label{poly-vc}
Suppose $f_{1}(\cdot),f_{2}(\cdot),\ldots,f_{T}(\cdot)$ are polynomials of degree at most $d$ in $u\leq T$ variables. The number of distinct sign vectors generated by varying ${\bm u}\in\mathbb{R}^{u}$ is bounded by
$$N:=\Big|\Big\{\Big(\sgnn(f_{1}({\bm u})),\ldots,\sgnn(f_{T}({\bm u}))\Big),\;{\bm u}\in\mathbb{R}^{u}\Big\}\Big|\leq 2\left(\frac{2edT} u\right)^{u}.$$
\end{lemma}

\begin{lemma}[Lemma~18 of \cite{peter2}]\label{tech-inequality}
Suppose that $2^{m}\leq2^{t}(mr/w)^{w}$ for some $r\geq16$ and $m\geq w\geq t\geq0$. Then, $m\leq t+w\log (2r\log r)$.
\end{lemma}

\section{Acknowledgments}

This work is supported by ERC grant A2B (grant agreement No. 101124751).

\bibliography{ref}

\end{document}